\documentclass[twoside,11pt]{article}

\usepackage{jmlr2e}

\usepackage{custom_tex}
\usepackage[page]{appendix}
\usepackage[T1]{fontenc}
\usepackage{subcaption}

\allowdisplaybreaks

\renewcommand{\paragraph}[1]{\medskip\noindent\textit{#1}}

\usepackage{lastpage}
\jmlrheading{24}{2023}{1-\pageref{LastPage}}{3/21; Revised 4/23}{8/23}{21-0224}{Shuxiao Chen, Qinqing Zheng, Qi Long, and Weijie J. Su}
\ShortHeadings{Minimax Estimation for Personalized Federated Learning}{Chen, Zheng, Long and Su}

\firstpageno{1}

\begin{document}

\title{Minimax Estimation for Personalized Federated Learning: An Alternative between FedAvg and Local Training?}

\author{\name Shuxiao Chen\textsuperscript{$\star$} \email shuxiaoc@wharton.upenn.edu \\
        \name Qinqing Zheng\textsuperscript{$\dagger$} \email zhengqinqing@gmail.com \\
        \name Qi Long\textsuperscript{$\star$} \email qlong@upenn.edu\\
        \name Weijie J.~Su\textsuperscript{$\star$} \email suw@wharton.upenn.edu\\
       \addr
       University of Pennsylvania\textsuperscript{$\star$}\\
       Meta AI Research\textsuperscript{$\dagger$}
       }

\editor{Ambuj Tewari}

\maketitle

\begin{abstract}
A widely recognized difficulty in federated learning arises from the statistical heterogeneity among clients: local datasets often originate from distinct yet not entirely unrelated probability distributions, and personalization is, therefore, necessary to achieve optimal results from each individual's perspective.
In this paper, we show how the excess risks of personalized federated learning using a smooth, strongly convex loss depend on data heterogeneity from a minimax point of view, with a focus on the \textsc{FedAvg} algorithm \citep{mcmahan2017communication} and pure local training (i.e., clients solve empirical risk minimization problems on their local datasets without any communication). Our main result reveals an \textit{approximate} alternative between these two baseline algorithms for federated learning: the former algorithm is minimax rate optimal over a collection of instances when data heterogeneity is small, whereas the latter is minimax rate optimal when data heterogeneity is large, and the threshold is sharp up to a constant.

As an implication, our results show that from a worst-case point of view, a dichotomous strategy that makes a choice between the two baseline algorithms is rate-optimal. Another implication is that the popular \textsc{FedAvg} following by local fine tuning strategy is also minimax optimal under additional regularity conditions.
Our analysis relies on a new notion of algorithmic stability that takes into account the nature of federated learning.





\end{abstract}

\begin{keywords}
  Empirical Risk Minimization,
  Federated Learning, Personalization, Data Heterogeneity, Minimax Rates, Algorithmic Stability
\end{keywords}


\addtocontents{toc}{\protect\setcounter{tocdepth}{0}}

\section{Introduction}

As one of the most important ingredients driving the success of machine learning, data are being generated and subsequently stored in an increasingly decentralized fashion in many real-world applications. 
For example, mobile devices will in a single day collect an unprecedented amount of data from users. These data commonly contain sensitive information such as web search histories, online shopping records, and health information, and thus are often not available to service providers~\citep{poushter2016smartphone}. This decentralized nature of (sensitive) data poses substantial challenges to many machine learning tasks.


To address this issue, \citet{mcmahan2017communication} proposed a new learning paradigm, which they termed \emph{federated learning}, for collaboratively training machine learning models on data that are locally possessed by multiple clients with the coordination of the central server (e.g., service provider), without having direct access to the local datasets. In its simplest form, federated learning considers a pool of $m$ clients, where the $i$-th client has a local dataset $S_i$ of size $n_i$, consisting of i.i.d.~samples $\{\zzz{i}_{j}: j\in[n_i]\}$ (denote $[n] := \{1, 2, \ldots, n\}$) from some unknown distribution $\calD_i$. Letting $\ell(\bsw, z)$ be a loss function, where $\bsw$ denotes the model parameter, the optimal local model for the $i$-th client is given by
\begin{equation}
    \label{eq:opt_local}
    \www{i}_\star \in \underset{\bsw}{\argmin}~  \bbE_{Z_i\sim \calD_i}\ell(\bsw, Z_i).
\end{equation}
From the \emph{client-wise} perspective, any data-dependent estimator $\hwww{i}(\bS)$, with $\bS = \{\Si\}_{i=1}^m$ denoting the collection of all samples, can be evaluated based on its {individualized excess risk}:
\begin{equation*}
    \ier_i:= \bbE_{Z_i \sim \calD_i} [\ell(\hwww{i}, Z_i) - \ell(\www{i}_\star, Z_i)],
\end{equation*}
where the expectation is taken over a fresh sample $Z_i\sim \calD_i$. At a high level, this learning paradigm of federated learning aims to obtain possibly different trained models for each client such that the individualized excess risks are low (see, e.g., \citealt{kairouz2019advances}).



From a statistical viewpoint, perhaps the most crucial factor in determining the effectiveness of federated learning is \textit{data heterogeneity}. When the data distribution $\calD_i$ is (approximately) homogeneous across different clients, presumably a \textit{single} global model would lead to small $\ier_i$ for all $i$. In this regime, indeed, \citet{mcmahan2017communication} proposed the \emph{federated averaging} algorithm (\falong, see Algorithm \ref{alg:fedavg}), which can be regarded as an instance of local stochastic gradient descent (SGD) for solving~\citep{mangasarian1993backpropagation,stich2019local} 
\begin{align}
    \label{eq:hard_sharing}
    & \min_{\bsw} \frac{1}{N} \sum_{i\in[m]} n_i L_i(\bsw, S_i), 
\end{align}
where $L_i(\bsw, S_i) := \sum_{j\in[n_i]} \ell(\bsw, \zzz{i}_j) / n_i$ is the empirical risk minimization (ERM) objective of the $i$-th client and $N = n_1 + \cdots + n_m$ denotes the total number of training samples. Translating Algorithm~\ref{alg:fedavg} into words, \falong~in effect learns a {shared global model} using gradients from each client and outputs a single model as an estimate of $\www{i}_\star$ for all clients. 
When the distributions $\{\calD_i\}$ coincide with each other, \falong~with a strongly convex loss achieves a weighted average excess risk of $\calO(1/N)$, which is minimax optimal up to a constant factor~\citep{shalev2009stochastic,agarwal2012information}, see the formal statement in Theorem~\ref{thm:FA}.



\begin{algorithm}[t]
\DontPrintSemicolon
\KwIn{initialize $\www{\glob}_0$, number of communication rounds $T$, step sizes $\{\eta_t\}_{t=0}^{T-1}$
}
\For{$t=0, 1, \hdots, T-1$}{
    Randomly sample a batch of clients $\calC_t\subseteq[m]$\\
    \For{client $i\in\calC_t$}{
        Obtain $\www{i}_{t+1}$ by running several steps of SGD on $S_i$ using  $\www{\glob}_t$ as the initialization
    }
    $\www{\glob}_{t+1}\leftarrow \www{\glob}_t - \frac{m\eta_t}{N |\calC_t|} \sum_{i\in\calC_t} n_i(\www{\glob}_t - \www{i}_{t+1})$ 
}
 
\KwOut{$\hwww{i} = \www{\glob}_T, \;i\in[m]$}
\caption{\falong~\citep{mcmahan2017communication}}
\label{alg:fedavg}
\end{algorithm}

However, it is an entirely different story in the presence of data heterogeneity. \falong~has been recognized to give inferior performance when there is a significant departure from complete homogeneity (see, e.g., \citealt{bonawitz2019towards}). To better understand this point, consider the extreme case where the data distributions $\{\calD_i\}$ are entirely unrelated. This roughly amounts to saying that the model parameters $\{\www{i}_\star\}$ can be arbitrarily different from each other. In such a ``completely heterogeneous'' scenario, the objective function \eqref{eq:hard_sharing} simply has no clear interpretation, and any single global model---for example, the output of \falong---would lead to unbounded risks for most, if not all, clients. As a matter of fact, it is not difficult to see that the optimal training strategy for federated learning in this regime is arguably \pllong, which lets each client separately run SGD to minimize its own local ERM objective 
\begin{equation}
    \label{eq:pure_local_training}
    \min_{\www{i}} L_i(\www{i}, S_i)
\end{equation}
without any communication. Indeed, \pllong~is minimax rate optimal in the completely heterogeneous regime, just as \falong~in the completely homogeneous regime (see Theorem~\ref{thm:PLT}).



The level of data heterogeneity in practical federated learning problems is apparently neither complete homogeneity nor complete heterogeneity. Thus, the foregoing discussion raises a pressing question of what would happen if we are in the \textit{wide} middle ground of the two extremes.
This underlines the essence of \emph{personalized federated learning}, which seeks to develop algorithms that perform well over a wide spectrum of data heterogeneity. Despite a venerable line of work on personalized federated learning (see, e.g., \citealt{kulkarni2020survey}), the literature remains relatively silent on how the \textit{fundamental} limits of personalized federated learning depend on data heterogeneity, as opposed to two extreme cases where both the minimax optimal rates and algorithms are known.


\subsection{Main Contributions}\label{subsec:contrib}
The present paper takes a step toward understanding the statistical limits of personalized federated learning by establishing the minimax rates of convergence for both individualized excess risks and their weighted average with smooth strongly convex losses. We briefly summarize our main contributions below.
\begin{enumerate}
    \item We prove that if the client-wise sample sizes are relatively balanced, then there exists a problem instance on which the $\ier_i$'s of any algorithm are lower bounded by  
    \begin{equation}
        \label{eq:lb_intro}
        \begin{cases}
            \Omega(1/N + R^2) & \textnormal{if } R^2 = \calO(m/N) \\
            \Omega(m/N) & \textnormal{if } R^2 = \Omega(m/N),
        \end{cases}
    \end{equation}
    where $R$ is the minimum quantity satisfying
    $\min_{\bsw\in\calW}  \sum_{i\in[m]} n_i \|\www{i}_\star - \bsw\|^2/N \leq R^2,$
    i.e., it measures the maximum level of heterogeneity among clients (here $\|\cdot\|$ throughout the paper denotes the Euclidean distance).
    Meanwhile, we show that the $\ier_i$'s of \falong~are upper bounded by $\calO(1/N + R^2)$, whereas the guarantee for \pllong~is $\calO(m/N)$, regardless of the specific value of $R$. Moreover, we also establish similar upper and lower bounds for a weighted average of the $\ier_i$'s under a weaker condition.

    \item A closer look at the above-mentioned bounds reveals a perhaps surprising phenomenon: for a given collection of problem instances with a specified maximum level of heterogeneity, exactly one of \falong~or \pllong~is minimax optimal. 

    \item The established minimax results suggest that the \naive~dichotomous strategy of (1) running \falong~when $R^2 = \calO(m/N)$, and (2) running \pllong~when $R^2=\Omega(m/N)$, attains the lower bound \eqref{eq:lb_intro}. Moreover, for supervised problems, this dichotomous strategy can be implemented without knowing $R$ by (1) running both \falong~and \pllong, (2) evaluating the test errors of the two algorithms in a distributed fashion, and (3) deploying the algorithm with a lower test error.
    We emphasize that the notion of optimality under our consideration overlooks constant factors.
    In practice, a better personalization result could be achieved by more sophisticated algorithms. 

    \item As a side product, we provide a novel analysis of \fplong, a popular algorithm for personalized federated learning that constrains the learned local models to be close via $\ell_2$ regularization \citep{li2018federated}. In particular, we show that its $\ier_i$'s are of order $\calO\big(\frac{1}{N/m}\land \frac{R}{\sqrt{N/m}} + \frac{\sqrt{m}}{N}\big)$, and a weighted average of the $\ier_i$'s satisfies a tighter $\calO\big(\frac{1}{N/m}\land \frac{R}{\sqrt{N/m}} + \frac{1}{N}\big)$ bound, where $a\land b = \min\{a, b\}$ for two real numbers $a$ and $b$.

    \item On the technical side, our upper bound analysis is based on a generalized notion of algorithmic stability \citep{bousquet2002stability}, which we term \emph{\fedstab}~and can be of independent interest. Briefly speaking, an algorithm $\calA(\bS) = \{\hwww{i}(\bS)\}$ has \fedstab~$\{\gamma_i\}$ if for any $i\in[m]$, the loss function evaluated at $\hwww{i}(\bS)$ can only change by an additive term of $\calO(\gamma_i)$, if we perturb $S_i$ a little bit, while keeping the rest of datasets $\{S_{i'}: i'\neq i\}$ fixed. Similar ideas have appeared in \cite{maurer2005algorithmic} and have been recently applied to multi-task learning \citep{wang2018distributed}. However, their notion of perturbation is based on the deletion of the {whole client-wise dataset}, whereas our notion of \fedstab~operates at the ``record-level'' and is more fine-grained. On the other hand, our construction of the lower bound is based on a generalization of Assound's lemma \citep{assouad1983deux} (see also \citealt{yu1997assouad}), which enables us to handle multiple heterogeneous datasets.
\end{enumerate}

\subsection{Related Work}
Ever since the proposal of federated learning by \citet{mcmahan2017communication}, recent years have witnessed a rapidly growing line of work that is concerned with various aspects of \falong~and its variants (see, e.g., \citealt{khaled2019first,haddadpour2019convergence,li2020convergence,bayoumi2020tighter,malinovsky2020local,li2020unified,woodworth2020minibatch,yuan2020federated,zheng2021federated}). 

In the context of personalized federated learning, there have been significant algorithmic developments in recent years. While the idea of using $\ell_2$ regularization to constrain the learned models to be similar has appeared in early works on multi-task learning \citep{evgeniou2004regularized}, its applicability to personalized federated learning was only recently demonstrated by \citet{li2018federated}, where the \fplong~algorithm was introduced. Similar regularization-based methods have been proposed and analyzed from the scope of convex optimization in \citet{hanzely2020federated,dinh2020personalized}, and \citet{hanzely2020lower}. 
In particular, \citet{hanzely2020lower} showed that an accelerated variant of \sfalong~is optimal in terms of communication complexity and the local oracle complexity. 
There is also a line of work using model-agnostic meta learning \citep{finn2017model} to achieve personalization \citep{jiang2019improving,fallah2020personalized}. 
Other strategies have been proposed (see, e.g., \citealt{arivazhagan2019federated,li2019fedmd,mansour2020three,yu2020salvaging}), and we refer readers to \citet{kulkarni2020survey} for a comprehensive survey. We briefly remark here that all the papers mentioned above only consider the \emph{optimization properties} of their proposed algorithms, while we focus on statistical properties of personalized federated learning. 

Compared to the optimization understanding, our statistical understanding (in terms of sample complexity) of federated learning is still limited. \citet{deng2020adaptive} proposed an algorithm for personalized federated learning with learning-theoretic guarantees. 
However, it is unclear how their bound scales with the heterogeneity among clients. 

More generally, exploiting the information ``shared among multiple learners'' is a theme that constantly appears in other fields of machine learning such as multi-task learning \citep{caruana1997multitask}, meta learning \citep{baxter2000model}, and transfer learning \citep{pan2009survey}, from which we borrow a lot of intuitions (see, e.g., \citealt{ben2006analysis,ben2008notion,ben2010theory,maurer2016benefit,cai2019transfer,hanneke2019value,hanneke2020no,du2020few,tripuraneni2020provable,tripuraneni2020theory,kalan2020minimax,shui2020beyond,li2020transfer,zhang2020sharp,jose2021information}). 



More related to our work, a series work by \citet{denevi2018learning,denevi2019learning,balcan2019provable}, and \citet{khodak2019adaptive} assumes the optimal local models lie in a small sub-parameter-space, and establishes ``heterogeneity-aware'' bounds on a weighted average of individualized excess risks. However, we would like to point out that they operate under the online learning setup, where the datasets are assumed to come in streams, and this is in sharp contrast to the federated learning setup, where the datasets are decentralized.
Our notion of heterogeneity is also related to the hierarchical Bayesian model considered in \citet{bai2020important,lucas2020theoretical,konobeev2020optimality}, and \citet{chen2020global}.

\subsection{Paper Organization}
The rest of this paper is organized as follows. In Section \ref{sec:setup}, we give an exposition of the problem setup and main assumptions. 
Section \ref{sec:main_res} presents our main results with proof sketches.
We conclude this paper with a discussion of open problems in Section \ref{sec:discuss}. 
For brevity, detailed proofs are deferred to the appendix.

\section{Problem Setup}\label{sec:setup}

In this section, we detail some preliminaries to prepare the readers for our main results.

\medskip
\paragraph{Notation.} We introduce the notation we are going to use throughout this paper. For two real numbers $a, b$, we let $a\lor b = \max\{a, b\}$ and $a \land b = \min\{a, b\}$.
For two non-negative sequences $a_n, b_n$, we denote $a_n\lesssim b_n$ (resp. $a_n\gtrsim b_n$) if $a_n\leq C b_n$ (resp. $a_n\geq C b_n$) for some constant $C>0$ when $n$ is sufficiently large. 
We use $a_n \asymp b_n$ to indicate that $a_n\gtrsim b_n, a_b \lesssim b_n$ hold simultaneously. We also use $a_n = \calO(b_n)$, whose meaning is the same as $a_n\lesssim b_n$, and $a_n = \Omega(b_n)$, whose meaning is the same as $a_n \gtrsim b_n$.
For two probability distributions $\calD_1$ and $\calD_2$, we use $\calD_1 \otimes \calD_2$ to denote their joint distribution under independence.
We use $\calW$ to denote the parameter space and $\calZ$ to denote the sample space.
Finally, we let $\calP_\calW (x) := \argmin_{y\in\calW}\|x-y\|$ denote the operator that projects $x$ onto $\calW$ in Euclidean distance.

\paragraph{Evaluation Metrics.}
The presentation of our main results relies on how to evaluate the performance of a federated learning algorithm. To this end, we consider the following two evaluation metrics.


\begin{definition}[Individualized excess risk]
Consider an algorithm $\calA$ that outputs $\calA(\bS) = \{\hwww{i}(\bS)\}_{i=1}^m$. For the $i$-the client, its \emph{individualized excess risk} (IER) is defined as
\begin{equation}
    \label{eq:ier}
    \ier_i(\calA) := \bbE_{Z_i\sim\calD_i}[\ell(\hwww{i}(\bS), Z_{i}) - \ell(\www{i}_\star, Z_{i})],
\end{equation}
where $Z_i\sim \calD_i$ is a fresh data point independent of $\bS$. 
\end{definition}

\begin{definition}[$\bp$-average excess risk]
    Consider an algorithm $\calA$ that outputs $\calA(\bS) = \{\hwww{i}(\bS)\}$. For a vector $\bp = (p_1, \hdots, p_m)$ lying in the $m$-dimensional probability simplex (i.e., all $p_i$'s are non-negative and they sum to one), we define the $\bp$-average excess risk ($\aer_\bp$) of $\calA$ to be 
    \begin{equation}
        \label{eq:aer_p}
        \aer_\bp(\calA) := \sum_{i\in[m]} p_i \cdot  \ier_i(\calA).
    \end{equation}
\end{definition}


In words, \ier~measures the performance of the algorithm from the \emph{client-wise} perspective, whereas \aer~evaluates the performance of the algorithm from the \emph{system-wide} perspective.

Intuitively speaking, the weight vector $\bp$ in \eqref{eq:aer_p} can be regarded as the {importance weight} on each client and controls ``how many resources are allocated to each client''. 
For example, setting $p_i = 1/m$ enforces ``fair allocation'', so that each client is treated uniformly, regardless of sample sizes. 
As another example, setting $p_i = n_i/N$ (recall that $N = \sum_{i\in[m]}n_i$ is the total sample size) means that the central server pays more attention to clients with larger sample sizes, which, to a certain extend, incentivize the clients to contribute more data. 

Notably, while a uniform upper bound on all $\ier_i$'s can be carried over to the same bound on $\aer_\bp$, a bound on the $\aer_\bp$ alone in general does not imply a {tight} bound on each $\ier_i$, other than the trivial bound $\ier_i \leq \aer_\bp / p_i$. 
Such a subtlety is a distinguishing feature of personalized federated learning in the following sense: under homogeneity, it suffices to estimate a single shared global model, and thus $\aer_\bp$ and all of $\ier_i$s are mathematically equivalent.


\paragraph{Regularity Conditions.} 
In this paper, we restrict ourselves to bounded, smooth, and strongly convex loss functions. Such assumptions are common in the federated learning literature (see, e.g., \citealt{li2020convergence,hanzely2020lower}) and cover many unsupervised learning problems such as mean estimation in exponential families and supervised learning problems such as generalized linear models.

\begin{assumption}[Regularity conditions]
   \label{assump:regularity} 
   Suppose the following conditions hold:
   \begin{enumerate}
        \item[(\compact)] \emph{Compact and convex domain.} The parameter space $\calW$ is a compact convex subset of $\bbR^d$ with diameter $D:= \sup_{\bsw, \bsw'\in\calD}\|\bsw-\bsw'\| < \infty$;
        \item[(\cvx)] \emph{Smoothness and strong convexity.} For any $i\in[m]$, the loss function $\ell(\cdot, \bsz)$ is $\beta$-smooth for almost every $\bsz$ in the support of $\calD_i$, and the $i$-th ERM objective $L_i(\cdot, S)$ is almost surely $\mu$-strongly convex on the convex domain $\calW\subseteq \bbR^d$. We also assume that there exists a universal constant $\|\ell\|_\infty$ such that $0\leq \ell(\cdot, \bsz)\leq \| \ell \|_\infty$ for almost every $\bsz$ in the support of $\calD_i$; 
        \item[(\bdvar)] \emph{Bounded gradient variance at optimum.} There exists a positive constant $\sigma$ such that for any $i\in[m]$, we have $\bbE_{Z_i\sim\calD_i}  \|\nabla\ell(\www{i}_\star, Z_i)\|^2 \leq \sigma^2$.
   \end{enumerate}
\end{assumption}

\paragraph{Heterogeneity Conditions.}
To quantify the level of heterogeneity among clients, we start by introducing the notion of an \emph{average global model}. Assuming a strongly convex loss, the optimal local models \eqref{eq:opt_local} are uniquely defined. Thus, we can define the average global model as
\begin{equation}
    \label{eq:p_avg_glob_model}
    \www{\glob}_\bp =  \sum_{i\in[m]} p_i \www{i}_\star.
\end{equation} 
We remark that the average global model defined in \eqref{eq:p_avg_glob_model} should \emph{not} be interpreted as the ``optimal global model''.
Rather, it is more suitable to think of $\www{\glob}_\bp$ as a point in the parameter space, from which every local model is close to. 
Indeed, one can readily check that the average global model is the minimizer of $\sum_{i\in[m]} p_i \|\www{i}_\star - \boldsymbol{w}\|^2$ over $\boldsymbol{w} \in \bbR^d$. 

We are now ready to quantify the level of client-wise heterogeneity as follows.
\begin{assumption}[Level of heterogeneity]
\label{assump:heterogeneity}
There exists a positive constant $R$ such that
\begin{enumerate}
    \item[(\aersim)] either $\sum_{i\in[m]}p_i \| \www{i}_\star  - \www{\glob}_\bp\|^2 \leq R^2$,
    \item[(\iersim)] or $\|\www{i}_\star - \www{\glob}_\bp\|^2\leq R^2 ~ \forall i\in[m]$.
\end{enumerate} 
\end{assumption}
Our study of the $\aer_\bp$~and $\ier_i$s~will be based on Part (\aersim) and (\iersim) of Assumption~\ref{assump:heterogeneity}, respectively. Intuitively, the quantity $R$ encodes one's belief on ``how heterogeneous'' the clients can be.

\section{Main Results}\label{sec:main_res}

\subsection{Analyses of Two Baseline Algorithms}\label{subsec:ub}
In this subsection, we characterize the performance of \textsc{PureLocalTraining} and \textsc{FedAvg} under the heterogeneity conditions imposed by Assumption \ref{assump:heterogeneity}. 



\subsubsection{Warm Up: Uniform Stability and Analysis of \textsc{PureLocalTraining}}

The analysis of \textsc{PureLocalTraining} is based on the classical notion of uniform stability, proposed by \cite{bousquet2002stability}.
\begin{definition}[Uniform stability]
\label{def:unif_stab}
Consider an algorithm $\calA$ that takes a {single dataset} $S = \{\bsz_j\}_{j=1}^n$ of size $n$ as input and outputs a \emph{single model}: $\calA(S) = \hat\bsw(S)$. We say $\calA$ is $\gamma$-uniformly stable if for any dataset $S$, any $j\in[n]$, and any $\bsz_j'\in\calZ$, we have
\begin{equation*}
    \|\ell(\hat \bsw(S), \cdot) - \ell(\hat \bsw(\SSS{j}), \cdot)\|_\infty \leq \gamma,
\end{equation*}
where $\SSS{j}$ is the dataset formed by replacing $\bsz_j$ with $\bsz'_j$:
\begin{equation*}
    \SSS{j} = \{\bsz_{1}, \hdots, \bsz_{j-1}, \bsz_{j}', \bsz_j, \hdots, \bsz_{n}\}.
\end{equation*}
\end{definition}

The main implication of uniformly stable algorithms is that ``stable algorithms do not overfit'': if $\calA$ is $\gamma$-uniformly stable, then its \emph{generalization error} is upper bounded by a constant multiple of $\gamma$. Thus, one can dissect the analysis of $\calA$ into two separate parts: (1) bounding its optimization error; (2) bounding its stability term.

Under our working assumptions, SGD with properly chosen step sizes is guaranteed to converge to the global minimum of \eqref{eq:pure_local_training} (see, e.g., \cite{rakhlin2011making}). 
Note that the bounds for the approximate minimizers only involve an extra additive term representing the optimization error, and this term will be negligible if we run SGD until convergence since our focus is sample complexity.
Thus, we conduct the analysis for the global minimizer of \eqref{eq:pure_local_training}. 
The performance of \textsc{PureLocalTraining} is given by the following theorem.

\begin{theorem}[Performance of \pllong]
    \label{thm:PLT}
    Let Assumption \ref{assump:regularity}(\cvx) hold and assume $n_i \geq 4\beta/\mu ~\forall i\in[m]$. Then the algorithm $\calA_\PLT$ which outputs the minimizer of \eqref{eq:pure_local_training} satisfies
    \begin{equation*}
        \label{eq:ier_PLT}
        \bbE_{\bS}[\ier_i(\calA_\PLT)] \lesssim \frac{\beta \|\ell\|_\infty}{\mu n_i}
    \end{equation*}
for all $i = 1, \ldots, m$.
\end{theorem}
\begin{proof}
    The proof is a direct consequence of standard results on uniform stability of strongly convex ERM (see, e.g., Section 5 of \cite{shalev2009stochastic} and Section 13 of \cite{shalev2014understanding}), which assert that under the current assumptions, the minimizer of \eqref{eq:pure_local_training} is $\calO\big(\frac{\beta\|\ell\|_\infty}{\mu n_i}\big)$-uniformly stable. We omit the details.
\end{proof}
By definition, for any weight vector $\bp$, $\aer_\bp$ of \textsc{PureLocalTraining} also admits the same upper bound as \eqref{eq:ier_PLT}.

\subsubsection{Federated Stability and Analysis of \textsc{FedAvg}}\label{subsubsec:fed_avg}
We consider the following weighted version of \eqref{eq:hard_sharing}:
\begin{equation}
    \label{eq:wted_hard_sharing}
    \min_{\bsw\in\calW }  \sum_{i\in[m]} p_i L_i(\bsw, S_i). 
\end{equation}
The \textsc{FedAvg} algorithm (Algorithm \ref{alg:fedavg}) also seamlessly generalizes. The above optimization formulation is in fact covered by the general theory of \cite{li2020convergence}, where they showed that \falong~is guaranteed to converge to the global optimum under a suitable hyperparameter choice, even in the presence of heterogeneity (but the convergence is slower). Thus, in the following discussion, we again consider the global minimizer of \eqref{eq:wted_hard_sharing}. 

It turns out that a tight analysis of \textsc{FedAvg} requires a more fine-grained notion of uniform stability, which we present below.
\begin{definition}[Federated stability]
\label{def:fed_stab}
An algorithm $\calA$ that outputs $\calA(\bS) = \{\hwww{i}(\bS)\}$ has \emph{\fedstab} $\{\gamma_i\}_{i=1}^m$ if for every $\bS\sim \bigotimes_i\calD_i^{\otimes n_i}$ and for any $i\in[m], j_i\in[n_i], z'_{i, j_i}\in \calZ$, we have 
\begin{equation*}
    \|\ell(\hwww{i}(\bS), \cdot) - \ell(\hwww{i}(\bSSS{i, j_i}), \cdot)\|_\infty \leq \gamma_i.
\end{equation*}
Above, $\bSSS{i, j_i}$ is the dataset formed by replacing $\zzz{i}_{j_i}$ in the $i$-the dataset with $z_{i, j_i}'$:
\begin{equation*}
  \begin{aligned}
&    \bSSS{i, j_i} = \{S_1, \hdots, S_{i-1}, \SSS{j_i}_i, S_{i+1}, \hdots, S_m\}, \\
&\SSS{j_i}_i = \{\zzz{i}_1, \hdots, \zzz{i}_{j_i-1}, z_{i, j_i}', \zzz{i}_{j_i+1}, \hdots, \zzz{i}_{n_i}\}.
  \end{aligned}
\end{equation*}
\end{definition}

Compared to the conventional uniform stability in Definition \ref{def:unif_stab}, \fedstab~provides a finer control by allowing distinct stability measures $\{\gamma_i\}$ for different clients. 
Moreover, the classical statement that ``stable algorithms do not overfit'' still holds, in the sense that the average (resp. individualized) generalization error can be upper bounded by $\calO(\sum_{i\in[m]}n_i \gamma_i/N)$ (resp. $\calO(\gamma_i)$), plus a term scaling with the level of heterogeneity $R$. And this again enables us to separate the analysis of $\calA$ into two parts (namely bounding the optimization error and bounding the stability), as is the case with the conventional uniform stability. 

The notion of federated stability has other implications when restricted to the \textsc{FedProx} algorithm, and we refer the readers to Section \ref{subsec:fedprox} for details. 

We are now ready to state the theorem that characterizes the performance of \falong.
\begin{theorem}[Performance of \textsc{FedAvg}]
    \label{thm:FA}
    Let Assumption \ref{assump:regularity}(\cvx, \bdvar) hold and assume $n_i\geq 4\beta p_i/\mu ~\forall i\in[m]$. Suppose the \textsc{FedAvg} algorithm $\calA_\FA$ outputs the minimizer of \eqref{eq:wted_hard_sharing}. Then under Assumption \ref{assump:heterogeneity}(\aersim), we have
    \begin{align}
        \label{eq:aer_FA}  
        \bbE_{\bS} [\aer_\bp(\calA_{\FA})] & \lesssim \frac{\beta\|\ell\|_\infty}{\mu} \sum_{i\in[m]}\frac{p_i^2}{n_i} + \beta R^2, 
    \end{align}
    and under Assumption \ref{assump:heterogeneity}(\iersim), we have
    \begin{align}
        \label{eq:ier_FA}  
        \bbE_{\bS} [\ier_i(\calA_\FA)] & \lesssim  \frac{\beta \sigma^2}{\mu^2 } \sum_{i'\in[m]} \frac{p_{i'}^2}{n_{i'}} + \frac{\beta^3}{\mu^2} R^2. 
    \end{align}
\end{theorem}
\begin{proof}
    The proof of \eqref{eq:aer_FA} is, roughly speaking, based on the fact that the global minimizer of \eqref{eq:wted_hard_sharing} has \fedstab~$\gamma_i \lesssim \frac{\beta\|\ell\|_\infty p_i}{\mu n_i}$, and thus the first term in the right-hand side of \eqref{eq:aer_FA} corresponds to the average \fedstab~$\sum_{i\in[m]}p_i \gamma_i$. 
    The second term $\beta R^2$ in the right-hand side of \eqref{eq:aer_FA} reflects the presence of heterogeneity. For Equation \eqref{eq:ier_FA}, we were not able to obtain a \fedstab~based proof, and our current proof is based on an adaptation of the arguments in Theorem 7 of \cite{foster2019complexity}, which explains why the dependence on $(\sigma, \beta, \mu)$ are different (and slightly worse) compared to Equation \eqref{eq:aer_FA}. In particular, the bound \eqref{eq:ier_FA} has inverse quadratic dependence on $\mu$, wheres the bound \eqref{eq:aer_FA} only has $1/\mu$ dependence. The $1/\mu$ dependence comes from the fact that the \fedstab~term has such dependence, and the $1/\mu^2$ dependence comes from the fact that the $\ell_2$ estimation error has such dependence.
    We refer the readers to Appendix \ref{prf:thm:FA} for details.
\end{proof}

Note that both bounds in the above theorem are minimized by choosing $p_i = n_i/N$. With this choice of $\bp$, the two bounds read
\begin{align}
    \bbE_{\bS}[\aer_\bp(\calA_\FA)] \lesssim \frac{\beta\|\ell\|_\infty}{\mu N } + \beta R^2,
    \qquad
    \bbE_{\bS}[\ier_i(\calA_\FA)] \lesssim \frac{\beta\sigma^2}{\mu^2 N}   + \frac{\beta^3 R^2}{\mu^2}.
\end{align}
This makes sense, since this choice of weight corresponds to the ERM objective under complete homogeneity. This observation also suggests that ensuring ``fair resource allocation'' (i.e., setting $p_i = 1/m$) can lead to statistical inefficiency, especially when the sample sizes are imbalanced. 


We conclude this subsection by noting that though the compactness assumption (Assumption \ref{assump:regularity}(\compact)) is not needed in Theorem \ref{thm:FA},
it is usually needed in the analysis of the optimization error of \falong~and \pllong (see, e.g., \cite{rakhlin2011making,li2020convergence}).

\subsection{Lower Bounds}\label{subsec:lb}
In this subsection, we present our construction of lower bounds, which characterize the information-theoretic limit of personalized federated learning. Throughout this section, we restrict out attention to the case where $p_i = n_i/N$ for any $i\in[m]$.

Our construction starts by considering a special class of problem instances: {logistic regression}. In logistic regression, given the collection of regression coefficients $\{\www{i}_\star\} \subseteq \calW$ where $\calW$ has a diameter $D$, the data distributions $\calD_i$'s are supported on $\bbR^d\times \{\pm 1\}$ and specified by a two-step procedure as follows:
\begin{enumerate}
    \item Generate a {feature vector} $\bsx$, whose coordinates are i.i.d.~copies from some distribution $\bbP_X$ on $\bbR$, which is assumed to have mean zero and is almost surely bounded by some absolute constant $c_X$;
    \item Generate the binary {label} $y \in\{\pm 1\}$, which is a biased Rademacher random variable with head probability $\big({1+ \exp\{-\bsx^\top \www{i}_\star\}}\big)^{-1}$.
\end{enumerate}
The loss function is naturally chosen to be the negative log-likelihood function, which takes the following form:
\begin{equation*}
    \ell(\bsw, \bsz) = \ell(\bsw, \bsx, y) = \log(1+ e^{-y \bsx^\top \bsw}).
\end{equation*}

The following lemma says that Assumption \ref{assump:regularity} holds for the aforementioned logistic regression models.
\begin{lemma}[Logistic regressions are valid problem instances]
\label{lemma:logistic_instance}
    The logistic regression problem described above is a class of problem instances that satisfies Assumption \ref{assump:regularity} with $\|\ell\|_\infty = c_XD\sqrt{d}$ and $\sigma^2 = \beta =c_X^2 d/4$. Moreover, if $m \lesssim (N/m)^c$ for some $c \geq 0$ and $N/m\geq Cd$ for some $C>1$, then there exists some event $\sfE$ which only depends on the features $\{\xxx{i}_j: i\in[m], j\in[n_i]\}$ and happens with probability at least $1-e^{-{\calO(\sqrt{N/m})}}$, such that on this event, the strongly convex constant in Assumption \ref{assump:regularity} satisfies  
    \begin{equation}
        \label{eq:sc_const_logistic}
        \mu \asymp \mu_0 = \bigg(\exp\{c_X D\sqrt{d}/2\} + \exp\{-c_X D \sqrt{d}\}\bigg)^{-2}.
    \end{equation}        
\end{lemma}


\begin{proof}
The compactness of the domain and the boundedness of the loss function hold by construction. 
To verify the rest parts of Assumption \ref{assump:regularity}, with some algebra one finds that
\begin{equation}
\label{eq:logistic_hessian}
    \nabla^2 \ell(\bsw, \bsx, y) = \frac{\bsx \bsx^\top \exp\{y \bsx^\top \bsw\}}{\big(1 + \exp\{y \bsx^\top \bsw\}\big)^2}\preceq \frac{1}{4} \bsx \bsx^\top,
\end{equation}    
where $\preceq$ is the Loewner order and the inequality holds because $x/(1+x)^2 = 1/(x^{-1/2} + x^{1/2})^2\leq 1/4$ for $x > 0$.
Since the population gradient has mean zero at optimum, the gradient variance at optimum can be upper bounded by the trace of the expected Hessian matrix, which, by the above display, is further upper bounded by $c_X^2 d/4$. Thus, we can take $\sigma^2 = c_X^2 d/4$ in Part (\bdvar).  
Another message of the above display is that we can set the smoothness constant in Part (\cvx) to be $\beta = c_X^2 d/4$. 

The only subtlety that remains is to ensure each local loss function is $\mu$-strongly convex. Note that since $x/(1+x)^2$ is decreasing from $(0, 1)$ and is increasing from $(1, \infty)$, the right-hand side of \eqref{eq:logistic_hessian} dominates $\mu_0 \bsx \bsx^\top$ in Loewner order, where 
    $
    \mu_0
    $
is the right-hand side of \eqref{eq:sc_const_logistic}.
Thus, the local population losses $\bbE_{(\bsx, y) \sim \calD_i}[\ell(\cdot, \bsx, y)]$ are all $\mu_0$-strongly convex. 

Now, note that
\begin{align*}
    \nabla^2 L_i(\www{i}, S_i) & = \frac{1}{n_i} \sum_{j\in[n_i]} \frac{\xxx{i}_j (\xxx{i}_j)^\top \exp\{\yyy{i}_j \la \xxx{i}_j , \www{i} \ra\}}{\big(1+ \exp\{\yyy{i}_j \la \xxx{i}_j , \www{i} \ra\}\big)^2} \succeq \mu_0 \cdot \frac{1}{n_i}\sum_{j\in[n_i]} \xxx{i}_j (\xxx{i}_j)^\top.
\end{align*}    
Invoking Theorem 5.39 of \cite{vershynin2010introduction} along with a union bound over all clients, we conclude that for any $i\in[m]$, the minimum eigenvalue of $\sum_{j\in[n_i]} \xxx{i}_j (\xxx{i}_j)^\top$ is lower bounded by a constant multiple of $n_i - p \gtrsim n_i$ (this is the definition of the event $\sfE$) with probability at least $1-me^{-\calO(n_i)}\geq  1 - e^{-\calO(\sqrt{N/m})}$, 
and the proof is concluded.
\end{proof}

Note that in the proof of the above lemma, we have established the $\mu_0\asymp \mu$-strong convexity of the client-wise population losses. Hence, lower bounding the excess risks reduces to lower bounding the \emph{$\ell_2$ estimation errors} $\|\hwww{i} - \www{i}_\star\|^2$ of the estimators $\hwww{i}$ for $\www{i}_\star$. Such a reduction allows us to use powerful tools from information theory. 

To this end, we introduce two parameter spaces, corresponding to Part (\aersim) and (\iersim) of Assumption \ref{assump:heterogeneity}. Recalling $\www{\glob}_{\bp} = \sum_{i\in[m]}p_i \www{i}_\star$, we define
\begin{align*}
    \calP_1& := \bigg\{\{\www{i}_\star\}_{i=1}^m \subseteq \calW: \sum_{i\in[m]}p_i \|\www{i}_\star - \www{\glob}_{\bp}\|^2 \leq  R^2 \bigg\},\\
    \calP_2 &:= \bigg\{\{\www{i}_\star\}_{i=1}^m \subseteq \calW:  \|\www{i}_\star - \www{\glob}_{\bp}\|^2 \leq R^2 ~\forall i\in[m] \bigg\}.
\end{align*}
Note that $\calP_1$ and $\calP_2$ index all possible values of $\{\www{i}_\star\}$ that can arise in the logistic regression models under Assumption \ref{assump:heterogeneity} (\aersim) and (\iersim), respectively. 

With the notations introduced so far, we are ready to state the main result of this subsection.
\begin{theorem}[Minimax lower bounds for estimation errors]
\label{thm:lb_logistic}
    Consider the logistic regression model described above. Suppose $n_i\asymp n_{i'}$ for any $i\neq i'\in[m]$ and assume $p_i = n_i/N$ for any $i\in[m]$. Then we have
    \begin{align}
        \label{eq:lb_logistic_aer}
        \inf_{\{\hwww{i}\}} \sup_{\{\www{i}_\star\}\in\calP_1} \frac{1}{N} \sum_{i\in[m]} n_i \bbE_{\bS}\|\hwww{i} - \www{i}_\star\|^2 &\gtrsim \frac{d}{N/m} \land R^2 + \frac{d}{N},\\
        \label{eq:lb_logistic_ier}
        \inf_{\hwww{i}} \sup_{\{\www{i}_\star\}\in\calP_2} \bbE_{\bS}\|\hwww{i} - \www{i}_\star\|^2 & \gtrsim  \frac{d}{n_i} \land R^2 + \frac{d}{N}
    \end{align}
for all $i \in[m]$, where the infimum is taken over all possible $\hwww{i}$s that are measurable functions of the data $\bS$.
\end{theorem}
\begin{proof}
    See Appendix \ref{prf:thm:lb_logistic}.
\end{proof}

Note that both lower bounds in Theorem \ref{thm:lb_logistic} are a superposition of two terms, and they correspond to two distinct steps in the proof. 

The first step in our proof is to argue that the lower bound under complete homogeneity is in fact a valid lower bound under our working assumptions, which gives the $\Omega(d/N)$ term. This is reasonable, since estimation under complete homogeneity is, in many senses, an ``easier'' problem. The proof of the $\Omega(d/N)$ term is based on the classical Assouad's method \citep{assouad1983deux}. 

The second step is to use a generalized version of Assouad's method that allows us to deal with multiple heterogeneous datasets. In particular, we need to carefully choose the prior distributions over the parameter space based on the level of heterogeneity, which ultimately leads to the $\Omega(\frac{d}{N/m} \land R^2)$ term. Recall that in the vanilla version of Assouad's method where there is only one parameter, say $\bsw_\star$, one can lower bounds the minimax risk by the Bayes risk, and the prior distribution is usually chosen to be $\bsw_\star  = \delta \bsv$, where $\bsv$ follows a uniform distribution over all $d$-dimensional binary vectors and $\delta$ is chosen so that the resulting hypothesis testing problem has large type-\RN{1} plus type-\RN{2} error. In our case where there are $m$ parameters $\{\www{i}_\star\}$, we need to consider a different prior of the following form: 
\begin{equation*}
\www{i}_\star = \delta_i \vvv{i},
\end{equation*}
where $\vvv{i}$ are i.i.d.~samples from the uniform distribution over all $d$-dimensional binary vectors, and $\delta_i$'s are scalers that need to be carefully chosen to make the resulting hypothesis testing problem hard. 

The following result is an immediate corollary of Theorem \ref{thm:lb_logistic}.

\begin{corollary}[Minimax lower bounds for excess errors]
\label{cor:lb}
Assume there exist constants $C, C'>0, c\geq 0$ such that $n_i\geq C\beta~\forall i\in[m]$ and $m \leq C' (N/m)^{c}$. Moreover, assume $n_i\asymp n_{i'}$ for any $i\neq i'\in[m]$ and $p_i = n_i/N$ for any $i\in[m]$.
Then there exists an absolute constant $c'$ such that the following two statements hold:

1. There exists a problem instance such that Assumptions \ref{assump:regularity} and \ref{assump:heterogeneity}(\aersim) are satisfied with probability at least $1- e^{-c'\sqrt{N/m}}$. Call this high probability event $\mathsf{E}$. On this problem instance, any randomized algorithm $\calA$ must suffer
\begin{equation}
    \label{eq:lb_aer}
    \bbE_{\calA, \bS}[\aer_\bp(\calA)\cdot \indc_{\sfE}] \gtrsim \mu \cdot \bigg(\frac{\beta}{N/m} \land R^2 + \frac{\beta}{N}\bigg);
\end{equation}  

2. For any $i\in[m]$, there exists a problem instance such that Assumptions \ref{assump:regularity} and \ref{assump:heterogeneity}(\iersim) are satisfied with probability at least $1- e^{-c'\sqrt{N/m}}$. Call this high probability event $\sfE_i$. On this problem instance, any randomized algorithm $\calA$ must suffer
\begin{equation}
    \label{eq:lb_ier}
    \bbE_{\calA, \bS}[\ier_i(\calA)\cdot \indc_{\sfE_i}] \gtrsim \mu \cdot \bigg(\frac{\beta}{n_i} \land R^2 + \frac{\beta}{N}\bigg).
\end{equation}  
In the two displays above, the expectation is taken over the randomness in both the algorithm $\calA$ and the sample $\bS$.
\end{corollary}
\begin{proof}
    Along with Lemma \ref{lemma:logistic_instance} and Theorem \ref{thm:lb_logistic}, this corollary follows by the fact that the smoothness constant $\beta$ is of the same order as $d$ and the population losses are all $\mu_0\asymp\mu$-strongly convex.
\end{proof}



\subsection{Implications of the Main Results} \label{subsec:alternative}
The upper bounds in Section \ref{subsec:ub} and the lower bounds in Section \ref{subsec:lb} together reveal several intriguing phenomena regarding personalized FL, which we detail in this subsection.

Focusing on the dependence on the sample sizes and assuming the client-wise samples sizes are balanced (i.e., $n_i \asymp N/m$), the heterogeneity measure $R$ enters the lower and upper bounds in a {dichotomous} fashion: 
\begin{itemize}
    \item If $R^2\lesssim {m/N}$, then both lower bounds become $\Omega(R^2 + 1/N)$, and this lower bound can be attained by \textsc{FedAvg} up to factors that do not depend on the sample sizes; 
    \item If $R^2 \gtrsim {m/N}$, then both lower bounds become $\Omega(m/N)$. They agree with the minimax rate {as if we were under complete heterogeneity} and can be achieved by \textsc{PureLocalTraining}.
\end{itemize}

Now, let us consider the following \naive~dichotomous strategy: if output $R^2 \leq \frac{\|\ell\|_\infty}{\mu} \cdot m/N$, then output $\calA = \calA_\FA$; otherwise, output $\calA = \calA_\PLT$. That is, we switch between the two baseline algorithms at the threshold of $R^2 \asymp m/N$.
Then under the assumptions in Theorems \ref{thm:PLT} and \ref{thm:FA}, one can readily check that this dichotomous strategy satisfies the following AER guarantee:
\begin{align}
    \label{eq:aer_dichotomous_wted}
    \bbE_{\bS}[\aer_\bp(\calA)] & \lesssim \beta\bigg(\frac{\|\ell\|_\infty}{\mu N/m} \land R^2\bigg) + \frac{\beta \|\ell\|_\infty}{\mu} \sum_{i\in[m]} \frac{p_i^2}{n_i}.  
\end{align}
If in addition, $n_i\asymp n_{i'}$ for any $i\neq i' \in[m]$, then it also satisfies the following IER guarantee: 
\begin{align}
    \label{eq:ier_dichotomous_wted}
    \bbE_{\bS}[\ier_i (\calA)] & \lesssim \frac{\beta^3}{\mu^2} \bigg(\frac{\|\ell\|_\infty}{\mu n_i} \land R^2\bigg) + \frac{\beta \sigma^2}{\mu^2} \sum_{i'\in[m]} \frac{p_{i'}^2}{n_{i'}}. 
\end{align}
When $p_i = n_i/N$, the two displays above simplify to 
\begin{align*}
    \bbE_{\bS}[\aer(\calA)] \lesssim \beta \bigg(\frac{\|\ell\|_\infty}{\mu N/m} \land R^2\bigg) + \frac{\beta \|\ell\|_\infty}{\mu N},
    \qquad
    \bbE_{\bS}[\ier_i(\calA)]  \lesssim \frac{\beta^3}{\mu^2} \bigg(\frac{\|\ell\|_\infty}{\mu n_i} \land R^2\bigg) + \frac{\beta\sigma^2}{\mu^2 N },
\end{align*}
which matches the lower bound in Corollary \ref{cor:lb} up to constant factors, provided $(\beta,\mu, \|\ell\|_\infty, \sigma)$ are all of constant order. In other words, switching between the two algorithms at the threshold of $R^2 \asymp m/N$ gives an oracle algorithm that is minimax rate optimal. 

Thus, we have shown an interesting property for personalized FL on the choice of the two baseline algorithms. In particular, consider a collection of problem instances indexed by $(R, \beta, \mu, \|\ell\|, \sigma)$ using Assumptions \ref{assump:regularity} and \ref{assump:heterogeneity} and assume $(\beta,\mu, \|\ell\|_\infty, \sigma)$ are all of constant order. Now, for a fixed value of $R$, \emph{exactly one of these two algorithms is minimax optimal},
where the optimality is defined over the specified collection of problem instances and with respect to both AER and IER. Moreover, the oracle dichotomous strategy that switches between the two baseline algorithms at the threshold of $R^2 \asymp m/N$ is minimax optimal. 

More implications of the theoretical results are described below.

\paragraph{Optimality of a dichotomous strategy.}
From the practical side, for supervised learning problems, such a dichotomous strategy can be implemented without prior knowledge of $R$ if test errors can be evaluated in a distributed fashion. Indeed, we can first run both \textsc{FedAvg} and \textsc{PureLocalTraining} separately, evaluate their test errors (in a distributed fashion), and deploy the one with a lower test error. 
Due to the upper and lower bounds proved in Sections \ref{subsec:ub} and \ref{subsec:lb}, such a strategy is guaranteed to be minimax rate optimal. 
As a caveat, however, one should refrain from interpreting our results as saying either of the two baseline algorithms is sufficient for practical problems. 
From a practical viewpoint, constants that are omitted in the minimax analysis are crucial.
Even for supervised problems, a better personalization result could be achieved by more sophisticated algorithms in practice.
Nevertheless, our results suggest that the two baseline algorithms can at least serve as a good starting point in the search for efficient personalized algorithms.

For unsupervised problems where the quality of a model is hard to evaluate, implementing the dichotomous strategy requires estimating an upper bound $R$ of the level of heterogeneity. This is an important open problem, which we leave for future work. 

\paragraph{Optimality of \textsc{FedAvg} followed by local fine tuning.}
Another popular baseline algorithm for personalized FL is to first run \textsc{FedAvg} until convergence, and then let each client run \textsc{PureLocalTraining} to fine tune the model. 
In strongly convex problems, global optima can be reached by gradient descent regardless the initialization with a suitable choice of the learning rate (see, e.g., Theorem 2.1.15 of \citealt{nesterov2018lectures}). 
Thus, if each client run \textsc{PureLocalTraining} for long enough, the global optima for its local loss function will finally be reached. 
This fact tells that along the whole fine tuning trajectory, there is a point at which the model gives the worst-case optimal AER and IER, and for a fixed level of heterogeneity, this point is either at the very beginning (which is \textsc{FedAvg}), or at the very end (which is \textsc{PureLocalTraining}).
Although this conclusion is almost trivial from a technical point of view given our minimax results, it provides a reassuring theoretical property (of being minimax optimal) for a popular method used by practitioners.

\begin{figure}[t!]
\centering
    \begin{subfigure}{.48\textwidth}
        \centering
        \includegraphics[width=1\linewidth]{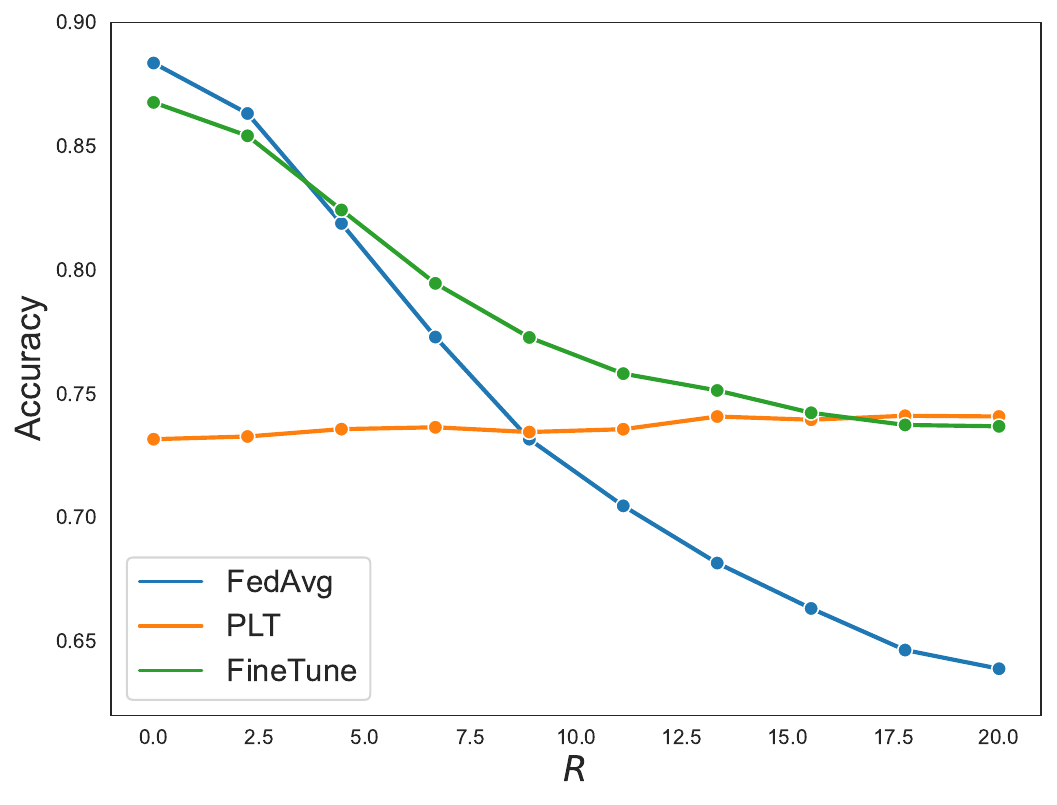}
    \end{subfigure}
    \begin{subfigure}{.48\textwidth}
        \centering 
        \includegraphics[width=1\linewidth]{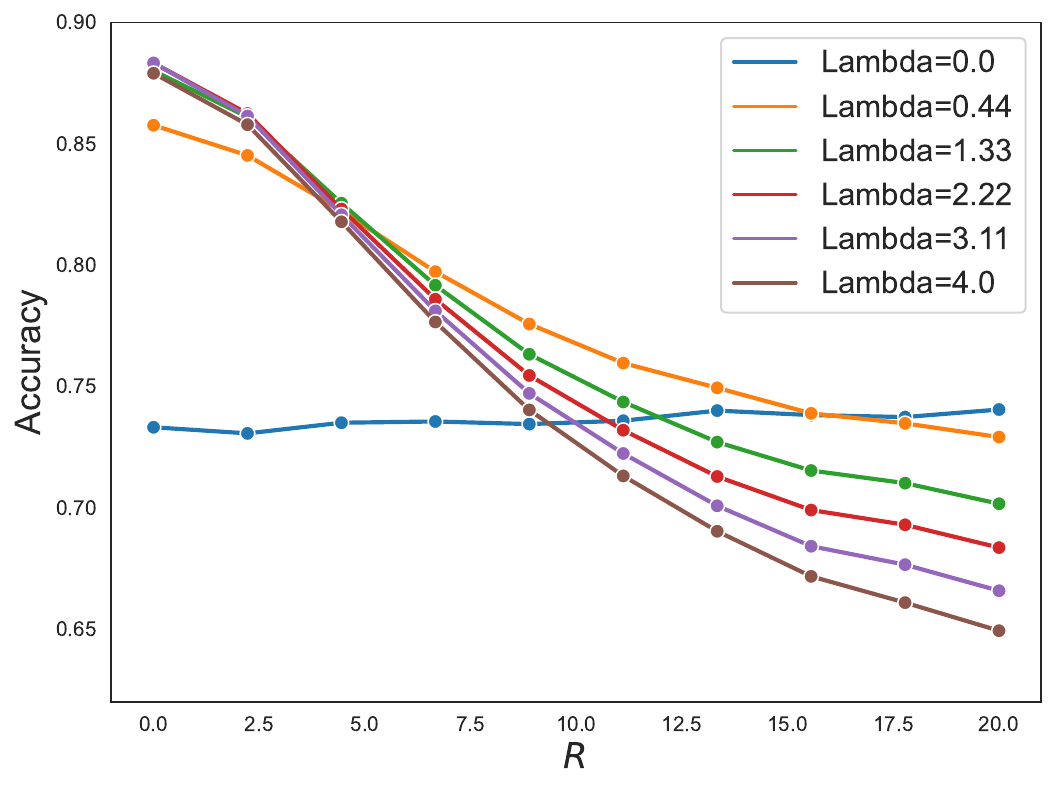}
    \end{subfigure}
       \caption{\small Average classification accuracy of \textsc{FedAvg}, \textsc{PureLocalTraining} and \textsc{FedAvg} followed by fine tuning (left panel) as well as \textsc{FedProx} with different choice of $\lambda$ (right panel).}
    \label{fig:sim}
\end{figure}

\paragraph{Illustrating the minimaxity in a simulated example.}
We conduct a simulation on federated logistic regression to corroborate our theoretical results and the optimality of the \textsc{FedAvg} following by local fine tuning strategy. In the simulation, we set $m=5$, $n_i=100, \forall i \in[m]$ and we vary $R$ from $0$ to $20$ (see Appendix \ref{appx:exp} for details). In the left panel of Figure \ref{fig:sim}, we plot the test accuracy (averaged over $100$ rounds of simulations) of those three methods against the value of $R$. One can see that the accuracy of the fine tuning strategy roughly follows the maximum of the accuracies achieved by \textsc{FedAvg} and \textsc{PureLocalTraining}, confirming our theoretical prediction that the fine tuning strategy can indeed perform as well as the best between \textsc{FedAvg} and \textsc{PureLocalTraining}.

\paragraph{Beyond the current heterogeneity assumption.}
Our minimax results are established under Assumption \ref{assump:heterogeneity}, which states that all optimal local models are close to a certain ``centroid'' (i.e., the average global model defined in \eqref{eq:p_avg_glob_model}). If we draw a graph of clients and connect two clients if their optimal local models are similar, then the current heterogeneity assumption gives rise to a complete graph (or a star-shaped graph if we introduce another node to represent the average global model). While such an idealized graph structure enables a clean theoretical analysis, the real world proximity patterns among clients are clearly far more sophisticated. In fact, counterexamples exist under which the minimax results does not hold in a ``global'' sense.

Suppose all $m$ clients exhibit a clustering structure as follows. We have $\sqrt{m}$ clients whose optimal global models serve as cluster centroids, and those centroids are very far apart. In the neighborhood of each centroid, there are $\sqrt{m}$ clients whose optimal local models are $R_c$ away from the centroid in the sense of Assumption \ref{assump:heterogeneity}. Additionally, assume each client has equal sample sizes, so that $n_i = n$ for some $n$.
Under this setting, the ``global'' heterogeneity parameter $R$ in Assumption \ref{assump:heterogeneity} is very large, so our theory would suggest choosing \textsc{PureLocalTraining}, which gives a rate of $\calO(1/n)$. However, this rate is clearly suboptimal.
If one can successfully cluster the $m$ clients into $\sqrt{m}$ clusters (which is hopeful as the centroids are assumed to be far apart), then one can apply our theory to each cluster (i.e., run the dichotomous strategy for each cluster) and conclude that the rate for each cluster is $\calO(\frac{1}{n\sqrt{m}} + \frac{1}{n}\land R_c^2) \ll \frac{1}{n}$ if $m$ diverges to infinity and $R_c \ll 1/\sqrt{n}$. The foregoing discussion reveals in such a clustered setting, our theoretical results only make sense at the cluster level, but not at the global level.

The behaviors of the minimax rate for more general client proximity graphs can be even more complicated, which we leave for future work.

\subsection{More Implications of Federated Stability and Analysis of \textsc{FedProx}}\label{subsec:fedprox}
In this subsection, we are concerned with the performance guarantees for \textsc{FedProx}. 
As with our earlier analysis of \textsc{FedAvg}, we consider a $\bp$-weighted version of \textsc{FedProx}, whose optimization formulation is given below:
\begin{equation}
    \label{eq:soft_sharing_wted}
    \min_{\substack{\www{\glob}\in\calW \\ \{\www{i}\}_{i=1}^m \subseteq \calW}} \sum_{i\in[m]} p_i  \bigg(L_i(\www{i}, S_i) + \frac{\lambda}{2}\| \www{\glob} - \www{i}\|^2\bigg),
\end{equation}
where we recall that $L_i(\bsw, S_i) := \sum_{j\in[n_i]} \ell(\bsw, \zzz{i}_j) / n_i$ is the ERM objective for the $i$-th client. 
In this subsection, we let $(\twww{\glob}, \{\twww{i}\})$ be the global minimizer of the above problem. 
Compared to \eqref{eq:wted_hard_sharing}, which imposes a ``hard'' constraint $\www{i} = \www{\glob}$, and compared to \eqref{eq:pure_local_training}, where there is no constraint at all, the above formulation imposes a ``soft'' constraint that the norm of $(\www{\glob} - \www{i})$ should be small, with a hyperparameter $\lambda$ controlling the strength of this constraint. 

The rationale behind the optimization formulation \eqref{eq:soft_sharing_wted} of \textsc{FedProx} is clear: by setting $\lambda = 0$, the optimization formulation of \pllong~\eqref{eq:pure_local_training} is recovered, and as $\lambda \to \infty$, the optimization formulation of \falong~\eqref{eq:wted_hard_sharing} is recovered. 
The hope is that by varying $\lambda \in (0, \infty)$, one can interpolate between the two extremes.



\begin{algorithm}[t]
\DontPrintSemicolon
\SetKwFunction{SoftLocalSGD}{SoftLocalSGD}
\KwIn{Initial global model $\www{\glob}_{0}$, initial local models $\{\www{i}_{0}\}_{i=1}^m \equiv\{\www{i}_{0, 0}\}_{i=1}^m$, global rounds $T$, global batch size $\BBB{\glob}$, global step sizes $\{\etasup{\glob}_t\}_{t=0}^{T-1}$, local rounds $\{K_t\}_{t=0}^{T}$, local batch sizes $\{\BBB{i}\}_{i=0}^m$ local step sizes $\{\etasup{i}_{t, k}: 0\leq t\leq T, 0\leq k\leq K_t-1\}$.}
\KwOut{Local models $\{\www{i}_{T+1}\}_{i=1}^m$.}
\texttt{{\# Stage \RN{1}: joint training}} \\
\For{$t=0, 1, \hdots, T-1$}{
    Randomly sample a batch $\calC_t\subseteq[m]$ of size $\BBB{\glob}$\\ 
    \For{$i\in[m]$}{
        \lIf{$i\in\calC_t$}{$\www{i}_{t+1} \leftarrow \www{i}_t$} 
        \Else{ 
            Pull $\www{\glob}_t$ from the server\;
            $\www{i}_{t+1} \leftarrow\SoftLocalSGD(i, \www{i}_t, \www{\glob}_t, K_t, \BBB{i}, \{\etasup{i}_{t, k}\}_{k=0}^{K_t-1})$\;
            Push $\www{i}_{t+1}$ to the server\;
        }
    }
    $\www{\glob}_{t+1}\leftarrow \www{\glob}_t - \frac{\lambda m\etasup{\glob}_t}{\BBB{\glob}} \sum_{i\in\calC_t}p_i(\www{\glob}_t - \www{i}_{t+1})$\; 

}
\texttt{{\# Stage \RN{2}: final training before deployment}} \\
\For{$i\in[m]$}{
    Pull $\www{\glob}_T$ from the server\;
    $\www{i}_{T+1} \leftarrow \SoftLocalSGD(i, \www{i}_T, \www{\glob}_T, K_T, \BBB{i}, \{\etasup{i}_{T, k}\}_{k=0}^{K_T-1})$\;
}
\KwRet{$\{\www{i}_{T+1}\}_{i=1}^m$}\;
\texttt{{\# Local SGD subroutine}} \\
\SetKwFunction{SoftLocalSGD}{SoftLocalSGD}
\SetKwProg{fn}{Function}{}{}
\fn{\SoftLocalSGD{$i, \www{i}, \www{\glob}, K, B, \{\eta_{k}\}_{k=0}^{K-1}$}}{
    \For{$k=0, 1, \hdots, K-1$}{
        Randomly sample a batch $I\subseteq[n_i]$ of size $|I| = B$\; 
        $\www{i}\leftarrow \calP_{\calW}\bigg[\www{i} -  \frac{\eta_{k}}{B}\sum_{j\in I} \bigg(\nabla \ell(\www{i}, \zzz{i}_j) + \lambda(\www{i} - \www{\glob})\bigg)\bigg]$\; 
    }
    \KwRet{$\www{i}$}\;
}
\caption{\sfalong, general version}
\label{alg:softfedavg_wted}
\end{algorithm}

Applying the idea of local SGD to \eqref{eq:soft_sharing_wted}, one obtains the \textsc{FedProx} algorithm\footnote{In fact, Algorithm \ref{alg:softfedavg_wted} is not exactly the same as the original \textsc{FedProx} algorithm introduced in \cite{li2018federated}. But since both algorithms share the idea of imposing regularization, we still call Algorithm \ref{alg:softfedavg_wted} \textsc{FedProx} for conceptual simplicity.}, which we detail in Algorithm \ref{alg:softfedavg_wted}.
We separate the whole algorithm into two stages as they has distinct interpretations:
in Stage \RN{1}, the central server aims to learn a good global model with the help of local clients, whereas in Stage \RN{2}, each local client takes advantage of the global model to personalize. Alternatively, one can also interpret \sfalong~as an instance of the general framework of \emph{model-agnostic meta learning} \citep{finn2017model}, where Stage \RN{1} learns a good {initialization}, and Stage \RN{2} trains the local models starting from this initialization.


In contrast to our analyses for \textsc{FedAvg} and \textsc{PureLocalTraining} in Section \ref{subsec:ub}, where we largely focused on global minimizers, the analysis for \sfalong~will be carried out for the approximate minimizer output by Algorithm \ref{alg:softfedavg_wted}. 
The reason for this is rooted in the tradeoff between the optimization error and the generalization error.
Note that given the results derived in Section \ref{subsec:ub}, the analysis for the global minimizer $(\twww{\glob}, \{\twww{i}\})$ becomes trivial: by setting $\lambda = 0$, we reduce the task to analyzing \textsc{PureLocalTraining}; by sending $\lambda \to \infty$, we reduce the task to analyzing \textsc{FedAvg}. Based on Theorems \ref{thm:PLT} and \ref{thm:FA}, one immediately concludes that there exists a choice of $\lambda$, such that the AER and IER of $\{\twww{i}\}$ satisfy the bounds in \eqref{eq:aer_dichotomous_wted} and \eqref{eq:ier_dichotomous_wted}, respectively.
However, the foregoing discussion is purely restricted to generalization error. When we set $\lambda = 0$ or send $\lambda \to\infty$, it is not known a priori whether \textsc{FedProx} algorithm will converge to the global minima. Worse still, the optimization error may depends on $\lambda$ in a particular way so that it becomes unbounded when $\lambda$ approaches zero or infinity.
To the best of our knowledge, prior work only proved the optimization convergence of \textsc{FedProx} for the global model with a fixed value of $\lambda$, namely the convergence of $\www{\glob}_T$ to $\twww{\glob}$ as the number of global communication rounds $T$ tends to infinity \citep{li2018federated,dinh2020personalized}.
To have a theoretical understanding of the performance of \textsc{FedProx}, it is crucial to (1) establish the optimization convergence for both global and local models; (2) bound the generalization error; and (3) balance the optimization error and the generalization error, both of which are functions of $\lambda$. In the following, we execute the those steps with the aid of federated stability.


\paragraph{Implications of federated stability for \textsc{FedProx}.} 
We have briefly mentioned the main implications of federated stability in Section \ref{subsubsec:fed_avg}: for an algorithm $\calA =\{\hwww{i}\}$ with \fedstab~$\{\gamma_i\}$, its {average generalization error} (resp. individualized generalization error) can be upper bounded by $\calO(\sum_{i\in[m]}p_i \gamma_i)$ (resp. $\calO(\gamma_i)$), plus a term scaling with the level of heterogeneity $R$. 
We make such a statement precise here. 
Let us first define the {optimization error} of a generic algorithm $\calA= (\hwww{\glob}, \{\hwww{i}\})$ (which tries to solve \eqref{eq:soft_sharing_wted}) as 
\begingroup
\allowdisplaybreaks[0] 
\begin{align*}
    \calE_{\opt} & := \sum_{i\in[m]} p_i \bigg( L_i(\hwww{i}, S_i) + \frac{\lambda}{2}\| \hwww{\glob} - \hwww{i}\|^2  \bigg) \nonumber\\
    & \qquad 
    -   \sum_{i\in[m]} p_i\bigg( L_i(\twww{i}, S_i) + \frac{\lambda}{2} \|\twww{\glob}  - \twww{i}\|^2\bigg).
\end{align*}
\endgroup

The main implications of \fedstab, when applied to the specifics of \fplong, can then be summarized in the following proposition.
\begin{proposition}[Implications of \fedstab~restricted to \fplong]
    \label{prop:fedstab_implication}
    Consider an algorithm $\calA=  (\hwww{\glob}, \{\hwww{i}\})$ with federated uniform stability $\{\gamma_i\}_1^m$. Then we have
    \begin{align}
    \label{eq:aer_via_stab}
    \bbE_{\calA, \bS}[\aer_\bp(\calA)]& \leq \bbE_{\calA, \bS}[\calE_{\opt}] + 2 \sum_{i\in[m]} p_i \bbE_{\calA, \bS}[\gamma_i]+ \frac{\lambda}{2} \sum_{i\in[m]} p_i\| \www{\glob}_{\avg} - \www{i}_\star \|^2,\\
    \label{eq:ier_via_stab}
    \bbE_{\calA, \bS}[\ier_i(\calA)] &\leq \frac{\bbE_{\calA, \bS}[\calE_\train]}{p_i} + 2\bbE_{\calA, \bS}[\gamma_i] + \frac{\lambda}{2} \bbE_{\calA, \bS}\| \hwww{\glob} - \www{i}_\star\|^2 ~~~\forall i\in[m].
    \end{align}
\end{proposition}
\begin{proof}
    The proof of \eqref{eq:aer_via_stab} is based on the following basic inequality for the AER: 
    \begin{equation}
    \label{eq:basic_ineq_aer}
    \sum_{i\in[m]} p_i \bigg(L_i(\twww{i}, S_i)+ \frac{\lambda}{2}\|\twww{\glob}- \twww{i}\|^2\bigg) \leq \sum_{i\in[m]} p_i \bigg(L_i(\www{i}_\star, S_i)+ \frac{\lambda}{2}\|\www{\glob}_{\avg}- \www{i}_\star\|^2\bigg),
    \end{equation}
    whereas the proof of \eqref{eq:ier_via_stab} is based on the following basic inequality for the IER: for any $s\in[m]$, we have
    \begin{align}
    & \sum_{i\in[m]} p_i \bigg(L_i(\twww{i}, S_i)+ \frac{\lambda}{2}\|\twww{\glob}- \twww{i}\|^2\bigg) \nonumber\\
    \label{eq:basic_ineq_ier}
    & \leq p_s L_s(\www{s}_\star, S_s) + \frac{\lambda}{2} \|\hwww{\glob} - \www{s}_\star\|^2
     +  \sum_{i\neq s} p_i \bigg(L_i(\hwww{i}, S_i)+ \frac{\lambda}{2}\|\hwww{\glob}- \hwww{i}\|^2\bigg).
    \end{align}
We refer the readers to Appendix \ref{prf:prop:fedstab_implication} for details.
\end{proof}

Note that both bounds in Proposition \ref{prop:fedstab_implication} involve a term that scales linearly with both $\lambda$ and the heterogeneity measure. In general, we expect the stability measures to scale inversely with $\lambda$, and thus opening the possibility of carefully choosing $\lambda$ to balance the stability term and the heterogeneity term.

Let us observe that the heterogeneity term of \eqref{eq:ier_via_stab} is slightly different than that of \eqref{eq:aer_via_stab}, in that it involves the {estimated global model} $\hwww{\glob}$. This suggests that achieving the \ier~guarantees might be intrinsically more difficult than achieving the \aer~guarantees. 

In view of Proposition \ref{prop:fedstab_implication}, we are left to bound the optimization error and the \fedstab~of \sfalong. As discussed above, achieving the \aer~and \ier~guarantees requires somewhat different assumptions, as the latter involves characterizing the performance of the global model. So we split our discussion into two parts below.

\paragraph{Bounding the average excess error.} The following theorem characterize the performance of \sfalong~in terms of the \aer.
\begin{theorem}[AER guarantees for \sfalong]
    \label{thm:aer_sfa_wted}
    Let Assumptions \ref{assump:regularity} and \ref{assump:heterogeneity}(\aersim) hold, and assume $n_i \geq 4\beta/\mu$ for all $i\in[m]$. Choose the weight vector $\bp$ such that 
    \begin{equation}
        \label{eq:req_on_wts_aer}
        \frac{p_{\max} \sum_{i\in[m]}p_i/n_i}{\sum_{i\in[m]}p_i^2/n_i} \leq C_\bp
    \end{equation} 
    for some constant $C_\bp$, where $p_{\max} = \max_i p_i$. Consider the \emph{\sfalong}~algorithm, $\calA_\SFA$, with the following hyperparameter configuration:
    \begin{enumerate}
        \item In the joint training stage (i.e., $0\leq t \leq T-1$), set
        \begin{align}
            \label{eq:hyperparam_stage_1}
            \etasup{i}_{t, k} & = \frac{1}{(\mu+\lambda)(k+1)}, ~~
            \etasup{\glob}_t = \frac{2(\mu+\lambda)}{\lambda\mu(t+1)}, ~~
            K_t+1  \geq C_1  (\lambda^2 \lor 1) t, \\
            T & \geq C_2 \lambda(\lambda \lor 1) m \|\bp\|^2 \cdot \bigg(\bigl[ {\sum_{i\in[m]} p_i/n_i}\bigr]^{-1} \lor \bigl[\lambda(\lambda \lor 1) n_{\max}^2 \bigr]\bigg);\nonumber
        \end{align}
        \item In the final training stage (i.e., $t = T$), set 
        \begin{align}
            \label{eq:hyperparam_stage_2}   
            \etasup{i}_{T, k}& = \frac{1}{(\mu+\lambda)(k+1)},\\
            K_T & \geq C_3  (\lambda+1)^2 \cdot \bigg(\bigl[{\sum_{i\in[m]}p_i/n_i}\bigr]^{-1} \lor \bigl[\lambda^2 \max_{i\in[m]} (p_in_i)^2\bigr]\bigg), \nonumber
        \end{align}
    \end{enumerate} 
    where $C_1, C_2, C_3$ are constants depending only on $(\mu, \beta, \|\ell\|_\infty, D)$.
    Then, there exists a choice of $\lambda$ such that
    \begin{equation}
        \label{eq:aer_sfa_full}
        \begin{aligned}
  &      \bbE_{\calA_\SFA, \bS}[\aer_{\bp}(\calA_{\SFA})] \\
&\lesssim \bigg(\frac{\mu}{1 \land C_\bp} + \frac{(1\lor C_\bp) \beta \|\ell\|_\infty}{\mu} \bigg)\bigg[\bigg( R \sqrt{\sum_{i\in[m]}\frac{p_i}{n_i}}\bigg)\land \bigg( \sum_{i\in[m]}\frac{p_i}{n_i}\bigg)  + \sum_{i\in[m]}\frac{p_i^2}{n_i} \bigg].
        \end{aligned}
    \end{equation}
\end{theorem}
\begin{proof}
    See Appendix \ref{prf:thm:aer_sfa_wted}.
\end{proof}

A few remarks are in order. 
First, \eqref{eq:req_on_wts_aer} essentially says that the weight $\bp$ cannot be too imbalanced, and too much imbalance in $\bp$ can hurt the performance in view of the multiplicative factor of $C_\bp$ in our bound \eqref{eq:aer_sfa_full}. If we set $p_i = 1/m$, then $C_\bp$ is naturally of constant order; whereas if we set $p_i = n_i/N$, we have $C_\bp \asymp m n_{\max}/N$, where $n_{\max} = \max_i n_i$, which calls for relative balance of the sample sizes.  

We then briefly comment on the hyperparameter choice in the above theorem. The step sizes are of the form $1/(\textnormal{strongly convex constant} \times \textnormal{iteration counter})$, and such a choice is common in strongly convex stochastic optimization problems (see, e.g., \citealt{rakhlin2011making,shamir2013stochastic}). Such a choice, along with the smoothness of the problem, is also the key for us to by-pass the need of doing any time-averaging operation, as is done in, for example, \citet{dinh2020personalized}. 

In Theorem \ref{thm:aer_sfa_wted}, the choice of the communication rounds $T$ and the final local training round $K_T$ both scale polynomially with $\lambda$, which means that the optimization convergence of \sfalong~is slower when the data are less heterogeneous. This phenomenon happens more generally. For example, in \cite{hanzely2020federated}, they proposed a variant of SGD that optimizes \eqref{eq:soft_sharing_wted} with $p_i = 1/m$ in $\calO\big(\frac{L + \lambda}{\mu} \log 1/\ep\big)$-many iterations, where $L$ is the Lipschitz constant of the loss function and $\ep$ is the desired accuracy level. 

The constants $C_1, C_2, C_3$ in the statement of Theorem \ref{thm:aer_sfa_wted} can be explicitly traced in our proof. We remark that the dependence on problem-specific constants $(\mu, \beta, \|\ell\|_\infty, D)$ in our hyperparameter choice and on $\lambda$ may not be tight. A tight analysis of the optimization error is interesting, but less relevant for our purpose of understanding the sample complexity. So we defer such an analysis to future work\footnote{The theories developed by \citet{hanzely2020lower} can be useful for such an analysis.}.

\paragraph{Bounding individualized excess errors.} The following theorem gives the \ier~guarantees for \sfalong.
\begin{theorem}[IER guarantees for \sfalong]
    \label{thm:ier_sfa_wted}
    Let Assumptions \ref{assump:regularity} and \ref{assump:heterogeneity}(\iersim) hold. Moreover, assume that $n_i \asymp n_{i'}$ for any $i\neq i'\in[m]$ and $n_i \geq 4\beta/\mu ~\forall i\in[m]$. Let the weight vector be chosen as $p_i \asymp 1/m ~\forall i\in[m]$. Consider the \emph{\sfalong}~algorithm, $\calA_\SFA$, with the following hyperparameter configuration:
    \begin{enumerate}
        \item In the joint training stage (i.e., $0\leq t \leq T-1$), set $\etasup{i}_{t, k}, \etasup{\glob}_t, K_t$ as in \eqref{eq:hyperparam_stage_1}, and set
        \begin{align*}
            T & \geq C_2'  \lambda(\lambda \lor 1) \max_{i\in[m]} n_i \cdot \bigg( p_i^{-1} \lor [\lambda(\lambda \lor 1)n_i]\bigg);
        \end{align*}
        \item In the final training stage (i.e., $t = T$), set $\etasup{i}_{T, k}$ as in \eqref{eq:hyperparam_stage_2}, and set
        \begin{align*}
            K_T & \geq C_3'(\lambda+1)^2 \max_{i\in[m]} n_i \bigg( p_i^{-1} \lor \lambda^2 p_i^2 n_i\bigg),
        \end{align*}
    \end{enumerate} 
    where $C_2', C_3'$ are constants only depending on $(\mu, \beta, \|\ell\|_\infty, D)$.
    Then, there exists a choice of $\lambda$ such that for any $i\in[m]$, we have
    \begin{equation}
        \label{eq:ier_sfa}
        \bbE_{\calA_{\SFA}, \bS}[\ier_{i}(\calA_{\SFA})] \lesssim\bigg[(\mu+\mu^{-1})\bigg(\beta\|\ell\|_\infty + \frac{\sigma^2 \beta^2 + \beta^2 + \sigma^2}{\mu^2}\bigg) + \mu D^2\bigg]\cdot \bigg(\frac{R}{\sqrt{n_i}} \land \frac{1}{n_i} + \frac{\sqrt{m}}{N}\bigg).
    \end{equation}
\end{theorem}
\begin{proof}
    See Appendix \ref{prf:thm:ier_sfa_wted}.
\end{proof}

Compared to Theorem \ref{thm:aer_sfa_wted}, the above theorem imposes extra assumptions that the sample sizes are relative balanced and that $p_i \asymp 1/m$, both of which are due to the fact that we need to additionally take care of the estimation error of the global model. 
The hyperparameter choice slightly differs from that in Theorem \ref{thm:aer_sfa_wted} for the same reason. 
In practice, when one is to use \sfalong~to optimize highly non-convex functions like the loss function of deep neural networks, instead of sticking to the choices made in Theorems \ref{thm:aer_sfa_wted} and \ref{thm:ier_sfa_wted}, the hyperparameters are usually tuned by trial-and-error for best test performance.

\paragraph{Comparison with the lower bounds.}
In order to comment about the optimality/suboptimality of \textsc{FedProx}, let us restrict to the case when $p_i = n_i/N$. In this case, the bound in Theorem \ref{thm:aer_sfa_wted} becomes
\begin{align}
    \label{eq:aer_sfa}
    \bbE_{\calA_\SFA, \bS}[\aer_\bp(\calA_\SFA)] 
    \lesssim \bigg(\mu + \frac{\beta \|\ell\|_\infty}{\mu}\bigg) \cdot \bigg(\frac{1}{N/m} \land \frac{R}{\sqrt{N/m}} + \frac{1}{N}\bigg).
\end{align}
Recall the lower bound in \eqref{eq:lb_aer}.
Focusing on the dependence of sample sizes and heterogeneity measure, 
we have the following three cases. If $R^2 \gtrsim m/N$, then \eqref{eq:aer_sfa} becomes $\calO(m/N)$, which matches the lower bound.
Meanwhile, if $1 / mN\lesssim R^2 \lesssim m/N$, then \eqref{eq:aer_sfa} becomes $\calO(m/N)$, whereas the lower bound reads $\Omega(R^2 + 1/N)$, and thus \eqref{eq:aer_sfa} is suboptimal unless $R^2 \asymp m/N$.
Moreover, if $R^2 \lesssim 1/mN$, then \eqref{eq:aer_sfa} becomes $\calO(1/N)$, and is minimax optimal again.

A similar trilogy holds for IER of \textsc{FedProx}. Comparing the upper bound in \eqref{eq:ier_sfa} and the lower bound in \eqref{eq:lb_ier}, we still have three cases as follows. 
If $R^2\gtrsim m/N$, then \eqref{eq:ier_sfa} is $\calO(1/n_i)$, which agrees with the lower bound. 
Meanwhile, if $1/N \lesssim R^2 \lesssim m/N$, then \eqref{eq:ier_sfa} is $\calO(R/\sqrt{n_i})$, and is suboptimal compared to the $\Omega(R^2 + 1/N)$ lower bound unless $R^2 \asymp m/N$.
Moreover, if $R^2 \lesssim 1/N$, then \eqref{eq:ier_sfa} is $\calO(\sqrt{m}/N)$, and is off by a factor of order $\sqrt{m}$ compared to the $\Omega(1/N)$ lower bound.

While the bounds in Theorems \ref{thm:aer_sfa_wted} and \ref{thm:ier_sfa_wted} in general do not attain the lower bounds in Corollary \ref{cor:lb}, they are still non-trivial in the sense that they scale with the heterogeneity measure $R$. While there are some recent works establishing the \aer~guarantees for an objective similar to \eqref{eq:soft_sharing_wted} under the online learning setup (see, e.g., \citealt{denevi2019learning,balcan2019provable,khodak2019adaptive}), to the best of our knowledge, Theorems \ref{thm:aer_sfa_wted} and \ref{thm:ier_sfa_wted} are the first to establish \emph{both} the \aer~and \ier~guarantees for \eqref{eq:soft_sharing_wted} under the federated learning setup.

Curious readers may wonder if the suboptimality of the theoretical guarantees for \textsc{FedProx} (with non-zero $\lambda$) is a characteristic of this algorithm or if it is due to the artifact of our technical proof. To answer this question, we conduct a simulation where we apply \textsc{FedProx} with different $\lambda$s on datasets generated by federated logistic regression (see Appendix \ref{appx:exp} for details). The accuracies versus different values of $R$ is shown in the right panel of Figure \ref{fig:sim}. As expected by our theory, the performance of \textsc{FedProx} with $\lambda=0$ mimics that of \textsc{PureLocalTraining}, whereas the performance with $\lambda = 4$ resembles that of \textsc{FedAvg}. Interestingly, \textsc{FedProx} with $\lambda = 0.44$ bears a similar performance with the \textsc{FedAvg} followed by fine tuning strategy, which we know is minimax optimal. This observation supports the conjecture that optimally tuned \textsc{FedProx} is indeed minimax optimal, and the suboptimality of bounds from Theorems \ref{thm:aer_sfa_wted} and \ref{thm:ier_sfa_wted} are likely to be a consequence of the artifact of our theoretical analysis.


\section{Discussion}\label{sec:discuss}
This paper studies the statistical properties of personalized federated learning. Focusing on strongly-convex, smooth, and bounded empirical risk minimization problems, we have uncovered an intriguing phenomenon that given a specific level of heterogeneity, exactly one of \falong~or \pllong~is minimax optimal. In the course of proving this result, we obtained a novel analysis of \fplong~and introduced a new notion of algorithmic stability termed federated stability, which is possibly of independent interest for analyzing generalization properties in the context of federated learning.

We close this paper by mentioning several open problems.
\begin{itemize}
  \item \emph{Dependence on problem-specific parameters.} 
  This paper focuses on the dependence on the sample sizes, and in our bounds, the dependence on problem-specific parameters (e.g., the smoothness and strong convexity constants) may not be optimal. This can be problematic if those parameters are not of constant order, and it will be interesting to give a refined analysis that gives optimal dependence on those parameters.  

  \item \emph{A refined analysis of \textsc{FedProx}.}
  The upper bounds we develop for \textsc{FedProx}, as we have mentioned, do not match our minimax lower bounds. According to a simulated example, we suspect that this is an artifact of our analysis and a refined analysis of \textsc{FedProx}~would be a welcome advance. 

  \item \emph{Estimation of the level of heterogeneity and development of adaptive algorithms.} For unsupervised problems where evaluation of a model is difficult, implementation of the oracle dichotomous strategy described in Section \ref{subsec:alternative} would require estimating the level of heterogeneity $R$. Even for supervised problems, estimation of $R$ would be interesting, as it allows one to decide which algorithm to choose without model training. More generally, developing adaptive algorithms that attains the lower bound without prior information of $R$ is an important open problem.

  \item \emph{Beyond the current heterogeneity assumption.} As discussed in Section \ref{subsec:alternative}, our theoretical results may not hold globally when one moves from Assumption \ref{assump:heterogeneity} to more general heterogeneity assumptions. Establishing the minimax rates and designing provably optimal algorithms under those assumptions are of both theoretical and practical interest.

  \item \emph{Beyond convexity.}
  Our analysis is heavily contingent upon the strong convexity of the loss function, which, to the best of our knowledge, is not easily generalizable to the non-convex case. Meanwhile, our notion of heterogeneity, which is based on the distance of optimal local models to the convex combination of them, may not be natural for non-convex problems. It is of interest, albeit difficult, to have a theoretical investigation of personalized federated learning for non-convex problems.
\end{itemize}

\acks{This work was supported in part by NIH through R01-GM124111 and RF1-AG063481, NSF through CAREER DMS-1847415, CCF-1763314, and CCF-1934876, and an Alfred P.~Sloan Research Fellowship.}


\newpage

\appendixtitleon
\appendixtitletocon
\addtocontents{toc}{\protect\setcounter{tocdepth}{3}}
\begin{appendices}
\tableofcontents

\section{Proof of Theorem \ref{thm:lb_logistic}: Lower Bounds}\label{prf:thm:lb_logistic}

We start by presenting a lower bound when all $\www{i}_\star$'s are the same.
\begin{lemma}[Lower bound under homogeneity]
  \label{lemma:lb_hom}
    Consider the logistic regression model with $\www{i}_\star = \www{\glob}_\bp$ for any $i\in[m]$. Then
    $$
      \inf_{\hwww{\glob}} \sup_{\www{\glob}_\bp} \bbE_{\bS}\|\hwww{\glob}  - \www{\glob}_\bp\|^2 \gtrsim \frac{d}{N}.
    $$
\end{lemma}
\begin{proof}
  This is a classical result. See, e.g., Example 8.4 of \cite{duchi2016lecture}.
\end{proof}

\begin{proof}[Proof of \eqref{eq:lb_logistic_aer}]
  We first give a lower bound based on the observation that the homogeneous case is in fact included in the parameter space $\calP_1$. 
  More explicitly, let us define $\calP_0 = \{\{\www{i}_\star\}\in \calP_1: \www{i}_\star = \www{\glob}_\bp ~ \forall i\in[m]\}$. By Lemma \ref{lemma:lb_hom}, we have
  \begin{align}
    \inf_{\{\hwww{i}\}} \sup_{\{\www{i}_\star\}\in \calP_1} \sum_{i\in[m]}p_i \bbE_{\bS} \| \hwww{i} - \www{i}_\star \|^2 & \geq \inf_{\{\hwww{i}\}} \sup_{\{\www{i}_\star\}\in \calP_0} \sum_{i\in[m]}p_i \bbE_{\bS} \| \hwww{i} - \www{i}_\star \|^2 \nonumber\\
    & = \inf_{\hwww{\glob}} \sup_{\www{\glob}_\bp} \bbE_{\bS} \|\hwww{\glob} - \www{\glob}_\bp\|^2 \nonumber\\
    \label{eq:lb_hom_aer}
    & \gtrsim \frac{d}{N}.
  \end{align}

  We now use a variant of Assouad's method \citep{assouad1983deux} that allows us to tackle multiple datasets. Consider the following data generating process: nature generates $\bsV= \{\vvv{i}: i\in[m]\}$ i.i.d. from the uniform distribution on $\calV = \{\pm 1\}^d$ and sets $\www{i}_\star = \delta_i \vvv{i}$ for some $\delta_i$ such that the following constraint is satisfied:
  \begin{equation}
    \label{eq:prior_constr_aer}
    \sum_{i\in[m]} p_i \|\www{i}_\star - \www{\glob}_\bp\|^2 = \sum_{i\in[m]} p_i \bigg\|\delta_i \vvv{i} - \sum_{s\in[m]} p_s \delta_s \vvv{s}\bigg\|^2 \leq  R^2.
  \end{equation}  
  We will specify the choice of $\delta_i$'s later. Denoting $\bbE_{\bsX}$ as the marginal expectation operator with respect to all the features $\{\xxx{i}_j\}$ and $\bbE_{\bsY|\bsX}$ as the conditional expectation operator with respect to $\{\yyy{i}_j\}|\{\xxx{i}_j\}$, we can lower bound the minimax risk by the Bayes risk as follows:
  \begin{align*}
    & \inf_{\{\hwww{i}\}} \sup_{\{\www{i}_\star\}\in \calP} \sum_{i\in[m]}p_i \bbE_{\bS} \| \hwww{i} - \www{i}_\star \|^2 \\
    & \geq \inf_{\{\hwww{i}\}} \bbE_{\{\vvv{i}\}} \sum_{i\in[m]}p_i \bbE_{\bS} \| \hwww{i} - \delta_i \vvv{i} \|^2 \\
    & = \inf_{\{\hvvv{i}\}\subseteq \calV } \sum_{i\in[m]} p_i \bbE_{\bsV, \bS} \| \delta_i \hvvv{i} - \delta_i \vvv{i} \|^2  \\
    & \geq \bbE_{\bsX} \sum_{i\in[m]} p_i \delta_i^2 \inf_{\hvvv{i}\in\calV } \bbE_{\bsV, \bsY|\bsX} \|\hvvv{i} - \vvv{i}\|^2 \\
    & \geq \bbE_{\bsX} \sum_{i\in[m]} p_i\delta_i^2 \sum_{k\in[d]}  \inf_{\hvvv{i}_k\in\{\pm 1\}} \bbE_{\bsV, \bsY|\bsX}(\hvvv{i}_k - \vvv{i}_k)^2 \\
    & \gtrsim \bbE_{\bsX} \sum_{i\in[m]} p_i\delta_i^2 \sum_{k\in[d]}  \inf_{\hvvv{i}_k\in\{\pm 1\}} \bbP_{\bsV, \bsY|\bsX}(\hvvv{i}_k \neq \vvv{i}_k) \\
    & = \frac{1}{2}\bbE_{\bsX} \sum_{i\in[m]} p_i\delta_i^2 \sum_{k\in[d]}  \inf_{\hvvv{i}_k\in\{\pm 1\}} \bigg( \bbP_{i, +k} (\hvvv{i}_k = -1) + \bbP_{i, -k}(\hvvv{i}_k = +1) \bigg),
  \end{align*}
  where in the last line, we have let $\bbP_{i, \pm k}(\cdot) = \bbP_{\bsV, \bsY|\bsX}(\cdot | \vvv{i}_k = \pm 1)$ to denote the probability measure with respect to the randomness in $(\bsV, \bS)$ conditional on the features $\{\xxx{i}_j\}$ as well as the realization of $\vvv{i}_k = \pm 1$. More explicitly, we can write 
  \begin{align*}
    \bbP_{i, \pm k} & = \bigg(\bigotimes_{s\neq i} \bbP_{\vvv{s}}\otimes  \bbP_{\{\yyy{s}\}_{j=1}^{n_s}| \vvv{s}, \{\xxx{s}_j\}_{j=1}^{n_s}} \bigg) \otimes \bigg(\bbP_{\vvv{i}| \vvv{i}_k = \pm 1} \otimes \bbP_{\{\yyy{i}_j\}_{j=1}^{n_i} | \vvv{i}, \vvv{i}_k = \pm 1, \{\xxx{i}_j\}_{j=1}^{n_i}} \bigg) \\
    & = \frac{1}{2^{(m-1)d + d-1 }} \sum_{\bsV\setminus\{\vvv{i}_k\} }  \bbP_{\bsV, i, \pm k},
  \end{align*}
  where the $\otimes$ symbol stands for taking the product of two measures and $\bbP_{\bsV, i, \pm k}$ corresponds to the law of all the labels $\bsY$ conditional on a specific realization of $\{\bsV: \vvv{i}_k = \pm 1\}$ and the features $\bsX$. With the current notations and letting $\|\bbP - \bbQ\|_\tv$ be the total variation distance between two probability measures $\bbP$ and $\bbQ$, we can invoke Neyman-Pearson lemma to get
  \begin{align}
     \inf_{\{\hwww{i}\}} \sup_{\{\www{i}_\star\}\in \calP} \sum_{i\in[m]}p_i \bbE_{\bS} \| \hwww{i} - \www{i}_\star \|^2  
    & \gtrsim \bbE_{\bsX} \sum_{i\in[m]} p_i \delta_i^2 \sum_{k\in[d]} \bigg(1 -  \|\bbP_{i, +k} - \bbP_{i, -k}\|_\tv \bigg)\nonumber \\
    \label{eq:aer_lb_mid}
    & = d \sum_{i\in[m]}p_i \delta_i^2 - \bbE_{\bsX} \sum_{i\in[m]} p_i \delta_i^2 \sum_{k\in[d]}\|\bbP_{i, +k}- \bbP_{i, -k}\|_\tv.
  \end{align}
  We then proceed by
  \begin{align*}
    & \sum_{i\in[m]} p_i \delta_i^2 \sum_{k\in[d]}\|\bbP_{i, +k}- \bbP_{i, -k}\|_\tv  \\
    & \leq \sum_{i\in[m]} p_i \delta_i^2 \sqrt{d} \bigg( \sum_{k\in[d]} \|\bbP_{i, +k} - \bbP_{i, -k}\|_\tv^2 \bigg)^{1/2} \\
    & = \sum_{i\in[m]} p_i \delta_i^2 \sqrt{d} \bigg( \sum_{k\in[d]} \bigg\|\frac{1}{2^{(m-1)d + d-1 }} \sum_{\bsV\setminus\{\vvv{i}_k\}}  \bbP_{\bsV, i, + k} - \bbP_{\bsV, i, -k} \bigg\|_\tv^2 \bigg)^{1/2}  \\
    & = \sum_{i\in[m]} p_i \delta_i^2 \sqrt{d} \bigg( \sum_{k\in[d]} \frac{1}{2^{(m-1)d + d-1 }} \sum_{\bsV\setminus\{\vvv{i}_k\}}  \|\bbP_{\bsV, i, + k} - \bbP_{\bsV, i, -k} \|_\tv^2 \bigg)^{1/2}, 
  \end{align*}
  where the last inequality is by convexity of the total variation distance. Note that $\bbP_{\bsV, i, \pm k}$ is the product of biased Rademacher random variables: if we let $\textnormal{Rad}(p)$ be the $\pm 1$-valued random variable with positive probability $p$, we can write
  $$
    \bbP_{\bsV, i, \pm k} =\bigotimes_{s\in[m]}\bigotimes_{j\in[n_s]} \textnormal{Rad}\bigg(\frac{1}{1 + \exp\{- \delta_s \la \vvv{s}, \xxx{s}_j\ra \}}\bigg), \qquad \vvv{i}_k = \pm 1.
  $$
  Thus, by Pinsker's inequality, we have 
  \begin{align*}
    & \|\bbP_{\bsV, i, + k} - \bbP_{\bsV, i, -k} \|_\tv^2 \\
    & \leq \frac{1}{2}  D_{\js}(\bbP_{\bsV, i, +k}\| \bbP_{\bsV, i, -k}) \\ 
    & = \frac{1}{2} \sum_{s\neq i} \sum_{j\in[n_s]} 0 + \frac{1}{2}  \sum_{j\in[n_i]} D_\js\bigg[\textnormal{Rad}\bigg( \frac{1}{1+ \exp\{-\delta_i \la \vvv{i}, \xxx{i}_j\}} \bigg)\bigg\| \textnormal{Rad}\bigg( \frac{1}{1+ \exp\{-\delta_i \la \tvvv{i}, \xxx{i}_j\}}  \bigg)\bigg],
  \end{align*}
  where $D_\js(\bbP\| \bbQ) = \frac{D_\kl(\bbP\| \bbQ) + D_{\kl}(\bbQ\|\bbP)}{2}$ is the Jensen–Shannon divergence between $\bbP$ and $\bbQ$, and $\vvv{s}, \tvvv{s}$ are two $\calV$-valued vectors that only differs in the $k$-th coordinate. By a standard calculation, one finds that 
  \begin{align*}
    D_\js\bigg[\textnormal{Rad}\bigg( \frac{1}{1+ \exp\{-\delta_i \la \vvv{i}, \xxx{i}_j\}} \bigg)\bigg\| \textnormal{Rad}\bigg( \frac{1}{1+ \exp\{-\delta_i \la \tvvv{i}, \xxx{i}_j\}}  \bigg)\bigg] 
    & \leq \delta_i^2 (\vvv{i}_k - \tvvv{i}_k)^2 (\xxx{i}_{j, k})^2 \\
    & = 4\delta_i^2(\xxx{i}_{j, k})^2 .
  \end{align*}  
  This gives 
  $$
    \|\bbP_{\bsV, i, + k} - \bbP_{\bsV, i, -k} \|_\tv^2  \leq 2\delta_i^2 \sum_{j\in[n_i]} (\xxx{i}_{j, k})^2 \leq 2 \delta_i^2 c_X^2n_i .
  $$
  and hence
  \begin{align*}
    \sum_{i\in[m]} p_i \delta_i^2 \sum_{k\in[d]}\|\bbP_{i, +k}- \bbP_{i, -k}\|_\tv  \leq \sqrt{2}c_X \sum_{i\in[m]} p_i\delta_i^3 d n_i^{1/2}.
  \end{align*}
  Plugging the above display to \eqref{eq:aer_lb_mid} gives
  \begin{align}
    \label{eq:aer_lb_choice}
    \inf_{\{\hwww{i}\}} \sup_{\{\www{i}_\star\}\in \calP} \sum_{i\in[m]}p_i \bbE_{\bS} \| \hwww{i} - \www{i}_\star \|^2  
    &\gtrsim d \bigg( \sum_{i\in[m]}p_i \delta_i^2 - \sqrt{2}c_X \sum_{i\in[m]}p_i\delta^3_i \sqrt{n_i}\bigg).
  \end{align}
  To this end, all that is left is to choose $\delta_i$ approriately so that (1) the above display is as tight as possible; (2) \eqref{eq:prior_constr_aer} is satisfied. We consider the following two cases:
  \begin{enumerate}
    \item Assume $ R^2 \geq d \sum_{i\in[m]}p_i/n_i = dm/N$. Note that we can re-write the requirement \eqref{eq:prior_constr_aer} to be
    $$
      d \sum_{i\in[m]} p_i \delta_i^2 - \|\sum_{i\in[m]} p_i \delta_i \vvv{i}\|^2 \leq  R^2.
    $$
    Under the current assumption, this requirement will be satisfied if we choose $\delta_i = c/\sqrt{n_i}$ for any $c \leq 1$. Under such a choice, the right-hand side of \eqref{eq:aer_lb_choice} becomes 
    $
      \frac{c^2 dm}{N} (c - \sqrt{2}c_X).
    $
    Thus, by setting $c = 2\sqrt{2}c_X$, we get the following lower bound:
    $$
      \inf_{\{\hwww{i}\}} \sup_{\{\www{i}_\star\}\in \calP} \sum_{i\in[m]}p_i \bbE_{\bS} \| \hwww{i} - \www{i}_\star \|^2  \gtrsim \frac{d}{N/m}.
    $$
    \item Assume $ R^2 \leq d \sum_{i\in[m]}p_i/n_i = dm/N$. Note that if we set $\delta_i \equiv \delta = c  R/\sqrt{d}$ where $c \leq 1$, \eqref{eq:prior_constr_aer} reads
    $$
      c^2  R^2 - \|\sum_{i\in[m]} p_i \delta_i \vvv{i}\|^2 \leq  R^2,
    $$
    which trivially holds. Now, the right-hand side of \eqref{eq:aer_lb_choice} becomes
    $$
      c^2 R^2 (1 - \sqrt{2} c c_X \sum_{i\in[m]} p_i  R \sqrt{n_i}/\sqrt{d}).
    $$
    Since $p_i = n_i/N$ and $n_i \asymp N/m$, our assumption on $ R$ gives
    $$
      \sqrt{2} c c_X \sum_{i\in[m]} p_i  R \sqrt{n_i}/\sqrt{d} \lesssim \sum_{i} \frac{n_i}{N} \cdot \sqrt{\frac{mn_i}{N}}= 1.
    $$
    This means that we can choose $c$ to be a small constant such that the following lower bound holds:
    $$
      \inf_{\{\hwww{i}\}} \sup_{\{\www{i}_\star\}\in \calP} \sum_{i\in[m]}p_i \bbE_{\bS} \| \hwww{i} - \www{i}_\star \|^2  \gtrsim  R^2.
    $$
  \end{enumerate} 
  Summarizing the above two cases, we arrive at
  $$
    \inf_{\{\hwww{i}\}} \sup_{\{\www{i}_\star\}\in \calP} \sum_{i\in[m]}p_i \bbE_{\bS} \| \hwww{i} - \www{i}_\star \|^2  \gtrsim \frac{d}{N/m} \land  R^2.
  $$
  Combining the above bound with \eqref{eq:lb_hom_aer}, we get
  $$
    \inf_{\{\hwww{i}\}} \sup_{\{\www{i}_\star\}\in \calP} \sum_{i\in[m]}p_i \bbE_{\bS} \| \hwww{i} - \www{i}_\star \|^2  \gtrsim \frac{d}{N/m} \land  R^2 + \frac{d}{N} ,
  $$
  which is the desired result.
\end{proof}

\begin{proof}[Proof of \eqref{eq:lb_logistic_ier}]
The proof is similar to the proof of \eqref{eq:lb_aer}, and we only provide a sketch here. 
Without loss of generality we consider the first client. 
By the same arguments as in the proof of \eqref{eq:lb_aer}, the left-hand side of \eqref{eq:lb_logistic_ier} is lower bounded by a constant multiple of $d/N$. 
Now, by considering the same prior distribution on $\calP$ as in the proof of \eqref{eq:lb_aer}, we get 
$$
\inf_{\barwww{1}} \sup_{\{\www{i}\}\in \calP} \bbE_{\bS} \|\barwww{1} - \www{1}_\star\|^2 \gtrsim d \delta_1^2 (1 - \delta_1 \sqrt{n_1}),
$$
where the $\delta_i$'s should obey the following inequality:
$$
  \|\delta_i \vvv{i} - \sum_{s\in[m]}p_s \delta_s \vvv{s} \|^2 \leq R^2.
$$
Choosing $\delta_i \asymp 1/\sqrt{n_i}$ when $R \geq dm/N$ and $\delta_i \asymp R/\sqrt{d}$ otherwise, we arrive at
$$
  \inf_{\barwww{1}} \sup_{\{\www{i}\}\in \calP} \bbE_{\bS} \|\barwww{1} - \www{1}_\star\|^2 \gtrsim \frac{d}{n_1} \land R^2,
$$
and the proof is concluded.
\end{proof}

\section{Optimization Convergence of \fplong}\label{append:opt_conv}

This section concerns the optimization convergence of \fplong.
We first introduce some notations.
Let $\www{i}_{t, k}$ be the output of $k$-th step of Algorithm \ref{alg:softfedavg_wted} when the initial local model is given by $\www{i}_t \equiv\www{i}_{t, 0} \equiv \www{i}_{t-1, K}$, let $\III{t}_{t, k}$ be the corresponding minibatch taken, and denote the initial global model by $\www{\glob}_t$. 
Let $\calF_{t, k}$ be the sigma algebra generated by the randomness by Algorithm \ref{alg:softfedavg_wted} up to $\www{i}_{t, k}$, namely the randomness in $\bigg\{\calC_{\tau}, \{\III{i}_{\tau, l}: i\in \calC_\tau, 0\leq l\leq K_\tau-1\}\bigg\}_{\tau=0}^{t-1}$, $\calC_t$, and $\{\III{i}_{t, l}: i\in \calC_t, 0\leq l \leq k-1\}$. For notational convenience we let $\calC_T = [m]$ (i.e., all clients are involved in local training in Stage \RN{2} of Algorithm \ref{alg:softfedavg_wted}). Then the sequence $\{\www{i}_{t, k}\}$ is adapted to the following filtration:
\begin{equation*}
  \calF_{0, 0}\subseteq \calF_{0, 1} \subseteq \cdots \subseteq\calF_{0, K}\subseteq \calF_{1, 0} \subseteq \calF_{1, 1}\subseteq \cdots \subseteq \calF_{1, K} \subseteq \cdots \subseteq \calF_{T, K}.
\end{equation*}
We write the optimization problem \eqref{eq:soft_sharing_wted} as
\begin{equation}
    \label{eq:soft_sharing_alt_repr}
    \min_{\www{\glob}\in\calW} \sum_{i\in[m]} p_i F_i(\www{\glob}, S_i),
\end{equation}
where
\begin{equation}
    \label{eq:F_i}
    F_i(\www{\glob}, S_i)= \min_{\www{i}\in\calW} \bigg\{L_i(\www{i}, S_i) + \frac{\lambda}{2}\| \www{\glob} - \www{i}\|^2\bigg\}.
\end{equation}
To simplify notations, we introduce the proximal opertor
\begin{equation}
    \label{eq:F_i_argmin}
    \prox_{L_i/\lambda}(\www{\glob})  = \prox_{L_i/\lambda}(\www{\glob}, S_i) = \argmin_{\www{i}\in\calW} \bigg\{L_i(\www{i}, S_i) + \frac{\lambda}{2}\| \www{\glob} - \www{i}\|^2\bigg\}.
\end{equation}

The high-level idea of this proof is to regard ${\lambda}\sum_{i\in\calC_t}(\www{\glob}_t - \www{i}_{t+1})/{\BBB{\glob}}$ as a biased stochastic gradient of $\frac{1}{n}\sum_{i\in[m]}F_i(\www{\glob}_t, S_i)$. This idea has appeared in various places (see, e.g., the proof of Proposition 5 in \cite{denevi2019learning} and the proof of Theorem 1 in \cite{dinh2020personalized}). However, the implementation of this idea in our case is more complicated than the above mentioned works in that (1) we are not in an online learning setup (compared to \cite{denevi2019learning}); (2) we don't need to assume all clients are training at every round (compared to \cite{dinh2020personalized}); and (3) we use local SGD for the inner loop (instead of assuming the inner loop can be solved with arbitrary precision as assumed in \cite{dinh2020personalized}), so the gradient norm depends on $\lambda$, and could in principle be arbitrarily large, which causes extra complications.

\begin{lemma}[Convergence of the inner loop]
  \label{lemma:conv_inner_loop}
  Let Assumption \ref{assump:regularity}(\compact, \cvx) holds. Choose $\etasup{i}_{t, k} = \frac{1}{(\mu+\lambda)(k+1)}$. Then for any $k\geq 0$, we have
  \begin{equation*}
    \bbE\bigg[\|\www{i}_{t, k} - \prox_{L_i/\lambda}(\www{\glob}_t) \|^2~\bigg|~\calF_{t, 0}, i\in\calC_{t}\bigg] \leq  \frac{8\beta^2 D^2}{\mu^2(k+1)}.
  \end{equation*}
\end{lemma}
\begin{proof}
  See Appendix \ref{prf:lemma:conv_inner_loop}.
\end{proof}

\begin{lemma}[Convergence of the outer loop]
\label{lemma:conv_outer_loop}
  Let the assumptions in Lemma \ref{lemma:conv_inner_loop} hold. Choose $\etasup{\glob}_t = \frac{2(\mu+\lambda)}{\lambda\mu(t+1)}$ and assume  
  \begin{equation}
    \label{eq:inner_rounds_req}
    K_\tau+1 \geq \frac{(4\tau+20)\lambda^2 \beta^2 D^2 }{\mu^2 (\beta^2D^2 \land 2 \lambda\|\ell\|_\infty \land \lambda^2D^2 )} \qquad \forall 0\leq \tau \leq t-1.
  \end{equation}
  Then for any $t\geq 0$, we have
  \begin{equation}
    \label{eq:conv_outer_loop}
    \bbE_{\calA_\SFA}\| \www{\glob}_t - \twww{\glob}\|^2 \leq \frac{12(\lambda+\mu)^2 m \|\bp\|^2 (\beta^2D^2 \land 2\lambda \|\ell\|_\infty\land \lambda^2 D^2)}{\lambda^2\mu^2(t+1)},
  \end{equation}
  where the expectation is taken over the randomness in Algorithm \ref{alg:softfedavg_wted}.
\end{lemma}
\begin{proof}
  See Appendix \ref{prf:lemma:conv_outer_loop}.
\end{proof}

\begin{proposition}[Optimization error of $\calA_\SFA$]
  \label{prop:opt_err_sfa}
  Under the assumptions of Lemma \ref{lemma:conv_inner_loop} and \ref{lemma:conv_outer_loop}, for any dataset $\bS\sim\bigotimes_{i}\calD_i^{\otimes n_i}$, we have 
  \begin{equation*}
    \bbE_{\calA_\SFA}[\calE_\train] \leq \frac{4(\beta + \lambda)\beta^2D^2}{\mu^2(K_T+1)} +\frac{6(\lambda+\mu)^2 m \|\bp\|^2 (\beta^2D^2 \land 2\lambda \|\ell\|_\infty\land \lambda^2 D^2)}{\lambda\mu^2(t+1)}
  \end{equation*}
\end{proposition}
\begin{proof}
  By definition we have
  \begin{align*}
    \bbE_{\calA_\SFA}[\calE_\train] &:= \bbE_{\calA_\SFA}\bigg[\sum_{i\in[m]} p_i \bigg( L_i(\www{i}_{T+1}, S_i)  + \frac{\lambda }{2}\|\www{\glob}_T - \www{i}_{T+1}\|^2 \bigg) -  \sum_{i\in[m]} p_i F_i(\twww{\glob}, S_i)\bigg] \\
    & \overset{\textnormal{(a)}}{\leq} \sum_{i\in[m]} \frac{p_i(\beta + {\lambda })}{2}\cdot \bbE_{\calA_{\SFA}}\|\www{i}_{T+1} - \prox_{L_i/\lambda}(\www{\glob}_T)\|^2  \\
    & \qquad  + \bbE_{\calA_{\SFA}} \bigg[ \sum_{i\in[m]}p_i F_i(\www{\glob}_T , S_i) - \sum_{i\in[m]}p_iF_i(\twww{\glob}, S_i)\bigg] \\
    & \overset{\textnormal{(b)}}{\leq} \frac{4(\beta + \lambda)\beta^2D^2}{\mu^2(K_T+1)} +  \frac{\lambda}{2} \bbE_{\calA_\SFA}\| \www{\glob}_T - \twww{\glob} \|^2\\
    & \overset{\textnormal{(c)}}{\leq} \frac{4(\beta + \lambda)\beta^2D^2}{\mu^2(K_T+1)} +\frac{6(\lambda+\mu)^2 m \|\bp\|^2 (\beta^2D^2 \land 2\lambda \|\ell\|_\infty\land \lambda^2 D^2)}{\lambda\mu^2(t+1)}.
  \end{align*}
where (a) is by smoothness of $L_i$, (b) is by Lemma \ref{lemma:conv_inner_loop} and $\lambda$-smoothness of $\sum_{i\in[m]}p_i F_i$ (which holds by Lemma \ref{lemma:reg_of_F_i}), and (c) is by Lemma \ref{lemma:conv_outer_loop}. 
\end{proof}

\subsection{Proof of Lemma \ref{lemma:conv_inner_loop}: Convergence of the Inner Loop} \label{prf:lemma:conv_inner_loop}
  The proof is an adaptation of the proof of Lemma 1 in \cite{rakhlin2011making}. However, we need to deal with the extra complication that the hyperparameter $\lambda$ can in principle be arbitrarily large. We start by noting that 
  \begin{align*}
    & \|\www{i}_{t, k+1} - \prox_{L_i/\lambda}(\www{\glob}_t)\|^2 \\
    & = \bigg\|\calP_\calW\bigg[\www{i}_{t, k} -  \frac{\etasup{i}_{t,k}}{\BBB{i}}\sum_{j\in \III{i}_{t, k}} \bigg(\nabla \ell(\www{i}_{t, k}, \zzz{i}_j) + \lambda (\www{i}_{t, k} - \www{\glob}_t)\bigg)\bigg]  - \prox_{L_i/\lambda}(\www{\glob}_t)\bigg\|^2\\
    &\leq \bigg\|\www{i}_{t, k} -  \frac{\etasup{i}_{t,k}}{\BBB{i}}\sum_{j\in \III{i}_{t, k}} \bigg(\nabla \ell(\www{i}_{t, k}, \zzz{i}_j) + \lambda(\www{i}_{t, k} - \www{\glob}_t)\bigg) - \prox_{L_i/\lambda}(\www{\glob}_t)\bigg\|^2\\
    & = \|\www{i}_{t, k} - \prox_{L_i/\lambda}(\www{\glob}_t))\|^2 + \bigg\|\frac{\etasup{i}_{t,k}}{\BBB{i}}\sum_{j\in \III{i}_{t, k}} \bigg(\nabla \ell(\www{i}_{t, k}, \zzz{i}_j) + \lambda(\www{i}_{t, k} - \www{\glob}_t)\bigg)\bigg\|^2\\
    & \qquad - 2\hugela \www{i}_{t, k} -  \prox_{L_i/\lambda}(\www{\glob}_t), \frac{\etasup{i}_{t,k}}{\BBB{i}}\sum_{j\in \III{i}_{t, k}} \bigg(\nabla \ell(\www{i}_{t, k}, \zzz{i}_j) + \lambda(\www{i}_{t, k} - \www{\glob}_t)\bigg) \hugera,
  \end{align*}
  where the inequality is because $\prox_{L_i/\lambda}(\www{\glob}_t)\in\calW$ and $\calP_\calW$ is non-expansive. Now by strong convexity and unbiasedness of the stochastic gradients, we have
  \begin{align*}
    & \bbE\bigg[ \hugela \www{i}_{t, k}- \prox_{L_i/\lambda}(\www{\glob}_t), \frac{1}{\BBB{i}}\sum_{j\in \III{i}_{t, k}} \bigg(\nabla \ell(\www{i}_{t, k}, \zzz{i}_j) + \lambda(\www{i}_{t, k} - \www{\glob}_t)\bigg) \hugera~\bigg|~ \calF_{t, k}, i\in\calC_t\bigg]\\
    & \geq \bigg(L_i(\www{i}_{t, k}, S_i) + \frac{\lambda }{2} \|\www{i}_{t, k} - \www{\glob}\|^2\bigg) \\
    & \qquad - \bigg(L_i(\prox_{L_i/\lambda}(\www{\glob}_t), S_i) + \frac{\lambda }{2}\|\prox_{L_i/\lambda}(\www{\glob}_t) - \www{\glob}_t\|^2\bigg)\\
    & \qquad + \frac{1}{2} \bigg(\mu_i + \frac{\lambda n}{mn_i}\bigg) \|\www{i}_{t, k} - \prox_{L_i/\lambda}(\www{\glob}_t)\|^2\\
    & \geq (\mu + \lambda) \|\www{i}_{t, k} - \prox_{L_i/\lambda}(\www{\glob}_t)\|^2.
  \end{align*}
  On the other hand, applying Lemma \ref{lemma:var_minibatch} gives 
  \begin{align*}
    & \bbE\bigg[\bigg\|\frac{\etasup{i}_{t,k}}{\BBB{i}}\sum_{j\in \III{i}_{t, k}} \bigg(\nabla \ell(\www{i}_{t, k}, \zzz{i}_j) + \lambda(\www{i}_{t, k} - \www{\glob}_t)\bigg)\bigg\|^2~\bigg|~\calF_{t, k}, i\in \calC_t\bigg]\\
    & = (\etasup{i}_{t, k})^2 \cdot \bigg[\frac{n_i/\BBB{i} - 1}{n_i (n_i-1)} \sum_{j\in[n_i]} \bigg\|\nabla \ell(\www{i}_{t, k}, \zzz{i}_j) - \overline{\nabla \ell(\www{i}_{t, k} , \zzz{i}_{\bigcdot})}\bigg\|^2 \\
    & \qquad \qquad \qquad+ \bigg\|\frac{1}{n_i}\sum_{j\in[n_i]} \nabla \ell(\www{i}_{t, k}, \zzz{i}_j) + \lambda (\www{i}_{t, k} - \www{\glob}_t)\bigg\|^2\bigg]\\
    & \leq 2(\etasup{i}_{t, k})^2 \beta^2 D^2 \cdot  \frac{n_i /\BBB{i} - 1}{(n_i-1)}  + \bigg(\beta + \frac{\lambda n}{mn_i}\bigg)^2 \|\www{i}_{t, k} - \prox_{L_i/\lambda}(\www{\glob}_t)\|^2,
  \end{align*}
  where in the second line we let $\overline{\nabla \ell(\www{i}_{t, k} , \zzz{i}_{\bigcdot})}:= \sum_{j\in[n_i]}\nabla \ell(\www{i}_{t, k}, \zzz{i}_j)/n_i$, and in the last line is by the $\beta$-smoothness of $\ell(\cdot, z)$. Thus, we get
  \begin{align}
    & \bbE\bigg[\|\www{i}_{t, k+1} - \prox_{L_i/\lambda}(\www{\glob}_t)\|^2~\bigg|~\calF_{t, k}, i\in\calC_t\bigg] \nonumber\\
    & \leq \bigg[1 - 2\etasup{i}_{t, k}(\mu + \lambda ) + (\etasup{i}_{t, k})^2 (\beta + \lambda)^2\bigg] \|\www{i}_{t, k} - \prox_{L_i/\lambda}(\www{\glob}_t)\|^2 \nonumber\\
    \label{eq:recursion_in_exp_conv_inner_loop}
    & \qquad + 2(\etasup{i}_{t, k})^2 \beta^2 D^2 \cdot  \frac{n_i /\BBB{i} - 1}{(n_i-1)}.
  \end{align}
  We then proceed by induction. Note that if $k + 1\leq \frac{8\beta^2}{\mu^2}$, then we have the following trivial bound:
  \begin{equation}
    \label{eq:induc_hypo_inner_loop}
    \bbE\bigg[\|\www{i}_{t, k} - \prox_{L_i/\lambda}(\www{\glob}_t)\|^2 ~\bigg|~\calF_{t, 0}, i\in\calC_t\bigg] \leq D^2 \leq \frac{8\beta^2D^2}{\mu^2(k+1)},
  \end{equation}  
  where the first inequality is by $\www{i}_{t, k}, \prox_{L_i/\lambda}(\www{\glob}_t)\in\calW$ and the second inequality is by our assumption on $k$. Thus, it suffices to show
  \begin{equation}
    \label{eq:desired_res_inner_loop}
    \bbE\bigg[\|\www{i}_{t, k+1} - \prox_{L_i/\lambda}(\www{\glob}_t)\|^2 ~\bigg|~\calF_{t, 0}, i\in\calC_t\bigg] \leq \frac{8\beta^2 D^2}{\mu^2(k+2)}
  \end{equation}  
  based on the inductive hypothesis \eqref{eq:induc_hypo_inner_loop} and $k+1\geq8\beta^2/\mu^2$. 
  By the recursive relationship \eqref{eq:recursion_in_exp_conv_inner_loop} and taking expectation, we have
  \begin{align*}
    & \bbE\bigg[\|\www{i}_{t, k+1} - \prox_{L_i/\lambda}(\www{\glob}_t)\|^2 ~\bigg|~\calF_{t, 0}, i\in\calC_t\bigg] \\
    & \leq  \bigg[1 - 2\etasup{i}_{t, k}(\mu + \lambda) + (\etasup{i}_{t, k})^2 \bigg(\beta + \lambda\bigg)^2\bigg]  \frac{8\beta^2 D^2}{k+1}+ 2(\etasup{i}_{t, k})^2 \beta^2 D^2 \cdot  \frac{n_i /\BBB{i} - 1}{(n_i-1)}.
  \end{align*}
  Hence \eqref{eq:desired_res_inner_loop} is satisfied if
  \begin{align*}
    & {8\beta^2 D^2} \cdot \bigg[\frac{1}{k+2} - \frac{1}{k+1} + \frac{2\etasup{i}_{t, k}}{k+1}(\mu + \lambda) - \frac{(\etasup{i}_{t, k})^2}{k+1} (\beta + \lambda )^2\bigg]\geq 2(\etasup{i}_{t, k})^2 \beta^2 D^2 \cdot  \frac{n_i /\BBB{i} - 1}{(n_i-1)}.
  \end{align*}  
  By our choice of $\etasup{i}_{t,k}$, the above display is equivalent to
  \begin{align*}
    & {8\beta^2 D^2} \cdot \bigg[-\frac{1}{(k+1)(k+2)}+ \frac{2}{(k+1)^2}- \frac{1}{(k+1)^3} \bigg(\frac{\beta + {\lambda }}{\mu + \lambda}\bigg)^2\bigg]\geq \frac{2\beta^2 D^2}{(\mu + \lambda )^2 (k+1)^2} \cdot \frac{n_i/\BBB{i} - 1}{n_i-1},
  \end{align*}  
  which is further equivalent to
  \begin{align*}
    & {8\beta^2 D^2} \cdot \bigg[-\frac{k+1}{k+2}+ 2 - \frac{1}{k+1} \bigg(\frac{\beta + {\lambda }}{\mu + \lambda}\bigg)^2\bigg]\geq \frac{2\beta^2 D^2}{(\mu + \lambda )^2 } \cdot \frac{n_i/\BBB{i} - 1}{n_i-1}.
  \end{align*}  
  We now claim that
  $$
    \frac{1}{k+1} \bigg(\frac{\beta + {\lambda }}{\mu + \lambda}\bigg)^2 \leq \frac{1}{2}.
  $$
  Indeed, since $k+1\geq 8\beta^2/\mu^2$, (1) if $\lambda \leq \beta$, then the left-hand side above is less than $\frac{4\beta^2}{\mu^2 (k+1)}\leq \frac{1}{2}$; and (2) if $\lambda \geq \beta$, the left-hand side above is less than $\frac{4}{k+1}\leq \frac{\mu^2}{2\beta^2}\leq \frac{1}{2}$. By the above claim, \eqref{eq:desired_res_inner_loop} would hold if
  $$
    {4\beta^2 D^2} 
    \geq \frac{2\beta^2 D^2}{(\mu + \lambda)^2 } \cdot \frac{n_i/\BBB{i} - 1}{n_i-1}.
  $$
  We finish the proof by noting that the right-hand side above is bounded above by
  $
    \frac{2\beta^2 D^2}{\mu^2}.
  $

\subsection{Proof of Lemma \ref{lemma:conv_outer_loop}: Convergence of the Outer Loop}\label{prf:lemma:conv_outer_loop}
  By construction we have
  \begin{align*}
    & \|\www{\glob}_t - \twww{\glob}\|^2 \\
    & = \bigg\|\frac{\lambda m \etasup{\glob}_t}{\BBB{\glob}} \sum_{i\in\calC_t} p_i(\www{\glob}_t - \www{i}_{t+1})\bigg\|^2 \\
    & = \|\www{\glob}_t - \twww{\glob}\|^2 + \bigg\| \frac{\lambda m \etasup{\glob}_t}{\BBB{\glob}}\sum_{i\in\calC_t} p_i(\www{\glob}_t - \www{i}_{t+1}) \bigg\|^2\\
    & \qquad - 2 \hugela \www{\glob}_t - \twww{\glob}, \frac{\lambda m \etasup{\glob}_t}{\BBB{\glob}} \sum_{i\in\calC_t} p_i(\www{\glob}_t - \www{i}_{t+1}) \hugera \\
    & \leq \|\www{\glob}_t - \twww{\glob}\|^2 -  \underbrace{2 \hugela \www{\glob}_t - \twww{\glob}, \frac{\lambda m \etasup{\glob}_t}{\BBB{\glob}}\sum_{i\in\calC_t}p_i\bigg(\www{\glob}_t - \prox_{L_i/\lambda}(\www{\glob}_t)\bigg) \hugera}_{\RN{1}}\\
    & \qquad + \underbrace{2 \bigg\|  \frac{\lambda m \etasup{\glob}_t}{\BBB{\glob}}\sum_{i\in\calC_t}p_i\bigg(\www{\glob}_t - \prox_{L_i/\lambda}(\www{\glob}_t)\bigg) \bigg\|^2}_{\RN{2}}\\
    & \qquad + \underbrace{2\bigg\|  \frac{\lambda m \etasup{\glob}_t}{\BBB{\glob}}\sum_{i\in\calC_t}p_i\bigg(\prox_{L_i/\lambda}(\www{\glob}_t) - \www{i}_{t+1}\bigg) \bigg\|^2 }_{\RN{3}}\\
    & \qquad - \underbrace{2 \hugela \www{\glob}_t - \twww{\glob}, \frac{\lambda m \etasup{\glob}_t}{\BBB{\glob}}\sum_{i\in\calC_t}p_i\bigg( \prox_{L_i/\lambda}(\www{\glob}_t)\bigg) - \www{i}_{t+1} \hugera}_{\RN{4}}.
  \end{align*}
  We first consider Term \RN{1}. Note that $\frac{\lambda m}{\BBB{\glob}}\sum_{i\in\calC_t} p_i(\www{\glob}_t - \prox_{L_i/\lambda}(\www{\glob}_t))$ is an unbiased stochastic gradient of $\sum_{i}p_i F_i$, which is $\mu_F = \lambda\mu/(\lambda+\mu)$-strongly convex. Thus, we have
  \begin{align*}
    \bbE[\RN{1}~|~ \calF_{t-1, K_{t-1}}] & = 2 \etasup{\glob}_t \hugela \www{\glob}_t - \twww{\glob}, \sum_{i\in[m]}p_i \nabla F_i(\www{\glob}_t, S_i)\hugera \\
    & \geq 2\etasup{\glob}_t \mu_F \|\www{\glob}_t - \twww{\glob}\|^2.
  \end{align*}
  Now for Term \RN{2}, we have
  \begin{align*}
    & \bbE[\RN{2} ~|~ \calF_{t-1, K_{t-1}}] \\
    & \leq 2(\etasup{\glob}_t)^2 \cdot \bbE\bigg[ \bigg(\frac{1}{\BBB{\glob}}\sum_{i\in\calC_t}mp_i\bigg)^2 ~|~\calF_{t-1, K_{t-1}}\bigg] \cdot \max_{i\in[m]} \|\nabla F_i(\www{\glob}_t, S_i)\|^2 \\
    & \leq 2(\etasup{\glob}_t)^2 \cdot \bbE\bigg[ \bigg(\frac{1}{\BBB{\glob}}\sum_{i\in\calC_t}mp_i\bigg)^2 ~|~\calF_{t-1, K_{t-1}}\bigg] \cdot \max_{i\in[m]} \cdot (\beta^2 D^2 \land  2\lambda \|\ell\|_\infty \land \lambda^2 D^2)\\
    & \leq  2(\etasup{\glob}_t)^2 \cdot \bigg( \frac{1}{m}\sum_{i\in[m]}(mp_i - 1)^2 + 1 \bigg)  \cdot (\beta^2 D^2 \land  2\lambda \|\ell\|_\infty \land \lambda^2 D^2) \\
    & = 2(\etasup{\glob}_t)^2 m\|\bp\|^2 (\beta^2 D^2 \land  2\lambda \|\ell\|_\infty \land \lambda^2 D^2),
  \end{align*}
  where the second line is by Lemma \ref{lemma:grad_F_i_bound} and the third line is by Lemma \ref{lemma:var_minibatch}.
  For Term \RN{3}, we invoke Lemma \ref{lemma:conv_inner_loop} to get
  \begin{align*}
    \bbE[\RN{3} ~|~ \calF_{t-1, K_{t-1}}] & \leq 2 \lambda^2(\etasup{\glob}_t)^2  \cdot \frac{8\beta^2 D^2}{\mu^2 (K_t + 1)} \cdot \bbE\bigg[ \bigg(\frac{1}{\BBB{\glob}}\sum_{i\in\calC_t}mp_i\bigg)^2 ~|~\calF_{t-1, K_{t-1}}\bigg]\\
    & \leq \frac{16\lambda^2 (\etasup{\glob}_t)^2 \beta^2 D^2 m \|\bp\|^2}{\mu^2 (K_t +1)},
  \end{align*}
  where the last line is again by Lemma \ref{lemma:var_minibatch}.
  For Term \RN{4}, we invoke Young's inequality for products to get
  \begin{align*}
    & \bbE[- \RN{4}~|~\calF_{t-1, K_{t-1}}]\\
    & \leq (\etasup{\glob}_t\mu_F) \|\www{\glob}_t - \twww{\glob}\|^2 + (\etasup{\glob}_t \mu_F)^{-1} \cdot \bbE\bigg[\frac{\RN{3}}{2} ~\bigg|~ \calF_{t-1, K_{t-1}}\bigg]\\
    & \leq (\etasup{\glob}_t\mu_F) \|\www{\glob}_t - \twww{\glob}\|^2 + (\etasup{\glob}_t \mu_F)^{-1} \cdot \frac{8\lambda^2 (\etasup{\glob}_t)^2 \beta^2 D^2 m \|\bp\|^2}{\mu^2 (K_t +1)}.
  \end{align*}
  Summarizing the above bounds on the four terms, we arrive at
  \begin{align*}
    & \bbE\bigg[ \|\www{\glob}_{t+1} - \twww{\glob}\|^2 ~\bigg|~ \calF_{t-1, K_{t-1}}\bigg] \\
    & \leq (1-\etasup{\glob}_t \mu_F) \|\www{\glob}_t - \twww{\glob}\|^2 + 2 \underbrace{(\etasup{\glob}_t)^2 m\|\bp\|^2 (\beta^2 D^2 \land  2\lambda \|\ell\|_\infty \land \lambda^2 D^2)}_{\RN{5}} \\
    & \qquad + \underbrace{\frac{\lambda^2 (\etasup{\glob}_t)^2 \beta^2 D^2 m \|\bp\|^2}{\mu^2(K_t+1)} \cdot \bigg(16 + \frac{8}{\deltasup{\glob}_t \mu_F}\bigg)}_{\RN{6}}.
  \end{align*}
  We claim that $\RN{6} \leq \RN{5}$. Indeed, with our choice of $\etasup{\glob}_t = \frac{2}{\mu_F(t+1)}$, with some algebra, one recognizes that this claim is equivalent to 
  $$
    \frac{20+4t}{\mu^2(K_t + 1)}   \leq \bigg(\frac{1}{\lambda^2} \land \frac{2\|\ell\|_\infty}{\lambda \beta^2 D^2} \land \frac{1}{\beta^2}\bigg),
  $$  
  which is exactly \eqref{eq:inner_rounds_req}. Thus, we have
  \begin{align}
    & \bbE\bigg[ \|\www{\glob}_{t+1} - \twww{\glob}\|^2 ~\bigg|~ \calF_{t-1, K_{t-1}}\bigg] \nonumber \\
    & \leq(1-\etasup{\glob}_t \mu_F) \|\www{\glob}_t - \twww{\glob}\|^2 + 3 \cdot\RN{5}\nonumber\\
    \label{eq:recursion_outer_loop}
    &  = \bigg(1 - \frac{2}{t+1}\bigg)\|\www{\glob}_t - \twww{\glob}\|^2 + \frac{12 m\|\bp\|^2 (\beta^2 D^2 \land 2\lambda \|\ell\|_\infty \land \lambda^2 D^2) }{\mu_F^2 (t+1)^2}.
  \end{align}
  We then proceed by induction. For the base case, we invoke the strong convexity of $\sum_{i}p_i F_i$ and Lemma \ref{lemma:grad_F_i_bound} to get
  \begin{align*}
    \frac{\mu_F^2}{4} \|\www{\glob}_0 - \twww{\glob} \|^2 \leq \bigg\|\sum_{i\in[m]}p_i \nabla F_i(\www{\glob}_0, S_i)\bigg\|^2 \leq  \beta^2 D^2 \land 2\lambda \|\ell\|_\infty \land \lambda^2 D^2.
  \end{align*}
  Along with the fact that $1 = (\sum_{i\in[m]}p_i)^2 \leq m \|\bp\|^2$, we conclude that \eqref{lemma:conv_outer_loop} is true for $t = 0$.  Now assume \eqref{eq:conv_outer_loop} hold for any $0\leq t\leq \tau$. For $t = \tau+1$, using \eqref{eq:recursion_outer_loop} and the inductive hypothesis, we have
  \begin{align*}
    \bbE_{\calA_\SFA}\|\www{\glob}_{\tau+1} - \twww{\glob}\|^2 & \leq \bigg(1 - \frac{2}{\tau+1}\bigg) \frac{12 m\|\bp\|^2 (\beta^2 D^2 \land 2\lambda \|\ell\|_\infty \land \lambda^2 D^2) }{(\tau+1) \mu_F^2} \\
    & \qquad + \frac{12 m\|\bp\|^2 (\beta^2 D^2 \land 2\lambda \|\ell\|_\infty \land \lambda^2 D^2) }{(\tau+1)^2 \mu_F^2} \\
    & = \bigg(\frac{1}{\tau+1} - \frac{1}{(\tau+1)^2}\bigg) \cdot \frac{12 m\|\bp\|^2 (\beta^2 D^2 \land 2\lambda \|\ell\|_\infty \land \lambda^2 D^2) }{\mu_F^2}\\
    & \leq \frac{12 m\|\bp\|^2 (\beta^2 D^2 \land 2\lambda \|\ell\|_\infty \land \lambda^2 D^2) }{(\tau+2)\mu_F^2},
  \end{align*}
  which is the desired result.

\subsection{Auxiliary lemmas}

\begin{lemma}[Convexity and smoothness $F_i$]
\label{lemma:reg_of_F_i}
Under Assumption \ref{assump:regularity}(\cvx), each $F_i$ is $\lambda$-smooth and $\frac{\mu\lambda}{\mu+\lambda}$-strongly convex. 
\end{lemma}
\begin{proof}
    The smoothness is a standard fact about the Moreau envelope. The strongly convex constant of $F_i$ follows from Theorem 2.2 of \cite{lemarechal1997practical}.
\end{proof}

\begin{lemma}[A priori gradient norm bound]
\label{lemma:grad_F_i_bound}
Under Assumption \ref{assump:regularity}(\compact, \cvx), for any $w \in\calW$ and $i\in[m]$, we have
$$
  \|\nabla F_i(w, S_i)\|^2 \leq \beta^2 D^2 \land  2\lambda \|\ell\|_\infty \land \lambda^2 D^2.
$$  
\end{lemma}
\begin{proof}
  Since $\nabla F_i(w, S_i) = \lambda(w - \prox_{L_i/\lambda}(w))$, its norm is trivially bounded by $\lambda D$. Now, since $\prox_{L_i/\lambda}(w)$ achieves a lower objective value than $w$ for the objective function $L_i(\cdot, S_i) + \frac{\lambda}{2}\|w - \cdot\|^2$, we have
  $$
    \frac{\lambda}{2} \|w - \prox_{L_i/\lambda}(w)\|^2 \leq L_i(w, S_i) - L_i(\prox_{L_i/\lambda}(w), S_i) \leq \|\ell\|_\infty,
  $$
  and hence $\|\nabla F_i(w, S_i)\|^2 \leq 2 \lambda \|\ell\|_\infty$. Finally, by the first-order condition, we have
  $$
    \nabla L_i(\prox_{L_i/\lambda}(w), S_i) + \lambda (\prox_{L_i/\lambda}(w) - w) = 0.
  $$
  Hence, we get $\|\nabla F_i(w, S_i)\| = \|\nabla L_i(\prox_{L_i/\lambda}(w), S_i)\|\leq \beta D$. 
\end{proof}

\begin{lemma}[Variance of minibatch sampling]
  \label{lemma:var_minibatch}
  Let $\calB\subseteq[n]$ be a randomly sampled batch with batch size $B$ and let $\{x_i\}_{i=1}^n \subseteq \bbR^d$ be an arbitrary set of vectors, then 
  \begin{equation*}
    \bbE_{\calB}\|\frac{1}{B}\sum_{i\in\calB} x_i\|^2 = \frac{n/B -1}{n(n-1)} \sum_{i\in[n]} \|x_i - \bar x\|^2 + \|\bar x\|^2 \leq \frac{1}{n}\sum_{i\in[n]}\|x_i - \bar x\|^2 + \|\bar x\|^2,
  \end{equation*}
  where $\bar x := \sum_{i\in[n]}x_i/n$.
\end{lemma}
\begin{proof}
  Since $\bbE_{\calB}\sum_{i\in\calB} x_i/B = \bar x$, we have
  \begin{align*}
    \bbE_{\calB} \|\frac{1}{B}\sum_{i\in\calB} x_i\|^2 & = \bbE_{\calB}\|\frac{1}{B}\sum_{i\in\calB} x_i - \bar x\|^2 + \|\bar x\|^2\\
    & = \frac{1}{B^2} \bigg(\sum_{i\in[n]}\indc\{i\in\calB\} \|x_i -\bar x\|^2 + 2\sum_{i<j} \indc\{i, j\in\calB\}\la x_i -\bar x, x_j-\bar x\ra\bigg) + \|\bar x\|^2\\
    & = \frac{1}{B^2} \bigg( \frac{B}{n}\sum_{i\in[n]}\|x_i-\bar x\|^2 +  \frac{2B(B-1)}{n(n-1)}\sum_{i<j}\la x_i - \bar x , x_j -\bar x\ra\bigg) + \|\bar x\|^2,
  \end{align*}
  where the last line is by $\bbP_{\calB}(i\in\calB) = B/n$ and $\bbP_{\calB}(i, j\in\calB) = B(B-1)n^{-1}(n-1)^{-1}$ for any $i\neq j$. Now, since $\sum_{i\in[n]}\|x_i -\bar x\|^2 + 2\sum_{i<j}\la x_i - \bar x , x_j - \bar x\ra = 0$, we arrive at
  \begin{align*}
    \bbE_{\calB}\|\frac{1}{B}\sum_{i\in\calB} x_i\|^2 & = \frac{1}{B^2}\bigg(\frac{B}{n} - \frac{B(B-1)}{n(n-1)}\bigg)\sum_{i}\|x_i - \bar x\|^2 + \|x\|^2\\
    & = \frac{n/B -1}{n(n-1)}\sum_{i\in[n]} \|x_i - \bar x\|^2 + \|\bar x\|^2,
  \end{align*}
  which is the desired result.
\end{proof}

\section{Proofs of Upper Bounds}\label{append:prf_ub}


\subsection{Proof of Theorem \ref{thm:FA}}\label{prf:thm:FA}
In this proof, we let $\hwww{\glob}$ be the global minimizer of \eqref{eq:wted_hard_sharing} and we write $\www{\glob, \bp}_\avg\equiv \www{\glob}_\bp$ when there is no ambiguity.
\begin{proof}[Proof of \eqref{eq:aer_FA}]
We have
\begin{align*}
    0  & = -\sum_{i\in[m]} p_i L_i(\hwww{\glob}(\bS), S_i) +\sum_{i\in[m]} p_i L_i(\hwww{\glob}(\bS), S_i) \\
    & \leq - \sum_{i\in[m]} p_i L_i(\hwww{\glob}(\bS), S_i) + \sum_{i\in[m]} p_iL_i(\www{1}_\star, S_i)\\
    & = -\sum_{i\in[m]} \frac{p_i}{n_i} \sum_{j\in[n_i]} \bigg(  \ell(\hwww{\glob}(\bSSS{i, j}), \zzz{i}_j)- \ell(\www{1}_\star, \zzz{i}_j) \bigg)\\
    & \qquad +  \sum_{i\in[m]} \frac{p_i}{n_i}  \sum_{j\in[n_i]} \bigg(\ell(\hwww{\glob}(\bSSS{i,j}), \zzz{i}_j ) - \ell(\hwww{\glob}(\bS), \zzz{i}_j)\bigg),
\end{align*}
where $\bSSS{i, j}$ stands for the dataset formed by replacing $\zzz{i}_j$ by another $\bsz_{i, j}' \sim \calD_i$, which is independent of everything else.
Taking expectation in both sides, we get
\begin{align*}
    0 & \leq - \sum_{i\in[m]} p_i \cdot \bbE_{\bS, Z_i\sim\calD_i}  [\ell(\hwww{\glob}(\bS), Z_i) - \ell(\www{1}_\star,  Z_i)] \\
    & \qquad  +  \sum_{i\in[m]} \frac{p_i}{n_i}\sum_{j\in[n_i]} \bbE_{\bS,\bsz'_{i, j}}[\ell(\hwww{\glob}(\bSSS{i, j}), \zzz{i}_j) - \ell(\hwww{\glob}(\bS), \zzz{i}_j)]\\
    & = -  \sum_{i\in[m]} p_i \cdot \bbE_{\bS, Z_i\sim\calD_i}  [\ell(\hwww{\glob}(\bS), Z_i) - \ell(\www{i}_\star,  Z_i)] \\
    & \qquad -  \sum_{i\in[m]} p_i\cdot \bbE_{Z_i\sim \calD_i}[\ell(\www{i}_\star, Z_i) - \ell(\www{1}_\star , Z_i)] \\
    & \qquad  +  \sum_{i\in[m]} \frac{p_i}{n_i}\sum_{j\in[n_i]} \bbE_{\bS,\bsz'_{i, j}}[\ell(\hwww{\glob}(\bSSS{i, j}), \zzz{i}_j) - \ell(\hwww{\glob}(\bS), \zzz{i}_j)].
\end{align*}
Noting that $\www{i}_\star$ is the argmin of $\bbE_{Z_i\sim\calD_i}[\ell(\cdot ,Z_i)]$ and invoking the $\beta$-smoothness assumption, we get
\begin{align}
    & \bbE_{\bS}[\aer_\bp(\hwww{\glob})] \nonumber \\
    &\leq \beta \sum_{i\in[m]} p_i \|\www{1}_\star - \www{i}_\star\|^2 +  \sum_{i\in[m]}\frac{p_i}{n_i} \sum_{j\in[n_i]} \bbE_{\bS,\bsz'_{i, j}}[\ell(\hwww{\glob}(\bSSS{i, j}), \zzz{i}_j) - \ell(\hwww{\glob}(\bS), \zzz{i}_j)] \nonumber\\
    & \leq 2 \beta \|\www{1}_\star - \www{\glob}_\bp\|^2+ 2\beta \sum_{i\in[m]}p_i \|\www{i}_\star - \www{\glob}_\bp\| \nonumber\\
    & \qquad  +  \sum_{i\in[m]}\frac{p_i}{n_i} \sum_{j\in[n_i]} \bbE_{\bS,\bsz'_{i, j}}[\ell(\hwww{\glob}(\bSSS{i, j}), \zzz{i}_j) - \ell(\hwww{\glob}(\bS), \zzz{i}_j)]\nonumber.
\end{align}
Taking a weighted average, we arrive at
\begin{align}
    & \bbE_{\bS}[\aer_\bp(\hwww{\glob})] \nonumber \\
    \label{eq:hard_sharing_decomp}
    & \leq 4\beta  R^2 +\sum_{i\in[m]}\frac{p_i}{n_i} \sum_{j\in[n_i]} \bbE_{\bS,\bsz'_{i, j}}[\ell(\hwww{\glob}(\bSSS{i, j}), \zzz{i}_j) - \ell(\hwww{\glob}(\bS), \zzz{i}_j)] 
\end{align}

To bound the second term in the right-hand side above, we bound the \fedstab~of $\hwww{\glob}$. Without loss of generality we consider the first client. By $\mu$-strongly convexity of $L_1$, for any $j_1 \in [n_1]$ we have
\begin{align}
    &\frac{\mu}{2}  \|\hwww{\glob}(\bS) - \hwww{\glob}(\bSSS{1, j_1})\|^2 \nonumber\\
    & \leq  \sum_{i\in[m]} p_i \bigg(L_i(\hwww{\glob}(\bSSS{1,j_1}), S_i) - L_i(\hwww{\glob}(\bS), S_i) \bigg) \nonumber\\
    & = \bigg(\sum_{i\neq 1} p_i L_i(\hwww{\glob}(\bSSS{1, j_1}), S_i) + p_1 L_1(\hwww{\glob}(\bSSS{1, j_1}), \SSS{j_1}_1)\bigg) \nonumber\\
    & \qquad - \bigg(\sum_{i\neq 1} p_i L_i(\hwww{\glob}(\bS), S_i) + p_1 L_1(\hwww{\glob}(\bS) , \SSS{j_1}_1)\bigg) \nonumber\\
    & \qquad + p_1 \bigg( L_1(\hwww{\glob}(\bSSS{1, j_1}), S_1) - L_1(\hwww{\glob}(\bSSS{1, j_1}), \SSS{j_1}_1)\bigg) \nonumber\\
    & \qquad + p_1\bigg( L_1(\hwww{\glob}(\bS), \SSS{j_1}_1) - L_1(\hwww{\glob}(\bS), S_1)\bigg) \nonumber\\
    & \leq p_1 \bigg( L_1(\hwww{\glob}(\bSSS{1, j_1}), S_1) - L_1(\hwww{\glob}(\bSSS{1, j_1}), \SSS{j_1}_1)\bigg)\nonumber\\
    & \qquad  + p_1\bigg(L_1(\hwww{\glob}(\bS), \SSS{j_1}_1) - L_1(\hwww{\glob}(\bS), S_1)\bigg)\nonumber\\
    \label{eq:hard_sharing_stab}
    & = \frac{p_1}{n_1} \bigg(\ell(\hwww{\glob}(\bSSS{1,j_1}), \zzz{1}_{j_1}) - \ell(\hwww{\glob}(\bS), \zzz{1}_{j_1})\bigg)\nonumber\\
    & \qquad  + \frac{p_1}{n_1}\bigg(\ell(\hwww{\glob}(\bS), \bsz'_{1, j_1}) - \ell(\hwww{\glob}(\bSSS{1,j_1}), \bsz'_{1, j_1})\bigg),
\end{align}
where the second inequality is because $\hwww{\glob}(\bSSS{1, j_1})$ minimizes $L_1(\cdot, \SSS{j_1}_1) + \sum_{i\neq 1}n_i L_i(\cdot ,S_i)$. 
By an identical argument as in the proof of Lemma \ref{lemma:loss_stab_to_param_stab}, we have
\begin{align}
    & \ell(\hwww{\glob}(\bSSS{1, j_1}), \zzz{1}_{j_1}) - \ell(\hwww{\glob}(\bS), \zzz{1}_{j_1})\nonumber\\
    \label{eq:hard_sharing_loss_stab_to_param_stab}
    & \leq  \sqrt{2\beta \|\ell\|_\infty} \cdot \|\hwww{\glob}(\bS) - \hwww{\glob}(\bSSS{1,j_1})\|+ \frac{\beta}{2} \|\hwww{\glob}(\bS) - \hwww{\glob}(\bSSS{1,j_1})\|^2
\end{align}
The same bound also holds for $\ell(\hwww{\glob}(\bS), \bsz'_{1, j_1}) - \ell(\hwww{\glob}(\bSSS{1,j_1}) , \bsz'_{1, j_1})$.
Plugging these two bounds to \eqref{eq:hard_sharing_stab} and rearranging terms, we get
\begin{align*}
    & \bigg(\frac{\mu}{2} - \frac{\beta p_1}{n_1} \bigg) \|\hwww{\glob}(\bS) - \hwww{\glob}(\bSSS{1,j_1})\| \leq \frac{2 \sqrt{2\beta \|\ell\|_\infty}\cdot p_1}{n_1}.
\end{align*}
Since $n_1\geq 4\beta p_1/\mu$, we in fact have 
$$
    \frac{\mu}{4}  \|\hwww{\glob}(\bS) - \hwww{\glob}(\bSSS{1,j_1})\| \leq \frac{2 \sqrt{2\beta \|\ell\|_\infty}\cdot p_1}{n_1}.
$$
Plugging the above display back to \eqref{eq:hard_sharing_loss_stab_to_param_stab}, we arrive at 
\begin{align*}
    & \ell(\hwww{\glob}(\bSSS{1, j_1}), \zzz{1}_{j_1}) - \ell(\hwww{\glob}(\bS), \zzz{1}_{j_1})\leq  \frac{16\beta \|\ell\|_\infty p_1}{\mu n_1} \bigg(1 + \frac{4\beta p_1}{\mu n_1}\bigg) \leq \frac{32 \beta \|\ell\|_\infty p_1}{\mu n_1},
\end{align*}
where the last inequality is again by $n_1 \geq 4 \beta p_1/\mu$.
The desired result follows by plugging the above inequality back to \eqref{eq:hard_sharing_decomp}.
\end{proof}

\begin{proof}[Proof of \eqref{eq:ier_FA}]
Without loss of generality we consider the first client. Since $\www{1}_\star$ is the minimizer of $\bbE_{Z_1\sim\calD_1}\ell(\cdot , Z_1)$, by $\beta$-smoothness we have
\begin{align}
  \bbE_{Z_1\sim \calD_1} [\ell(\hwww{\glob}, Z_1) - \ell(\www{1}_\star , Z_1)] 
  &\lesssim \beta \cdot \bbE_{Z_1\sim\calD_1}\|\hwww{\glob} - \www{1}_\star\|^2\nonumber\\
  \label{eq:proof_ier_FA_intermediate}
  & \lesssim \beta \cdot \bbE_{Z_1\sim\calD_1} \|\hwww{\glob}- \www{\glob}_\bp\|^2+ \beta R^2,
\end{align}
where the last inequality is by Part (\iersim) of Assumption \ref{assump:heterogeneity}.
By optimality of $\hwww{\glob}$ and the strong convexity of $L_i$'s, we have
$$
  \bigg\la \sum_{i\in[m]}p_i \nabla L_i(\www{\glob}_\bp, S_i)  , \hwww{\glob}- \www{\glob}_\bp \bigg\ra + \frac{\mu}{2} \|\hwww{\glob} - \www{\glob}_\bp\|^2 \leq 0.
$$
If $\hwww{\glob} - \www{\glob}_\bp = 0$ then we are done. Otherwise, the above display gives 
\begin{align*}
  & \|\hwww{\glob} - \www{\glob}_\bp\| \\
  & \leq \frac{2}{\mu} \|\sum_{i\in[m]}p_i \nabla L_i(\www{\glob}_\bp, S_i)\| \\
  & \leq \frac{2}{\mu} \bigg(\|\sum_{i\in[m]} p_i\nabla L_i(\www{i}_\star, S_i)\| + \big\|\sum_{i\in[m]}p_i \big(\nabla L_i(\www{\glob}_\bp ,  S_i)- \nabla L_i(\www{i}_\star, S_i) \big)\big\|\bigg)\\
  & \leq \frac{2}{\mu} \bigg(\|\sum_{i\in[m]} p_i\nabla L_i(\www{i}_\star, S_i)\| + \beta \sum_{i\in[m]}p_i \|\www{\glob}_\bp - \www{i}_\star\|\bigg)\\
  & \leq \frac{2}{\mu} \bigg(\|\sum_{i\in[m]} p_i\nabla L_i(\www{i}_\star, S_i)\| + \beta R\bigg).
\end{align*}
Thus, we get
\begin{align*}
  \|\hwww{\glob} - \www{\glob}_\bp\|^2 & \leq \frac{8}{\mu^2} \bigg(\|\sum_{i\in[m]} p_i \nabla L_i(\www{i}_\star, S_i)\|^2 + \beta^2 R^2\bigg) .
\end{align*}
Taking expectation with respect to the sample $\bS$ at both sides, we have
\begin{align*}
  \bbE_\bS\|\hwww{\glob} - \www{\glob}_\bp\|^2  
  & \lesssim \frac{1}{\mu^2} \bbE_{\bS} \bigg\|\sum_{i\in[m]} p_i\big( \nabla L_i(\www{i}_\star, S_i) - \bbE_{\bS}[\nabla L_i(\www{i}_\star, S_i)] \big)\bigg\|^2 + \frac{\beta^2R^2}{\mu^2}\\
  & \leq \frac{1}{\mu^2} \cdot \sum_{i\in[m]} \frac{p_i^2 \sigma^2}{n_i} + \frac{\beta^2R^2}{\mu^2}.
\end{align*}
Plugging the above inequality to \eqref{eq:proof_ier_FA_intermediate} gives the desired result.
\end{proof}


\subsection{Proof of Proposition \ref{prop:fedstab_implication}} \label{prf:prop:fedstab_implication}


\begin{proof}[Proof of \eqref{eq:aer_via_stab}]
  By the definitions of the $\aer$ and $\calE_\opt$, we have
  \begin{align*}
    \aer_\bp & = \calE_\opt + \frac{\lambda}{2} \sum_{i\in[m]} p_i \bigg(\|\twww{\glob}(\bS) - \twww{i}(\bS)\|^2 - \|\hwww{\glob}(\bS) - \hwww{i}(\bS)\|^2\bigg) \\
    & \qquad + \sum_{i\in[m]} p_i \bigg(\bbE_{Z_i\sim\calD_i}[\ell(\hwww{i}(\bS), Z_i)] - L_i(\hwww{i}(\bS), S_i)\bigg) \\
    & \qquad + \sum_{i\in[m]} p_i\bigg(L_i(\twww{i}(\bS), S_i) - \bbE_{Z_i\sim\calD_i}[\ell(\www{i}_\star, Z_i)]\bigg).
  \end{align*}
  By the basic inequality \eqref{eq:basic_ineq_aer}, we can bound the $\aer$ by
  \begin{align*}
    & \aer_\bp \\
    & \leq \calE_\opt + \frac{\lambda}{2} \sum_{i\in[m]} p_i \|\www{\glob}_\bp - \www{i}_\star\|^2 
    + \sum_{i\in[m]}p_i \bigg( \bbE_{Z_i\sim\calD_i}[\ell(\hwww{i}(\bS), Z_i)] - L_i(\hwww{i}(\bS), S_i)\bigg)\nonumber\\ 
    & \qquad + \sum_{i\in[m]} p_i \bigg(L_i(\www{i}_\star, S_i) - \bbE_{Z_i\sim\calD_i}[\ell(\www{i}_\star, Z_i)]\bigg). \nonumber
  \end{align*}
  Now, invoking \fedstab, we can further bound the $\aer$ by
  \begin{align*}
    \aer_\bp & \leq \calE_{\opt} + \frac{\lambda}{2} \sum_{i\in[m]} p_i \|\www{\glob}_\bp - \www{i}_\star\|^2 + 2\sum_{i\in[m]} p_i \gamma_i \nonumber\\
    & \qquad + \sum_{i\in[m]}p_i \cdot \frac{1}{n_i} \sum_{j\in[n_i]} \bbE_{z'_{i,j}\sim\calD_i}\bigg[\bbE_{Z_i\sim\calD_i}[\ell(\hwww{i}(\bSSS{i,j}), Z_i)] - \ell(\hwww{i}(\bSSS{i, j}), \zzz{i}_j)\bigg] \nonumber\\
    & \qquad + \sum_{i\in[m]} p_i \bigg(L_i(\www{i}_\star, S_i) - \bbE_{Z_i\sim\calD_i}[\ell(\www{i}_\star, Z_i)]\bigg),
  \end{align*}
  where $\bSSS{i, j}$ is the dataset formed by replacing $\zzz{i}_j$ with a new sample $\bsz'_{i, j}$, and here we are choosing $\bsz'_{i, j}$ to be an independent sample from $\calD_i$. Note that the last two terms of the above display have mean zero under the randomness of the algorithm $\calA$, the dataset $\bS$, and $\{\bsz'_{i, j}: i\in[m], j\in[n_i]\}$. Thus, the desired result follows by taking expectation in both sides. 
\end{proof}

\begin{proof}[Proof of \eqref{eq:ier_via_stab}]
  Without loss of generality we consider the first client. By definitions of $\ier_1$ and $\calE_\opt$, we have
  \begin{align*}
    p_1 \cdot \ier_1 & = \calE_\opt + \sum_{i\in[m]} p_i\bigg(L_i(\twww{i}(\bS), S_i) + \frac{\lambda}{2}\|\twww{\glob}(\bS) - \twww{i}(\bS)\|^2\bigg)\\
    & \qquad - \sum_{i\in[m]} p_i\bigg( L_i(\hwww{i}(\bS) , S_i) + \frac{\lambda}{2}\|\hwww{\glob}(\bS) - \hwww{i}(\bS)\|^2\bigg) \\
    & \qquad + p_1 \bbE_{Z_1\sim\calD_1}[\ell(\hwww{1}(\bS), Z_1) - \ell(\www{1}_\star, Z_1)].
  \end{align*}
  Invoking the basic inequality \eqref{eq:basic_ineq_ier}, with some algebra, we arrive at
  \begin{align*}
    & p_1 \cdot \ier_1 \\
    & \leq \calE_\opt + \frac{ p_1\lambda}{2} \|\hwww{\glob}(\bS) - \www{1}_\star\|^2 + p_1\bigg(\bbE_{Z_1\sim\calD_1}[\ell(\hwww{1}(\bS), Z_1)] - L_1(\hwww{1}(\bS), S_1) \bigg) \nonumber\\
    & \qquad  + p_1\bigg( L_1(\www{1}_\star, S_1) - \bbE_{Z_1\sim\calD_1}[\ell(\www{1}_\star, Z_1)] \bigg).
  \end{align*}
  Now, invoking federated stability for the first client, we can bound its $\ier$ by
  \begin{align*}
    p_1 \cdot \ier_1 & \leq \calE_\opt+ \frac{p_1 \lambda}{2} \|\hwww{\glob}(\bS) - \www{1}_\star\|^2 + 2p_1 \gamma_1 \nonumber\\
    & \qquad + \frac{p_1}{n_1} \sum_{j\in[n_1]} \bbE_{\bsz'_{1, j}\sim\calD_1}\bigg[ \bbE_{Z_1\sim\calD_1}[\ell(\hwww{1}(\bSSS{1, j}), Z_1)] - \ell(\hwww{1}(\bSSS{1, j}), \zzz{1}_j)\bigg] \nonumber\\
    & \qquad  + p_1\bigg( L_1(\www{1}_\star , S_1) - \bbE_{Z_1\sim\calD_1}[\ell(\www{1}_\star, Z_1)] \bigg),
  \end{align*}
  where we recall that $\bSSS{1, j}$ is the dataset formed by replacing $\zzz{1}_j$ with a new sample $\bsz'_{1, j}$, and here we are choosing $\bsz'_{1, j}$ to be an independent sample from $\calD_1$. We finish the proof by taking the expectation with respect to $\calA, \bS, \{\bsz'_{1, j}: j\in[n_1]\}$ at both sides.
\end{proof}

\subsection{Proof of Theorem \ref{thm:aer_sfa_wted}}\label{prf:thm:aer_sfa_wted}
In this proof, we let $\calA=(\hwww{\glob}, \{\hwww{i}\}$ be a generic algorithm that tries to minimize \eqref{eq:soft_sharing_wted}. For notiontional simplicity, we use $a_n\lesssim_{\beta} b_n$ (resp. $a_n \gtrsim_{\beta} b_n$) to denote that $a_n \leq C_\beta b_n$ (resp. $a_n \geq C_\beta b_n$) for large $n$, where $C_\beta$ has explicit dependence on a parameter $\beta$. 

Recall that $(\twww{\glob}, \{\twww{i}\})$ is the global minimizer of \eqref{eq:soft_sharing_wted}, and recall the notations in \eqref{eq:soft_sharing_alt_repr}--\eqref{eq:F_i_argmin}.
We start by bounding the \fedstab~of approximate minimizers of \eqref{eq:soft_sharing_wted}. We need the following definition.
\begin{definition}[Approximate minimizers]
\label{def:approx_min}
We say an algorithm $\calA=(\hwww{\glob}, \{\hwww{i}\}_1^m)$ produces an $(\epsup{\glob}, \{\epsup{i}\}_1^m)$-minimizer of the objective function \eqref{eq:soft_sharing_wted} on the dataset $\bS$ if the following two conditions hold:
\begin{enumerate}
  \item there exist a positive constant $\epsup{\glob}$ such that $\|\hwww{\glob} - \twww{\glob}\| \leq \epsup{\glob}$;
  \item for any $i\in[m]$, there exist a positive constant $\epsup{i}$ such that $\|\hwww{i} - \prox_{L_i/\lambda}(\hwww{\glob})\| \leq \epsup{i}$.
\end{enumerate}
\end{definition}

The stability bound is as follows.
\begin{proposition}[\Fedstab~of approximate minimizers]
\label{prop:stab_approx_min}
Let Assumption \ref{assump:regularity}(\cvx) holds, and consider an algorithm $\calA=(\hwww{\glob}, \{\hwww{i}\}_1^m)$ that produces an $(\epsup{\glob}, \{\epsup{i}\}_1^m)$-minimizer of the objective function \eqref{eq:soft_sharing_wted} on the dataset $\bS$. Assume in addition that
\begin{equation}
  \label{eq:sample_size_assump_for_stab}
  n_i \geq \frac{4\beta}{\mu}, \qquad p_i\lambda \leq \frac{\mu}{16} \qquad \forall i\in[m].
\end{equation}
Then $\calA$ has \fedstab
\begin{equation*}
  \gamma_i \leq \frac{160 \beta \|\ell\|_\infty}{n_i(\mu + \lambda)} + \err_{i},
\end{equation*}
where
\begin{align*}
  \err_{i} & := 2\sqrt{2\beta\|\ell\|_\infty} \bigg[ 4\epsup{\glob}\bigg(\frac{\beta+\lambda}{\mu+\lambda} + \frac{3\lambda}{\mu}\bigg) + \epsup{i} \bigg(\sqrt{\frac{\beta+ \lambda}{\mu+ \lambda}} + \frac{16p_i \lambda}{\mu}\bigg) \bigg] \nonumber\\
  & \qquad + 8\beta^2 \bigg[ 16 (\epsup{\glob})^2 \bigg(\frac{\beta+\lambda}{\mu+\lambda} + \frac{3\lambda}{\mu}\bigg)^2 + (\epsup{i})^2 \bigg(\sqrt{\frac{\beta+ \lambda}{\mu+ \lambda}} + \frac{16p_i \lambda}{\mu}\bigg)^2\bigg]
\end{align*}
is the error term due to not exactly minimizing the soft weight sharing objective \eqref{eq:soft_sharing_wted}.
\end{proposition}
\begin{proof}
  See Appendix \ref{prf:prop:stab_approx_min}.
\end{proof}

Taking the optimization error into account, we have the following result.
\begin{proposition}[\Fedstab~of $\calA_\SFA$]
  \label{prop:stab_sfa}
  Let Assumption \ref{assump:regularity}(\compact, \cvx) and \eqref{eq:sample_size_assump_for_stab} hold. Run $\calA_\SFA$ with hyperparameters chosen as in Lemma \ref{lemma:conv_inner_loop} and \ref{lemma:conv_outer_loop}. Then, as long as 
  \begin{equation}
    \label{eq:it_num_for_stab}
    T \geq C_1\cdot  \lambda^2(\lambda \lor 1)^2 m \|\bp\|^2 n_i^2 , \qquad K_T\geq C_2\cdot\lambda^2(\lambda \lor 1)^2 p_i^2 n_i^2  \qquad \forall i\in[m],
  \end{equation}
  the algorithm $\calA_\SFA$ have expected \fedstab  
  \begin{equation*}
    \bbE_{\calA_{\SFA}}[\gamma_i] \leq  C\cdot \frac{\beta\|\ell\|_\infty}{n_i(\mu+\lambda)},
  \end{equation*}
  where $C_1, C_2$ are two constants only depending on $(\mu, \beta, \|\ell\|_\infty, D)$, and $C$ is an absolute constant.
\end{proposition}
\begin{proof}
  By Proposition \ref{prop:stab_approx_min}, it suffices to upper bound the error term $\err_i$ by a constant multiple of $\frac{\beta\|\ell\|_\infty}{n_i(\mu+\lambda)}$. Invoking Lemma \ref{lemma:conv_outer_loop}, we have
  \begin{align*}
    & \bigl(\bbE_{\calA_\SFA}[\epsup{\glob}]\bigr)^2 \leq \bbE_{\calA_\SFA}[(\epsup{\glob})^2] \leq \frac{12 (\lambda +\mu)^2 m\|\bp\|^2(\beta^2D^2 \land 2\lambda \|\ell\|_\infty \land \lambda^2 D^2)}{\lambda^2 \mu^2 (T+1)}.
  \end{align*}  
  This gives
  \begin{equation}
    \label{eq:epsl_glob_big_O_bd}
    \bbE_{\calA_\SFA}[(\epsup{\glob})^2] \lesssim_{(\mu, \beta, \|\mu\|_\infty, D)}  \frac{(\lambda \lor 1)^2 m\|\bp\|^2 (1\land \lambda \land \lambda^2)}{\lambda^2 (T+1)} \lesssim \frac{m\|\bp\|^2}{T+1} ,
  \end{equation}  
  where we recall that $a_n \lesssim_{(\mu, \beta, \|\mu\|_\infty, D)} b_n$ means $|a_n| \leq C b_n$ for a constant $C$ that only depending on $(\mu, \beta, \|\mu\|_\infty, D)$, and the last inequality follows from $(\lambda \lor 1)^2 (1 \land \lambda \land \lambda^2) \leq \lambda^2$ regardless $\lambda \geq 1$ or $\lambda \leq 1$. Meanwhile, by Lemma \ref{lemma:conv_inner_loop}, we have
  $$
    \bigl(\bbE_{\calA_\SFA}[\epsup{i})]\bigr)^2 \leq \bbE_{\calA_\SFA}[(\epsup{i})^2] \leq \frac{8\beta^2 D^2}{\mu^2 (K_T + 1)} \lesssim_{(\beta, D)} \frac{1}{K_T+1}.
  $$
  Recalling the definition of $\err_i$, we have
  \begin{align*}
  & \bbE_{\calA_\SFA}[\err_i] \\
  & \lesssim_{(\mu, \beta, \|\mu\|_\infty, D)}  \lambda \bbE_{\calA_\SFA}[\epsup{\glob}]  + \lambda  p_i \bbE_{\calA_\SFA}[\epsup{i}] + \lambda^2 \bbE_{\calA_\SFA}[(\epsup{\glob})^2] + p_i^2\lambda^2 \bbE_{\calA_\SFA}[(\epsup{i})^2] \\
  & \lesssim_{(\mu, \beta, \|\mu\|_\infty, D)} \frac{\lambda \sqrt{m} \|\bp\|}{\sqrt{T+1}} + \frac{\lambda p_i }{\sqrt{K_T+1}} + \frac{\lambda^2 m \|\bp\|^2}{T+1} +  \frac{p_i^2 \lambda^2}{K_T+1}.
  \end{align*}
  Thus, it suffices to require
  \begin{align*}
    \sqrt{T} & \gtrsim_{(\mu, \beta, \|\mu\|_\infty, D)} \lambda \sqrt{m} \|\bp\| n_i(\mu +\lambda ), \qquad T  \gtrsim_{(\mu, \beta, \|\mu\|_\infty, D)}  \lambda^2 m \|\bp\|^2 n_i(\mu+\lambda), \\ 
    \sqrt{K_T} & \gtrsim_{(\mu, \beta, \|\mu\|_\infty, D)} \lambda p_i n_i(\mu+\lambda), \qquad  K_t\gtrsim_{(\mu, \beta, \|\mu\|_\infty, D)} p_i^2 \lambda^2 n_i (\mu+\lambda),
  \end{align*}  
  which is equivalent to
  \begin{align*}
    T & \gtrsim_{(\mu, \beta, \|\mu\|_\infty, D)} \max\{\lambda^2 m \|\bp\|^2 n_i^2 (\lambda \lor 1)^2, \lambda^2 m \|\bp\|^2 n_i(\lambda \lor 1) \} = \lambda^2 m \|\bp\|^2 n_i^2 (\lambda \lor 1)^2\\
    K_T & \gtrsim_{(\mu, \beta, \|\mu\|_\infty, D)} \max\{ \lambda^2 p_i^2 n_i^2 (\mu\lor 1)^2, p_i^2 \lambda^2 n_i (\mu \lor 1) \} = \lambda^2 p_i^2 n_i^2 (\lambda \lor 1)^2,
  \end{align*}  
  which is exactly \eqref{eq:it_num_for_stab}.
\end{proof}

Combining the above proposition with Proposition \ref{prop:fedstab_implication}. we get the following result.
\begin{proposition}[$\lambda$-dependent bound on the AER]
\label{prop:lambda_dep_bd_aer}
Let Assumption \ref{assump:regularity}(\compact, \cvx) and \eqref{eq:sample_size_assump_for_stab} hold. Run $\calA_\SFA$ with hyperparameters chosen as in Lemma \ref{lemma:conv_inner_loop} and \ref{lemma:conv_outer_loop}. Then, as long as
\begin{align}
  T & \geq C_1 \cdot \lambda(\lambda \lor 1) m \|\bp\|^2 \cdot \bigg(\bigl[ {\sum_{s\in[m]} p_s/n_s}\bigr]^{-1} \lor \bigl[\lambda(\lambda \lor 1) n_i^2 \bigr]\bigg) , \nonumber\\
  \label{eq:it_num_for_aer}
  K_T & \geq C_2 \cdot (\lambda+1)^2 \cdot \bigg(\bigl[{\sum_{s\in[m]}p_s/n_s}\bigr]^{-1} \lor \bigl[\lambda^2 p_i^2 n_i^2\bigr]\bigg), 
\end{align}
for any $i\in[m]$, the algorithm $\calA_\SFA$ satisfies
\begin{equation*}
  \bbE_{\calA_{\SFA}, \bS}[\aer_\bp(\calA_\SFA)] \leq C \cdot  \frac{\beta\|\ell\|_\infty}{\mu+\lambda}\sum_{i\in[m]}\frac{p_i}{n_i}+ \frac{\lambda}{2} \sum_{i\in[m]} p_i\| \www{\glob}_\bp - \www{i}_\star \|^2,
\end{equation*}
where $C_1, C_2$ are two constants only depending on $(\mu, \beta, \|\ell\|_\infty, D)$, and $C$ is an absolute constant.
\end{proposition}
\begin{proof}
  In view of Propositions \ref{prop:fedstab_implication} and \ref{prop:stab_sfa}, it suffices to set $T, K_T$ such that (1) \eqref{eq:it_num_for_stab} is satisfied; and (2) $\bbE_{\calA_\SFA}[\calE_\opt]$ is upper bounded by a constant multiple of $\frac{\beta\|\ell\|_\infty}{\mu+\lambda}\sum_{i\in[m]}\frac{p_i}{n_i}$. To achive the second goal, note that by Proposition \ref{prop:opt_err_sfa}, the optimization error is bounded by
  \begin{equation}
    \label{eq:opt_err_big_O_bd}
    \bbE_{\calA_{\SFA}}[\calE_\opt] \lesssim_{(\mu, \beta, \|\mu\|_\infty, D)} \frac{\lambda \lor 1}{K_T+1} + \frac{\lambda m \|\bp\|^2}{T+1}.
  \end{equation}  
  Thus, it suffices to require
  $
    T \gtrsim_{(\mu, \beta, \|\mu\|_\infty, D)} \frac{\lambda(\lambda\lor 1) m \|\bp\|^2}{\sum_{i\in[m]}p_i/n_i}
  $
  and 
  $K_T \gtrsim_{(\mu, \beta, \|\mu\|_\infty, D)} \frac{(\lambda\lor 1)^2}{\sum_{i\in[m]}p_i/n_i}.
  $
  This requirement, combined with \eqref{eq:it_num_for_stab}, is exactly \eqref{eq:it_num_for_aer}. 
\end{proof}

With the above proposition at hand, we are ready to give our proof of Theorem \ref{thm:aer_sfa_wted}.
\begin{proof}[Proof of Theorem \ref{thm:aer_sfa_wted}]
  We first define the following three events:
  \begin{align*}
    \sfA:= \bigg\{ R \geq \sqrt{\sum_{i\in[m]}\frac{p_i}{n_i}}\bigg\}, ~ \sfB:= \bigg\{ \frac{\sum_{i\in[m]}p_i^2/n_i}{\sqrt{\sum_{i\in[m]}{p_i}/{n_i}}}\leq  R \leq \sqrt{\sum_{i\in[m]}\frac{p_i}{n_i}}\bigg\} , ~ \sfC:= \bigg\{ R \leq \frac{\sum_{i\in[m]}p_i^2/n_i}{\sqrt{\sum_{i\in[m]}{p_i}/{n_i}}}\bigg\}.
  \end{align*}
  We then choose $\lambda$ to be
  \begin{align*}
    \lambda = \frac{\mu}{16 R^2} \sum_{i\in[m]}\frac{p_i}{n_i} \cdot \indc_\sfA + \frac{\mu}{16 C_\bp  R} \sqrt{\sum_{i\in[m]}\frac{p_i}{n_i}} \cdot \indc_\sfB + \frac{\mu }{16 C_\bp \sum_{i\in[m]}p_i^2/n_i} \sum_{i\in[m]}\frac{p_i}{n_i} \cdot \indc_\sfC.
  \end{align*}
  We now consider the three events separately.
  \begin{enumerate}
    \item If $\sfA$ holds, then 
    $
    p_i \lambda = \frac{p_i\mu}{16  R^2} \sum_{i\in[m]} \frac{p_i}{\mu_i} \leq  \frac{p_i \mu}{16} \leq \frac{\mu}{16}.
    $
    Thus we can invoke Proposition \ref{prop:lambda_dep_bd_aer} to get 
    $$
      \bbE_{\calA_\SFA, \bS}[\aer_\bp]\leq \bigg(\frac{C\beta\|\ell\|_\infty}{\mu} + \frac{\mu}{32}\bigg) \sum_{i\in[m]} \frac{p_i}{n_i} \lesssim \textnormal{right-hand side of } \eqref{eq:aer_sfa_full}.
    $$
    \item If $\sfB$ holds, then
    $
      p_i\lambda = \frac{p_i\mu}{16C_\bp R}\sqrt{\sum_{i\in[m]}\frac{p_i}{n_i}} \leq \frac{p_{\max} \mu N}{16 C_\bp \sum_{i\in[m]}p_i^2/n_i} \sum_{i\in[m]}\frac{p_i}{n_i} \leq \frac{\mu}{16},
    $
    where the last inequality is by the definition of $C_\bp$.
    Hence, by Proposition \ref{prop:lambda_dep_bd_aer}, we have
    \begin{align*}
      \bbE_{\calA_\SFA, \bS}[\aer_\bp] & \leq \bigg(\frac{16C C_\bp\beta\|\ell\|_\infty}{\mu} + \frac{\mu}{32 C_\bp}\bigg)     R \sqrt{\sum_{i\in[m]}\frac{p_i}{n_i}}  \lesssim  \textnormal{right-hand side of } \eqref{eq:aer_sfa_full}.
    \end{align*}
    \item If $\sfC$ holds, then 
    $
      p_i \lambda = \frac{p_i \mu }{16 C_\bp\sum_{i\in[m]}p_i^2/n_i} \sum_{i\in[m]} \frac{p_i}{n_i} \leq \frac{\mu}{16},
    $
    and thus Proposition \ref{prop:lambda_dep_bd_aer} gives
    \begin{align*}
      \bbE_{\calA_\SFA, \bS}[\aer_\bp]\leq \bigg(\frac{ 16 C C_\bp  \beta\|\ell\|_\infty}{\mu} + \frac{\mu }{32 C_\bp}\bigg) \sum_{i\in[m]}\frac{p_i^2}{n_i} \lesssim \textnormal{right-hand side of } \eqref{eq:aer_sfa_full}.
    \end{align*}  
  \end{enumerate} 
  The desired result follows by combining the above three cases together. 
\end{proof}

\subsubsection{Proof of Proposition \ref{prop:stab_approx_min}: Stability of Approximate Minimizers} \label{prf:prop:stab_approx_min}
We first present two lemmas, from which Proposition \ref{prop:stab_approx_min} will follow. 
\begin{lemma}[\Fedstab~of approximate minimizers, Part \RN{1}]
\label{lemma:stab_approx_min_part1}
Let Assumption \ref{assump:regularity}(\cvx) holds, and consider an algorithm $\calA=(\hwww{\glob}, \{\hwww{i}\}_1^m)$ that satisfies the following conditions:
\begin{enumerate}
\item there exist positive constants $\deltasup{\glob}, \zetasup{\glob}$ such that
\begin{align}
  \label{eq:delta_opt_glob_model}
  & \sum_{i\in[m]}p_i F_i(\hwww{\glob}, S_i) \leq \deltasup{\glob}+\sum_{i\in[m]}p_i F_i(\twww{\glob}, S_i), \\
  \label{eq:ep_stat_glob_model}
  & \|\sum_{i\in[m]}p_i \nabla F_i(\hwww{\glob}, S_i)\|  \leq \zetasup{\glob}.
\end{align} 
\item for any $i\in[m]$, there exist positive constants $\{\deltasup{i}, \zetasup{i}, \epsup{i}\}_{i=1}^m$ such that 
\begin{align}
  \label{eq:delta_opt_local_model}
  & L_i(\hwww{i}, S_i) + \frac{\lambda}{2} \| \hwww{\glob}-\hwww{i} \|^2  \leq \deltasup{i} + F_i(\hwww{\glob}, S_i), \\
  \label{eq:ep_stat_local_model}
  & \|\nabla L_i(\hwww{i}, S_i) + \lambda ( \hwww{i} - \hwww{\glob})\|  \leq \zetasup{i} ,\\
  \label{eq:xi_approx_local_model}
  & \|\hwww{i} - \prox_{L_i/\lambda}(\hwww{\glob}) \| \leq \epsup{i}.   
\end{align}
\end{enumerate}
Assume in addition that \eqref{eq:sample_size_assump_for_stab} holds.
Then $\calA$ has \fedstab
\begin{align}
  \label{eq:stab_approx_min_part1}
  \gamma_i \leq \frac{160 \beta \|\ell\|_\infty}{n_i(\mu + \lambda)} + \sqrt{2\beta \|\ell\|_\infty} \cdot \calE_{\lambda, i} + \beta \calE_{\lambda, i}^2,
\end{align}
where 
\begin{equation}
  \label{eq:err_in_stab_approx_min_part1}
  \calE_{\lambda, i} := \frac{8 \zetasup{i}}{\mu+\lambda} + \sqrt{\frac{8\deltasup{i}}{\mu+\lambda}} + {8\mu^{-1}\bigg(2\zetasup{\glob} + 4 p_i \lambda \epsup{i} + \sqrt{\frac{2\mu\lambda \deltasup{\glob}}{\mu+\lambda}}\bigg)}
\end{equation}
is the error term due to not exactly minimizing \eqref{eq:soft_sharing_wted}.
\end{lemma}

\begin{lemma}[\Fedstab~of approximate minimizers, Part \RN{2}]
\label{lemma:stab_approx_min_part2}
Let Assumption \ref{assump:regularity}(\cvx) holds and consider an algorithm $\calA=(\hwww{\glob}, \{\hwww{i}\}_1^m)$ that produces an $(\epsup{\glob}, \{\epsup{i}\}_1^m)$-minimizer in the sense of Definition \ref{def:approx_min}. Then $\calA$ also satisfies Equations \eqref{eq:delta_opt_glob_model}---\eqref{eq:xi_approx_local_model} with
\begin{align*}
  \deltasup{\glob} & = \frac{\lambda }{2}\epsup{\glob},  \ \ \zetasup{\glob} = \lambda \epsup{\glob}, \ \ \deltasup{i} = \frac{\beta+\lambda}{2} \epsup{i}, \ \ \zetasup{i}= (\beta+\lambda) \epsup{i}. 
\end{align*}
\end{lemma}
\begin{proof}
  These correspondences are consequences of $\lambda$-smoothness of $F_i$ and $(\beta+\lambda)$-smoothness of $L_i(\cdot, S_i) + \frac{\lambda}{2}\|\hwww{\glob} - \cdot\|^2$. We omit the details.  
\end{proof}

With the above two lemmas at hand, the proof of Proposition \ref{prop:stab_approx_min} is purely computational:
\begin{proof}[Proof of Proposition \ref{prop:stab_approx_min} given Lemma \ref{lemma:stab_approx_min_part1} and \ref{lemma:stab_approx_min_part2}]
Invoking \ref{lemma:stab_approx_min_part2}, the error term $\calE_{\lambda, i}$ defined in Equation \ref{eq:err_in_stab_approx_min_part1} can be bounded above by
\begin{align*}
  \calE_{\lambda, i} & \leq \frac{8(\beta+\lambda)}{\mu+\lambda} \cdot \epsup{\glob}+ \sqrt{\frac{4(\beta+\lambda)}{\mu+ \lambda}} \cdot \epsup{i} + \frac{8}{\mu} \bigg(2\lambda \epsup{\glob} + 4p_i \lambda \epsup{i} + \sqrt{\frac{\mu\lambda^2}{\mu+\lambda}}\bigg)\\
  & = 8\epsup{\glob}\bigg( \frac{\beta+\lambda}{\mu+\lambda} + \frac{2\lambda}{\mu} + \frac{\lambda}{\sqrt{\mu(\mu+\lambda)}} \bigg) + 2\epsup{i}\bigg(\sqrt{\frac{\beta+\lambda}{\mu+\lambda}} + \frac{16p_i\lambda}{\mu}\bigg)\\
  & \leq 8\epsup{\glob}\bigg( \frac{\beta+\lambda}{\mu+\lambda} + \frac{3\lambda}{\mu} \bigg) + 2\epsup{i}\bigg(\sqrt{\frac{\beta+\lambda}{\mu+\lambda}} + \frac{16p_i\lambda}{\mu}\bigg).
\end{align*}
This gives
$$
  \calE_{\lambda, i}^2 \leq  128 (\epsup{\glob})^2\bigg( \frac{\beta+\lambda}{\mu+\lambda} + \frac{3\lambda}{\mu} \bigg)^2 + 8(\epsup{i})^2\bigg(\sqrt{\frac{\beta+\lambda}{\mu+\lambda}} + \frac{16p_i\lambda}{\mu}\bigg)^2.
$$
Plugging the above two displays to \eqref{eq:stab_approx_min_part1} gives the desired result.
\end{proof}

We now present our proof of Lemma \ref{lemma:stab_approx_min_part1}.
We start by stating and proving several useful lemmas.

\begin{lemma}[From loss stability to parameter stability]
  \label{lemma:loss_stab_to_param_stab}
  Let Assumption \ref{assump:regularity}(\cvx) holds. Then the algorithm $\calA=(\hwww{\glob}, \{\hwww{i}\})$ has \fedstab
  \begin{equation*}
    \gamma_i \leq \sqrt{2\beta \|\ell\|_\infty}\cdot\|\hwww{i}(\bS) - \hwww{i}(\bSSS{i, j_i})\| + \frac{\beta}{2} \|\hwww{i}(\bS)-\hwww{i}(\bSSS{i,j_i})\|^2.
  \end{equation*}
\end{lemma}
\begin{proof}
  This lemma has implicitly appeared in the proofs of many stability-based generalization bounds (see, e.g., Section 13.3.2 of \cite{shalev2014understanding}), and we provide a proof for completeness.
  By $\beta$-smoothness, for an arbitrary $z\in\calZ$ we have
  \begin{align}
        & \ell(\hwww{i}(\bS), z) - \ell(\hwww{i}(\bSSS{i,j_i}), z) \nonumber\\
        & \leq \bigla \nabla \ell(\hwww{i}(\bSSS{i, j_i}), z), \hwww{i}(\bS) - \hwww{i}(\bSSS{i, j_i}) \bigra + \frac{\beta}{2} \| \hwww{i}(\bS) - \hwww{i}(\bSSS{i, j_i}) \|^2\nonumber\\
        & \leq \| \nabla \ell(\hwww{i}(\bSSS{i, j_i}), z)\| \cdot \| \hwww{i}(\bS) - \hwww{i}(\bSSS{i, j_i})\|  + \frac{\beta}{2} \| \hwww{i}(\bS) - \hwww{i}(\bSSS{i, j_i}) \|^2\nonumber\\
        & \leq \sqrt{2\beta\bigg( \ell(\hwww{i}(\bSSS{i, j_i}), z) - \min_{\www{i}\in\calW} \ell(\www{i}, z)\bigg)} \cdot \| \hwww{i}(\bS) - \hwww{i}(\bSSS{i, j_i})\| \nonumber \\
        & \qquad  + \frac{\beta}{2} \| \hwww{i}(\bS) - \hwww{i}(\bSSS{i, j_i}) \|^2\nonumber\\
        & \leq \sqrt{2\beta \|\ell\|_\infty}\cdot \| \hwww{i}(\bS) - \hwww{i}(\bSSS{i, j_i})\|  + \frac{\beta}{2} \| \hwww{i}(\bS) - \hwww{i}(\bSSS{i, j_i}) \|^2\nonumber,
    \end{align}
where the last inequality follows from boundedness of $\ell$. By a nearly identical argument, the above upper bound also holds for $-\ell(\hwww{i}(\bS), z) + \ell(\hwww{i}(\bSSS{i,j_i}), z)$, and the desired result follows.
\end{proof}

\begin{lemma}[Local stability implies global stability]
\label{lemma:stab_local_implies_stab_global}
  Assume Assumption \ref{assump:regularity}(\cvx) holds and consider an algorithm $\calA = (\hwww{\glob}, \{\hwww{i}\}_1^m)$ that satisfies Equations \eqref{eq:delta_opt_glob_model}, \eqref{eq:ep_stat_glob_model} and \eqref{eq:xi_approx_local_model}. 
  Then for any $i\in[m], j_i\in[n_i]$, we have
  \begin{align}
    & \|\hwww{\glob}(\bSSS{i, j_i})  - \hwww{\glob}(\bS) \| \nonumber\\
      \label{eq:stab_local_implies_stab_global}
    & \leq \frac{\lambda + \mu}{\lambda \mu} \bigg(2 \zetasup{\glob}  + \sqrt{\frac{2 \lambda \mu \deltasup{\glob}}{\lambda + \mu}} + {4 p_i\lambda \epsup{i}} + {2 p_i\lambda}\|\hwww{i}(\bSSS{i, j_1}) - \hwww{i}(\bS)\|\bigg).
  \end{align}
\end{lemma}
\begin{proof}
  Without loss of generality we consider the first client. Let $\mu_F$ be the strongly convex constant of $\sum_{i}p_i F_i$, which, by Lemma \ref{lemma:reg_of_F_i}, is equal to $\sum_{i}p_i \cdot\frac{\mu \lambda}{\mu + \lambda} = \lambda\mu/(\lambda + \mu)$. Now, by strong convexity, we have
  \begin{align*}
    & \frac{\mu_F}{2} \|\hwww{\glob}(\bS)  - \hwww{\glob}(\bSSS{1, j_1})\|^2 \\
    & \leq \sum_{i\in[m]} p_i \bigg( F_i (\hwww{\glob}(\bSSS{1, j_1}), S_i) - F_i(\hwww{\glob}(\bS), S_i)\bigg) \\
    & \qquad + \hugela \sum_{i\in[m]} p_i \nabla F_i(\hwww{\glob}(\bS), S_i), \hwww{\glob}(\bSSS{1, j_1} - \hwww{\glob}(\bS)) \hugera \\
    & \overset{\eqref{eq:ep_stat_glob_model}}{\leq} \sum_{i\in[m]} p_i \bigg( F_i (\hwww{\glob}(\bSSS{1, j_1}), S_i) - F_i(\hwww{\glob}(\bS), S_i)\bigg) \\
    & \qquad + \zetasup{\glob}\|\hwww{\glob}(\bSSS{1, j_1}) - \hwww{\glob}(\bS)\| \\
    & = \bigg(p_1 F_1(\hwww{\glob}(\bSSS{1, j_1}), \SSS{j_1}_1) + \sum_{i\neq 1} p_i F_i(\hwww{\glob}(\bSSS{1, j_1}), S_i)\bigg) \\
    & \qquad - \bigg(p_1 F_1(\hwww{\glob}(\bS, \SSS{j_1}_1) + \sum_{i\neq 1} p_i F_i(\hwww{\glob}(\bS, S_i)\bigg)  \\
    & \qquad + p_1 \bigg(F_1(\hwww{\glob}(\bSSS{1, j_1}), S_1) - F_1(\hwww{\glob}(\bSSS{1,j_1}), \SSS{j_1}_1) \\
    & \qquad \qquad \qquad + F_1(\hwww{\glob}(\bS), \SSS{j_1}_1) - F_1(\hwww{\glob}(\bS), S_1)\bigg)\\
    & \qquad + \zetasup{\glob} \|\hwww{\glob}(\bSSS{1, j_1}) - \hwww{\glob}(\bS)\| \\
    & \overset{\eqref{eq:delta_opt_glob_model}}{\leq} \deltasup{\glob} + \zetasup{\glob} \|\hwww{\glob}(\bSSS{1, j_1}) - \hwww{\glob}(\bS)\| \\ 
    & \qquad + p_1\bigg(F_1(\hwww{\glob}(\bSSS{1,j_1}), S_1) - F_1(\hwww{\glob}(\bS), S_1) \\
    & \qquad \qquad \qquad + F_1(\hwww{\glob}(\bS), \SSS{j_1}_1) - F_1(\hwww{\glob}(\bSSS{1, j_1}), \SSS{j_1}_1)\bigg).
  \end{align*}
  Since $F_1$ is $\lambda$-smooth by Lemma \ref{lemma:reg_of_F_i}, we can proceed by
  \begin{align*}
    & \frac{\mu_F}{2} \|\hwww{\glob}(\bS)  - \hwww{\glob}(\bSSS{1, j_1})\|^2 \\
    & \leq \deltasup{\glob} + \zetasup{\glob} \|\hwww{\glob}(\bSSS{1, j_1}) - \hwww{\glob}(\bS)\|  + p_1 \lambda \| \hwww{\glob}(\bSSS{1,j_1}) - \hwww{\glob}(\bS) \|^2 \\
    & \qquad + p_1 \hugela \nabla F_1(\hwww{\glob}(\bS), S_1) - \nabla F_1(\hwww{\glob}(\bSSS{1,j_1}), \SSS{j_1}_1) , \hwww{\glob}(\bSSS{1,j_1}) - \hwww{\glob}(\bS)\hugera.
  \end{align*}
  Since $\nabla F_1(\www{\glob}, S_1) = \lambda\bigg(\www{\glob} - \prox_{L_1/\lambda}(\www{\glob}, S_1)\bigg)$, with some algebra, the right-hand side above is in fact equal to
  \begin{align*}
    & \deltasup{\glob} + \zetasup{\glob} \|\hwww{\glob}(\bSSS{1, j_1}) - \hwww{\glob}(\bS)\|  \\
    & \qquad + p_1\lambda \hugela \hwww{1}(\bS) - \prox_{L_1/\lambda}(\hwww{\glob}(\bS), S_1) , \hwww{\glob}(\bSSS{1,j_1}) - \hwww{\glob}(\bS) \hugera \\
    & \qquad + p_1 \lambda\hugela \prox_{L_1/\lambda}(\hwww{\glob}(\bSSS{1,j_1}), \SSS{j_1}_1) -\hwww{1}(\bSSS{1, j_1} , \hwww{\glob}(\bSSS{1,j_1}) - \hwww{\glob}(\bS) \hugera \\
    & \qquad + p_1 \lambda \hugela \hwww{1}(\bSSS{1,j_1}) - \hwww{1}(\bS), \hwww{\glob}(\bSSS{1,j_1}) - \hwww{\glob}(\bS) \hugera \\
    & \overset{\eqref{eq:xi_approx_local_model}}{\leq} \deltasup{\glob} + (\zetasup{\glob} + 2p_1\lambda) \|\hwww{\glob}(\bSSS{1,j_1}) - \hwww{\glob}(\bS)\| \\
    & \qquad + p_1 \lambda  \|\hwww{i}(\bSSS{1,j_1}) - \hwww{i}(\bS)\|  \|\hwww{\glob}(\bSSS{1,j_1}) - \hwww{\glob}(\bS)\|.
  \end{align*}
  The above bound gives a quadratic inequality: if we let $\mathsf{s}_G:= \|\hwww{\glob}(\bSSS{1,j_1}) - \hwww{\glob}(\bS)\|$ and $\mathsf{s}_1:=\|\hwww{i}(\bSSS{1,j_1}) - \hwww{i}(\bS)\|$, then the above bound can be written as 
  \begin{align*}
    \frac{\mu_F}{2}\cdot \mathsf{s}_G^2 - (\zetasup{\glob}+ 2p_1\lambda \epsup{1} + p_1 \lambda \mathsf{s}_1) \cdot \mathsf{s}_G - \deltasup{\glob} \leq 0.
  \end{align*}
  Solving this inequality gives
  \begin{align*}
    \mathsf{s}_G & \leq \frac{1}{\mu_F} \cdot \bigg[\zetasup{\glob}+ 2p_1\lambda \epsup{1} + p_1 \lambda \mathsf{s}_1 + \sqrt{(\zetasup{\glob}+ 2p_1\lambda \epsup{1} + p_1 \lambda \mathsf{s}_1)^2 + 2\mu_F\deltasup{\glob}} \bigg] \\
    & \leq \frac{1}{\mu_F} \bigg(2\zetasup{\glob}+ 4p_1\lambda \epsup{1} + 2p_1 \lambda \mathsf{s}_1 +  \sqrt{2\mu_F\deltasup{\glob}}\bigg),
  \end{align*}
  which is exactly \eqref{eq:stab_local_implies_stab_global}.
\end{proof}

\begin{lemma}[Parameter stability]
  \label{lemma:param_stab}
  Under the same assumptions as Proposition \ref{prop:stab_approx_min}, for any $i\in[m], j_i\in[n_i]$, we have
  \begin{align*}  
    & \|\hwww{i}(\bSSS{i, j_i}) - \hwww{i}(\bS)\| 
    \leq \frac{16\sqrt{2 \beta \|\ell\|_\infty}}{n_i(\mu+\lambda)} + \calE_{\lambda, i}.
  \end{align*}
\end{lemma}
\begin{proof}
    Without loss of generality we consider the first client. Since $L_1(\cdot, S_1) + \frac{\lambda}{2}\|\hwww{\glob}(\bS) - \cdot \|^2$ is $(\mu+\lambda)$-strongly convex, we have
  \begin{align*}
    & \frac{1}{2}(\mu+\lambda) \|\hwww{1}(\bS) - \hwww{1}(\bSSS{1,j_1})\|^2\\
    & \leq \bigg(L_1(\hwww{1}(\bSSS{1, j_1}), S_1) + \frac{\lambda}{2}\|\hwww{\glob}(\bS) - \hwww{1}(\bSSS{1, j_1})\|^2\bigg) \\
    & \qquad - \bigg(L_1(\hwww{1}(\bS), S_1) + \frac{\lambda}{2}\|\hwww{\glob}(\bS) - \hwww{1}(\bS)\|^2\bigg) \\
    & \qquad + \hugela \nabla L_1 (\hwww{1}(\bS), S_1) + \lambda(\hwww{1}(\bS) - \hwww{\glob}(\bS)), \hwww{1}(\bSSS{1, j_1}) - \hwww{1}(\bS) \hugera \\
    & \overset{\eqref{eq:ep_stat_local_model}}{\leq} \bigg(L_1(\hwww{1}(\bSSS{1, j_1}), \SSS{j_1}_1) + \frac{\lambda}{2}\|\hwww{\glob}(\bSSS{1,j_1}) - \hwww{1}(\bSSS{1, j_1})\|^2\bigg) \\
    & \qquad - \bigg(L_1(\hwww{1}(\bS), \SSS{j_1}_1) + \frac{\lambda}{2}\|\hwww{\glob}(\bSSS{1,j_1}) - \hwww{1}(\bS)\|^2\bigg) \\
    & \qquad - \frac{1}{n_1} \ell(\hwww{1}(\bSSS{1,j_1}), \bsz'_{1, j_1}) + \frac{1}{n_1} \ell(\hwww{1}(\bSSS{1,j_1}), \zzz{1}_{j_1}) + \frac{1}{n} \ell(\hwww{1}(\bS), z'_{1,j_1}) - \frac{1}{n_1} \ell(\hwww{\bS}, \zzz{1}_{j_1})\\
    & \qquad - \frac{\lambda}{2} \|\hwww{\glob}(\bSSS{1, j_1}) - \hwww{1}(\bSSS{1,j_1})\|^2 + \frac{\lambda}{2} \|\hwww{\glob}(\bS) - \hwww{1}(\bSSS{1,j_1})\|^2 \\
    & \qquad + \frac{\lambda}{2} \|\hwww{\glob}(\bSSS{1,j_1}) - \hwww{1}(\bS)\|^2 - \frac{\lambda}{2} \|\hwww{\glob}(\bS) - \hwww{1}(\bS)\|^2 \\
    & \qquad + \zetasup{1}\|\hwww{1}(\bSSS{1,j_1}) - \hwww{1}(\bS)\|\\
    & \overset{\eqref{eq:delta_opt_local_model}}{\leq} \deltasup{1} + \zetasup{1}\|\hwww{1}(\bSSS{1,j_1}) - \hwww{1}(\bS)\|  \\
    & \qquad + \lambda \hugela\hwww{\glob}(\bS) - \hwww{\glob}(\bSSS{1,j_1}),  \hwww{\bS} - \hwww{1}(\bSSS{1,j_1})  \hugera \\
    & \qquad + \frac{1}{n_1} \bigg(\ell(\hwww{1}(\bS), z'_{1,j_1}) - \ell(\hwww{1}(\bSSS{1,j_1}), \bsz'_{1, j_1}) + \ell(\hwww{1}(\bSSS{1,j_1}), \zzz{1}_{j_1}) - \ell(\hwww{1}(\bS), \zzz{1}_{j_1})\bigg) \\
    & {\leq} \deltasup{1} + \zetasup{1}\|\hwww{1}(\bSSS{1,j_1}) - \hwww{1}(\bS)\| \\
    & \qquad + \frac{2}{n_1} \bigg(\sqrt{2\beta \|\ell\|_\infty} \|\hwww{1}(\bS) - \hwww{1}(\bSSS{1,j_1})\| + \frac{\beta}{2}\|\hwww{1}(\bS) - \hwww{1}(\bSSS{1,j_1})\|^2\bigg)\\
    & \qquad + \frac{\lambda + \mu}{\mu} \bigg(2 \zetasup{\glob}  + \sqrt{\frac{2 \lambda \mu \deltasup{\glob}}{\lambda + \mu}} + {4 p_i\lambda \epsup{i}} + {2 p_i\lambda}\|\hwww{i}(\bSSS{i, j_1}) - \hwww{i}(\bS)\|  \bigg) \\
    & \qquad \qquad \times \|\hwww{1}(\bS) - \hwww{1}(\bSSS{1,j_1})\|,
  \end{align*}
  where the last inequality is by Lemma \ref{lemma:loss_stab_to_param_stab} and Lemma \ref{lemma:stab_local_implies_stab_global}. Denoting $\sfs_1:= \|\hwww{1}(\bS) - \hwww{1}(\bSSS{1,j_1})\|$, the above inequality can be written as
  \begin{align}
    \label{eq:quad_ineq_param_stab}
    C_{\lambda, 1} \sfs_1^2 - \bigg[\frac{2\sqrt{2\beta \|\ell\|_\infty}}{n_1 }+ \zetasup{1} + \frac{\lambda + \mu}{\mu}\bigg( 2\zetasup{\glob} + 4 p_1\lambda \epsup{1} + \sqrt{\frac{2 \lambda \mu \deltasup{\glob}}{\lambda + \mu}}  \bigg)\bigg] \sfs_1 - \deltasup{1}  \leq 0,
  \end{align}
  where 
  $$
    C_{\lambda, 1} :=  \frac{1}{2}(\mu+ \lambda) - \frac{\beta}{n_1} - \frac{2p_1 \lambda(\lambda + \mu)}{\mu}.
  $$
  By \eqref{eq:sample_size_assump_for_stab}, we have
  $$
    C_{\lambda, 1 } \geq \frac{\mu+ \lambda}{2} - \frac{\mu}{4} -  \frac{2p_1 \lambda(\lambda + \mu)}{\mu} \geq \frac{\lambda + \mu}{4} -  \frac{2p_1 \lambda(\lambda + \mu)}{\mu} = \frac{\lambda+\mu}{4} \cdot  \bigg(1 - \frac{8p_1\lambda}{\mu}\bigg) \geq \frac{\lambda+\mu}{8}.  
  $$
  In particular, $C_{\lambda, 1} > 0$, and thus we can solve the quadratic inequality \eqref{eq:quad_ineq_param_stab} (similar to the proof of Lemma \ref{lemma:stab_local_implies_stab_global}) to get
  \begin{align*}
    \sfs_1 \leq \frac{2\sqrt{2\beta\|\ell\|_\infty}}{C_{\lambda, 1} n_1} + \frac{\zetasup{1} + \frac{\lambda + \mu}{\mu} \bigg( 2\zetasup{\glob} + 4 p_1\lambda \epsup{1} + \sqrt{\frac{2 \lambda \mu \deltasup{\glob}}{\lambda + \mu}}  \bigg) }{C_{\lambda, 1}} + \sqrt{\frac{\deltasup{1}}{C_{\lambda, 1}}}. 
  \end{align*}
  Plugging in $C_{\lambda, 1} \geq (\lambda+\mu)/8$ to the above inequality gives the desired result.
\end{proof}

We are finally ready to present a proof of Lemma \ref{lemma:stab_approx_min_part1}:
\begin{proof}[Proof of Lemma \ref{lemma:stab_approx_min_part1}]
  Invoking Lemma \ref{lemma:loss_stab_to_param_stab}, we have
  \begin{align*}
    \gamma_i & \leq \sqrt{2\beta\|\ell\|_\infty} \cdot \bigg( \frac{16 \sqrt{2\beta \|\ell\|_\infty}}{n_i(\mu + \lambda)} + \calE_{\lambda, i} \bigg) + \frac{\beta}{2}\bigg( \frac{16 \sqrt{2\beta \|\ell\|_\infty}}{n_i(\mu + \lambda)} + \calE_{\lambda, i} \bigg)^2 \\
    & \leq \frac{32\beta \|\ell\|_\infty}{n_i(\mu+\lambda)}\cdot \bigg(1 + \frac{\beta}{n_i(\mu+\lambda)}\bigg) + \sqrt{2\beta \|\ell\|_\infty} \cdot \calE_{\lambda, i} + \beta \calE_{\lambda, i}^2,
  \end{align*}
  where in the last line we have used $(a+b)^2 \leq 2a^2 + 2b^2$. We finish the proof by noting that $\frac{\beta}{n_i(\mu+\lambda)} \leq \frac{\beta}{n_i\mu} \leq 4$, where the last inequality is by \eqref{eq:sample_size_assump_for_stab}.
\end{proof}

\subsection{Proof of Theorem \ref{thm:ier_sfa_wted}}\label{prf:thm:ier_sfa_wted}

Compared to the proof of Theorem \ref{thm:aer_sfa_wted}, we need to additionally control the estimation error of the global model.
\begin{proposition}[Estimation error of the global model]
\label{prop:est_glob_model}
  Let Assumptions \ref{assump:regularity}(\cvx) and \ref{assump:heterogeneity}(\iersim) hold. Then 
    \begin{align*}
        & \bbE_{\bS}\| \twww{\glob} - \www{\glob}_\bp \|^2 \leq \frac{48\beta^2 \sigma^2}{\mu^2 \lambda^2} \bigg(\sum_{i\in[m]}\frac{p_i}{\sqrt{n_i}}\bigg)^2 + \frac{48\beta^2 R^2}{\mu^2} + \frac{12(\mu+\lambda)^2 \sigma^2}{\mu^2 \lambda^2} \sum_{i\in[m]} \frac{p_i^2}{n_i}.
    \end{align*}
\end{proposition}
\begin{proof}
  See Appendix \ref{prf:prop:est_glob_model}
\end{proof}

With the above proposition, the following result is a counterpart of Proposition \ref{prop:lambda_dep_bd_aer}.
\begin{proposition}[$\lambda$-dependent bound on the IER]
\label{prop:lambda_dep_bd_ier}
Let Assumptions \ref{assump:regularity}(\compact, \cvx), \ref{assump:heterogeneity}(\iersim) and Equation \eqref{eq:sample_size_assump_for_stab} hold. Run $\calA_\SFA$ with hyperparameters chosen as in Lemma \ref{lemma:conv_inner_loop} and \ref{lemma:conv_outer_loop}. Then, for any $i\in[m]$, as long as
\begin{align}
  T & \geq C_1  \lambda(\lambda \lor 1) m\|\bp\|^2 n_i \cdot \bigg( p_i^{-1} \lor [\lambda(\lambda \lor 1)n_i]\bigg), \nonumber\\
  \label{eq:it_num_for_ier}
  K_T & \geq C_2  (\lambda+1)^2 n_i \bigg( p_i^{-1} \lor \lambda^2 p_i^2 n_i\bigg), 
\end{align}
the algorithm $\calA_\SFA$ satisfies both
\begin{align}
  & \bbE_{\calA_{\SFA}, \bS}[\ier_i(\calA_\SFA)] \nonumber\\
  \label{eq:lambda_dep_bd_ier_part1}
  & \leq \frac{C}{\lambda n_i}\bigg[ \beta\|\ell\|_\infty + \frac{\sigma^2 \beta^2 n_i}{\mu^2}\bigg(\sum_{i\in[m]}\frac{p_i}{\sqrt{n_i}}\bigg)^2 + \sigma^2 n_i \sum_{i\in[m]}\frac{p_i^2}{n_i}\bigg] + C \lambda \bigg[ \bigg(1 + \frac{\beta^2}{\mu^2} \bigg)R^2 + \frac{\sigma^2}{\mu^2} \sum_{i\in[m]}\frac{p_i^2}{n_i}\bigg) \bigg],
\end{align} 
and
\begin{align}
  \label{eq:lambda_dep_bd_ier_part2}
  & \bbE_{\calA_{\SFA}, \bS}[\ier_i(\calA_\SFA)] \leq C  \bigg(\frac{\beta\|\ell\|_\infty}{\mu n_i } + \lambda D^2\bigg),
\end{align}
where $C_1, C_2$ are two constants only depending on $(\mu, \beta, \|\ell\|_\infty, D)$, and $C$ is an absolute constant.
\end{proposition}

\begin{proof}
  Without loss of generality we consider the first client. Our assumptions allow us to invoke Propositions \ref{prop:fedstab_implication} and \ref{prop:stab_sfa} to get
  \begin{align*}
    & \bbE_{\calA_\SFA, \bS}[\ier_1] \\
    & \leq \bbE_{\calA_\SFA, \bS}\bigg[\frac{\calE_\opt}{p_1} + \frac{3\lambda}{2} (\epsup{\glob})^2 + \frac{3\lambda}{2} \|\twww{\glob} - \www{\glob}_\bp\|^2 + \frac{3\lambda}{2} R^2 + \frac{320\beta\|\ell\|_\infty}{n_1(\mu+\lambda)}\bigg].
  \end{align*}
  We first show that the expected value of $\calE_\opt/p_1$ and $\lambda(\epsup{\glob})^2$ are both bounded above by a constant multiple of $\frac{\beta\|\ell\|_\infty}{n_1(\mu+\lambda)}$. Indeed, by the estimates we have established in Equations \eqref{eq:epsl_glob_big_O_bd} and \eqref{eq:opt_err_big_O_bd}, it suffices to require 
  $$
    T \gtrsim_{(\mu, \beta, \|\mu\|_\infty, D)} \lambda(\lambda \lor 1) m\|\bp\|^2 n_1,
  $$  
  and 
  $$
    K_T \gtrsim_{(\mu, \beta, \|\mu\|_\infty, D)} \frac{(\lambda \lor 1)^2 n_1}{p_1}, \qquad T\gtrsim_{(\mu, \beta, \|\mu\|_\infty, D)} \frac{\lambda(\lambda \lor 1) m \|\bp\|^2 n_1}{p_1},
  $$
  respectively. And the above two displays, combined with \eqref{eq:it_num_for_stab}, is exactly \eqref{eq:it_num_for_ier}. \eqref{eq:lambda_dep_bd_ier_part2} then follows from the compactness of $\calW$. To prove \eqref{eq:lambda_dep_bd_ier_part1}, we invoke Proposition \ref{prop:est_glob_model} to get
  \begin{align*}
    \bbE_{\calA_\SFA, \bS}[\ier_1] & \lesssim \frac{\beta\|\ell\|_\infty}{n_1(\mu+\lambda)} + \lambda \bigg(1+ \frac{\beta^2}{\mu^2}\bigg) R^2 + \frac{\beta^2 \sigma^2}{\mu^2 \lambda} \bigg(\sum_{i\in[m]}\frac{p_i}{\sqrt{n_i}}\bigg)^2 + \frac{(\mu+\lambda)^2\sigma^2}{\mu^2 \lambda } \sum_{i\in[m]} \frac{p_i^2}{n_i} \\
    & \lesssim \frac{\beta\|\ell\|_\infty}{n_1\lambda} + \lambda \bigg(1+ \frac{\beta^2}{\mu^2}\bigg) R^2 + \frac{\beta^2 \sigma^2}{\mu^2 \lambda} \bigg(\sum_{i\in[m]}\frac{p_i}{\sqrt{n_i}}\bigg)^2 + \bigg(\frac{\sigma^2}{\lambda} + \frac{\lambda \sigma^2}{\mu^2}\bigg)\sum_{i\in[m]} \frac{p_i^2}{n_i},
  \end{align*}  
  and \eqref{eq:lambda_dep_bd_ier_part2} follows by rearranging terms.
\end{proof}

We now present our proof of Theorem \ref{thm:ier_sfa_wted}.
\begin{proof}[Proof of Theorem \ref{thm:ier_sfa_wted}]
  Without loss of generality we consider the first client. 
  Since all $n_i$'s are of the same order, it suffices to show
  \begin{equation}
    \label{eq:ier_sfa_equiv_form}
    \bbE[\ier_{i}(\calA_{\SFA})] \lesssim \bigg[(\mu+\mu^{-1})\bigg(\beta\|\ell\|_\infty + \frac{\sigma^2 \beta^2 + \beta^2 + \sigma^2}{\mu^2}\bigg) + \mu D^2\bigg]\cdot \bigg(\frac{R}{\sqrt{N/m}} \land \frac{1}{N/m} + \frac{\sqrt{m}}{N}\bigg).
  \end{equation}  
  We define the following two events:
  \begin{align*}
    \sfA:= \{ R \geq \sqrt{m/N} \}, \qquad \sfB:= \sfA^c = \{  R  < \sqrt{m/N} \},
  \end{align*}
  and we set 
  $$
    \lambda = \frac{c_\sfA m}{D^2 N} \cdot \indc_\sfA + c_\sfB \sqrt{\frac{m}{R^2 N + 1}} \cdot \indc_\sfB,
  $$
  where $c_\sfA, c_\sfB$ are two constants to be specified later. 
  We consider two cases:
  \begin{enumerate}
    \item If $\sfA$ holds, then from \eqref{eq:lambda_dep_bd_ier_part2} we have
    $
      \bbE_{\calA_\SFA, \bS}[\ier_1] \lesssim \bigg(\frac{\beta\|\ell\|_\infty}{\mu} + c_\sfA\bigg)\cdot \frac{1}{N/m} ,
    $
    provided $\lambda p_{\max} \leq \mu/16$. Note that $\lambda p_{\max} \asymp \frac{c_\sfA}{D^2 N}\leq c_\sfA/D^2$. So we can choose $\lambda\asymp \mu D^2$, which gives
    $$
      \bbE_{\calA_\SFA, \bS}[\ier_1] \lesssim \bigg(\frac{\beta\|\ell\|_\infty}{\mu} + \mu D^2\bigg)\cdot \frac{1}{N/m} \leq \textnormal{right-hand side of }\eqref{eq:ier_sfa_equiv_form}.
    $$
    \item If $\sfB$ holds, and if $\lambda p_{\max} \leq \mu/16$ holds, then from \eqref{eq:lambda_dep_bd_ier_part1} we have 
    \begin{align*}
    & \bbE_{\calA_\SFA, \bS}[\ier_1]\\
    & \lesssim \frac{1}{\lambda N/m} \bigg(\beta \|\ell\|_\infty + \frac{\sigma^2\beta^2}{\mu^2} + \sigma^2\bigg) + \lambda\bigg[ (1 + \frac{\beta^2}{ \mu^2}) R^2 + \frac{\sigma^2}{\mu^2 N}\bigg] \nonumber\\
    & \lesssim \bigg(\beta\|\ell\|_\infty + \frac{\sigma^2 \beta^2 + \beta^2 + \sigma^2}{\mu^2}\bigg) \cdot \bigg(\frac{1}{\lambda N/m} + \lambda(R^2 + N^{-1})\bigg)\\
    & = \bigg(\beta\|\ell\|_\infty + \frac{\sigma^2 \beta^2 + \beta^2 + \sigma^2}{\mu^2}\bigg) \cdot \bigg(\frac{\sqrt{R^2 N + 1}}{c_\sfB N/\sqrt{m}} + \frac{c_\sfB}{N}\sqrt{m(R^2N + 1)}\bigg) \\
    & \leq \bigg(\beta\|\ell\|_\infty + \frac{\sigma^2 \beta^2 + \beta^2 + \sigma^2}{\mu^2}\bigg) \cdot \bigg(\frac{R}{c_\sfB \sqrt{N/m}} + \frac{1}{c_{\sfB} N/\sqrt{m}} + \frac{c_\sfB \sqrt{m}R}{\sqrt{N}} + \frac{c_\sfB \sqrt{m}}{N}\bigg)\\
    & = \bigg(\beta\|\ell\|_\infty + \frac{\sigma^2 \beta^2 + \beta^2 + \sigma^2}{\mu^2}\bigg) \cdot (c_\sfB+ c_\sfB^{-1}) \bigg(\frac{R}{\sqrt{N/m}} + \frac{\sqrt{m}}{N}\bigg).
    \end{align*}
    Note that $p_{\max}\lambda \leq c_\sfB p_{\max} \sqrt{m} \asymp c_\sfB/\sqrt{m}\leq c_{\sfB}$. So to satisfy $p_{\max}\lambda \leq \mu/16$, we can choose $c_{\sfB}\asymp \mu$. This gives
    \begin{align*}
      \bbE_{\calA_\SFA, \bS}[\ier_1]  & \lesssim\bigg(\beta\|\ell\|_\infty + \frac{\sigma^2 \beta^2 + \beta^2 + \sigma^2}{\mu^2}\bigg) (\mu +\mu^{-1}) \cdot \bigg(\frac{R}{\sqrt{N/m}} + \frac{\sqrt{m}}{N}\bigg)\\
      & \leq \textnormal{right-hand side of }\eqref{eq:ier_sfa_equiv_form}.
    \end{align*}      
  \end{enumerate} 
  The desired result follows by combining the above two cases together. 
\end{proof}

\subsubsection{Proof of Proposition \ref{prop:est_glob_model}: Estimation Error of the Global Model} \label{prf:prop:est_glob_model}

We begin by proving a useful lemma.
\begin{lemma}[Estimating $\www{i}_\star$ given the knowledge of $\www{\glob}_\bp$]
\label{lemma:est_local_given_glob}
  Let Assumption \ref{assump:regularity}(\cvx) hold. Then for any $i\in[m]$, we have
    \begin{equation*}
      \| \www{i}_\star - \prox_{L_i/\lambda}(\www{\glob}_\bp)\| \leq \frac{2}{\mu+ \lambda} \bigg\| \nabla L_i(\www{i}_\star, S_i) + \lambda(\www{i}_\star - \www{\glob}_\bp)\bigg\|.
    \end{equation*}
\end{lemma}

\begin{proof}
    This follows from an adaptation of the arguments in Theorem 7 of \cite{foster2019complexity}.
    By strong convexity, we have
    \begin{align*}
      & \hugela \nabla L_i(\www{i}_\star, S_i) + \lambda(\www{i}_\star - \www{\glob}_\bp), \prox_{L_i/\lambda}(\www{\glob}_\bp) - \www{i}_\star \hugera \\
      & \qquad + \frac{\mu+\lambda}{2} \|\www{i}_\star - \prox_{L_i/\lambda}(\www{\glob}_\bp)\|^2 \\
      & \leq L_i(\www{i}_\star, S_i) + \frac{\lambda}{2} \|\www{\glob} - \www{i}_\star\|^2 - L_i(\prox_{L_i/\lambda}(\www{\glob}_\bp), S_i) \\
      & \qquad - \frac{\lambda}{2} \|\www{\glob}_\bp - \prox_{L_i/\lambda}(\www{\glob}_\bp)\|^2 \\
      & \leq 0.
    \end{align*}
    If $\|\www{i}-\prox_{L_i/\lambda}(\www{\glob}_\bp)\| = 0$ we are done. Otherwise, Cauchy-Schwartz inequality applied to the above display gives the desired result. 
\end{proof}

Now, since $\sum_{i\in[m]}p_iF_i$ is $\mu_F =  \mu\lambda/(\mu+\lambda)$-strongly convex, we have
  \begin{align*}
    & \hugela \sum_{i\in[m]}p_i \nabla F_i(\www{\glob}_\bp), \twww{\glob} - \www{\glob}_\bp \hugera + \frac{\mu_F}{2} \|\twww{\glob} - \www{\glob}_\bp\|^2 \\
    & \leq \sum_{i\in[m]} p_i F_i(\twww{\glob}) - \sum_{i\in[m]} p_i F_i(\www{\glob}_\bp) \\
    & \leq 0.
  \end{align*}
  If $\|\twww{\glob} - \www{\glob}_\bp\| = 0$ we are done. Otherwise, by Cauchy-Schwartz inequality, we get
  \begin{align*}
    & \|\twww{\glob} - \www{\glob}_\bp\| \\
    & \leq \frac{2}{\mu_F} \bigg\|\sum_{i\in[m]}p_i \nabla F_i(\www{\glob}_\bp)\bigg\|\\
    & = \frac{2}{\mu_F}  \bigg\| \sum_{i\in[m]}p_i \nabla L_i(\prox_{L_i/\lambda}(\www{\glob}_\bp), S_i) \bigg\| \\
    & \leq \frac{2}{\mu_F}\bigg\| \sum_{i\in[m]}p_i \bigg(\nabla L_i(\prox_{L_i/\lambda}(\www{\glob}_\bp), S_i) - \nabla L_i(\www{i}_\star, S_i)\bigg)\bigg\|  + \frac{2}{\mu_F} \bigg\| \sum_{i\in[m]}p_i \nabla L_i(\www{i}_\star, S_i) \bigg\| \\
    & \overset{(*)}{\leq} \frac{2\beta}{\mu_F} \sum_{i\in[m]}p_i \bigg\|\prox_{L_i/\lambda}(\www{\glob}_\bp) - \www{i}_\star\bigg\|  + \frac{2}{\mu_F} \bigg\| \sum_{i\in[m]}p_i \nabla L_i(\www{i}_\star, S_i) \bigg\| \\
    & \overset{(**)}{\leq}  \frac{4\beta}{\mu_F (\mu+\lambda)}  \sum_{i\in [m]} p_i \bigg\|\nabla L_i(\www{i}_\star, S_i) + \lambda(\www{i}_\star - \www{\glob}_\bp)\bigg\| + \frac{2}{\mu_F} \bigg\| \sum_{i\in[m]}p_i \nabla L_i(\www{i}_\star, S_i) \bigg\|\\
    & \leq \frac{4\beta}{\mu\lambda} \sum_{i\in[m]} p_i \|\nabla L_i(\www{i}_\star, S_i)\| + \frac{4\beta R}{\mu} + \frac{2(\mu+\lambda)}{\mu\lambda} \bigg\|\sum_{i\in[m]} p_i \nabla L_i(\www{i}_\star, S_i)\bigg\|,
  \end{align*}
  where $(*)$ is by smoothness of $L_i$ and $(**)$ is by Lemma \ref{lemma:est_local_given_glob}. Thus, we have
  \begin{align}
    & \|\twww{\glob} - \www{\glob}_\bp\|^2 \nonumber \\
    \label{eq:est_glob_decomp}
    & \leq \frac{48\beta^2}{\mu^2 \lambda^2} \bigg(\sum_{i\in[m]}p_i \|\nabla L_i(\www{i}_\star, S_i)\|\bigg)^2 + \frac{48\beta^2 R^2}{\mu^2} + \frac{12(\mu+\lambda)^2}{\mu^2 \lambda^2} \bigg\|\sum_{i\in[m]}p_i \nabla L_i(\www{i}_\star, S_i)\bigg\|^2
  \end{align}
  Note that
  \begin{align*}
    \bigg(\sum_{i\in[m]}p_i \|\nabla L_i(\www{i}_\star, S_i)\|\bigg)^2 & = \sum_{i\in[m]}  p_i^2 \|\nabla L_i(\www{i}_\star, S_i)\|^2 + \sum_{i\neq s} p_ip_s \|\nabla L_i(\www{i}_\star, S_i)\| \|\nabla L_s(\www{s}_\star, S_s)\|.
  \end{align*}
  Taking expectation at both sides, we arrive at
  \begin{align*}
    \bbE_{\bS}\bigg[\bigg(\sum_{i\in[m]}p_i \|\nabla L_i(\www{i}_\star, S_i)\|\bigg)^2\bigg] \leq \sum_{i\in[m]} \frac{p_i^2 \sigma^2}{n_i} + \sum_{i\neq s} \frac{p_ip_s \sigma_i \sigma_s}{\sqrt{n_i n_s}} = \sigma^2\bigg(\sum_{i\in[m]} \frac{p_i }{\sqrt{n_i}}\bigg)^2.
  \end{align*}
  Meanwhile, we have
  \begin{align*}
    \bbE_{\bS}\bigg\|\sum_{i\in[m]}p_i \nabla L_i(\www{i}_\star, S_i)\bigg\|^2 \leq \sum_{i\in[m]} \frac{p_i^2 \sigma^2}{n_i}.
  \end{align*}
  The desired result follows by plugging the previous two displays to \eqref{eq:est_glob_decomp}.

\section{Details on Experiments} \label{appx:exp}
In each round (among $100$ rounds) of simulation, we first generate $\bsw_\star\in \bbR^{100}$ with i.i.d.~standard Gaussian entries, and we set each local model $\www{i}_\star = \bsw_\star + R \cdot \boldsymbol{v}_i$, where $\boldsymbol{v}_i\in\bbR^{100}$ is a random unit vector that has negative correlation with $\bsw_\star$ and we vary $R$ from $0$ to $20$. The dataset for the $i$-th client is then generated by a logistic regression model. We apply \textsc{FedAvg} (Algorithm \ref{alg:fedavg}), \textsc{PureLocalTraining}, and \textsc{FedAvg} followed by fine tuning, as well as \textsc{FedProx} (Algorithm \ref{alg:softfedavg_wted}) to this collection of datasets. 

For \textsc{FedAvg}, we assume full participation (i.e., $\calC_t = [m]$) and we set the number of communication rounds $T=20$ and global step size $\eta_t = 0.8$. In its local training stage, we run SGD for $5$ epochs with step size $0.2$.
For \textsc{PureLocalTraining}, we run SGD with step size $0.2$ for $20\cdot 5 = 100$ epochs. 
For the fine tuning strategy, we first run \textsc{FedAvg} (with the same hyperparameter as the previous case) and then for each client run SGD for $15$ epochs with step size $0.2$.
For \textsc{FedProx}, we again assume full participation, and we set the number of communication rounds $T=20$, global step size $\eta^{\glob}_t = 0.8$, local rounds $K_t = 5$, and local step size $\eta_{t, k}^{(i)} = 0.2$.
In all the experiments, the batch size is set to $16$.

\end{appendices}

\vskip 0.2in
\bibliography{references}

\begin{thebibliography}{70}
\providecommand{\natexlab}[1]{#1}
\providecommand{\url}[1]{\texttt{#1}}
\expandafter\ifx\csname urlstyle\endcsname\relax
  \providecommand{\doi}[1]{doi: #1}\else
  \providecommand{\doi}{doi: \begingroup \urlstyle{rm}\Url}\fi

\bibitem[Agarwal et~al.(2012)Agarwal, Bartlett, Ravikumar, and
  Wainwright]{agarwal2012information}
Alekh Agarwal, Peter~L Bartlett, Pradeep Ravikumar, and Martin~J Wainwright.
\newblock Information-theoretic lower bounds on the oracle complexity of
  stochastic convex optimization.
\newblock \emph{IEEE Transactions on Information Theory}, 58\penalty0
  (5):\penalty0 3235--3249, 2012.

\bibitem[Arivazhagan et~al.(2019)Arivazhagan, Aggarwal, Singh, and
  Choudhary]{arivazhagan2019federated}
Manoj~Ghuhan Arivazhagan, Vinay Aggarwal, Aaditya~Kumar Singh, and Sunav
  Choudhary.
\newblock Federated learning with personalization layers.
\newblock \emph{arXiv preprint arXiv:1912.00818}, 2019.

\bibitem[Assouad(1983)]{assouad1983deux}
Patrice Assouad.
\newblock Deux remarques sur l'estimation.
\newblock \emph{Comptes rendus des s{\'e}ances de l'Acad{\'e}mie des sciences.
  S{\'e}rie 1, Math{\'e}matique}, 296\penalty0 (23):\penalty0 1021--1024, 1983.

\bibitem[Bai et~al.(2020)Bai, Chen, Zhou, Zhao, Lee, Kakade, Wang, and
  Xiong]{bai2020important}
Yu~Bai, Minshuo Chen, Pan Zhou, Tuo Zhao, Jason~D Lee, Sham Kakade, Huan Wang,
  and Caiming Xiong.
\newblock How important is the train-validation split in meta-learning?
\newblock \emph{arXiv preprint arXiv:2010.05843}, 2020.

\bibitem[Balcan et~al.(2019)Balcan, Khodak, and Talwalkar]{balcan2019provable}
Maria-Florina Balcan, Mikhail Khodak, and Ameet Talwalkar.
\newblock Provable guarantees for gradient-based meta-learning.
\newblock In \emph{International Conference on Machine Learning}, pages
  424--433. PMLR, 2019.

\bibitem[Baxter(2000)]{baxter2000model}
Jonathan Baxter.
\newblock A model of inductive bias learning.
\newblock \emph{Journal of artificial intelligence research}, 12:\penalty0
  149--198, 2000.

\bibitem[Bayoumi et~al.(2020)Bayoumi, Mishchenko, and
  Richtarik]{bayoumi2020tighter}
Ahmed Khaled~Ragab Bayoumi, Konstantin Mishchenko, and Peter Richtarik.
\newblock Tighter theory for local sgd on identical and heterogeneous data.
\newblock In \emph{International Conference on Artificial Intelligence and
  Statistics}, pages 4519--4529, 2020.

\bibitem[Ben-David and Borbely(2008)]{ben2008notion}
Shai Ben-David and Reba~Schuller Borbely.
\newblock A notion of task relatedness yielding provable multiple-task learning
  guarantees.
\newblock \emph{Machine learning}, 73\penalty0 (3):\penalty0 273--287, 2008.

\bibitem[Ben-David et~al.(2006)Ben-David, Blitzer, Crammer, and
  Pereira]{ben2006analysis}
Shai Ben-David, John Blitzer, Koby Crammer, and Fernando Pereira.
\newblock Analysis of representations for domain adaptation.
\newblock \emph{Advances in neural information processing systems},
  19:\penalty0 137--144, 2006.

\bibitem[Ben-David et~al.(2010)Ben-David, Blitzer, Crammer, Kulesza, Pereira,
  and Vaughan]{ben2010theory}
Shai Ben-David, John Blitzer, Koby Crammer, Alex Kulesza, Fernando Pereira, and
  Jennifer~Wortman Vaughan.
\newblock A theory of learning from different domains.
\newblock \emph{Machine learning}, 79\penalty0 (1-2):\penalty0 151--175, 2010.

\bibitem[Bonawitz et~al.(2019)Bonawitz, Eichner, Grieskamp, Huba, Ingerman,
  Ivanov, Kiddon, Kone{\v{c}}n{\`y}, Mazzocchi, and
  McMahan]{bonawitz2019towards}
Keith Bonawitz, Hubert Eichner, Wolfgang Grieskamp, Dzmitry Huba, Alex
  Ingerman, Vladimir Ivanov, Chloe Kiddon, Jakub Kone{\v{c}}n{\`y}, Stefano
  Mazzocchi, and H~Brendan McMahan.
\newblock Towards federated learning at scale: System design.
\newblock \emph{Conference on Machine Learning and Systems}, 2019.

\bibitem[Bousquet and Elisseeff(2002)]{bousquet2002stability}
Olivier Bousquet and Andr{\'e} Elisseeff.
\newblock Stability and generalization.
\newblock \emph{Journal of machine learning research}, 2\penalty0
  (Mar):\penalty0 499--526, 2002.

\bibitem[Cai and Wei(2019)]{cai2019transfer}
T~Tony Cai and Hongji Wei.
\newblock Transfer learning for nonparametric classification: Minimax rate and
  adaptive classifier.
\newblock \emph{arXiv preprint arXiv:1906.02903}, 2019.

\bibitem[Caruana(1997)]{caruana1997multitask}
Rich Caruana.
\newblock Multitask learning.
\newblock \emph{Machine learning}, 28\penalty0 (1):\penalty0 41--75, 1997.

\bibitem[Chen et~al.(2020)Chen, Liu, and Ma]{chen2020global}
Shuxiao Chen, Sifan Liu, and Zongming Ma.
\newblock Global and individualized community detection in inhomogeneous
  multilayer networks.
\newblock \emph{arXiv preprint arXiv:2012.00933}, 2020.

\bibitem[Denevi et~al.(2018)Denevi, Ciliberto, Stamos, and
  Pontil]{denevi2018learning}
Giulia Denevi, Carlo Ciliberto, Dimitris Stamos, and Massimiliano Pontil.
\newblock Learning to learn around a common mean.
\newblock In \emph{Advances in Neural Information Processing Systems}, pages
  10169--10179, 2018.

\bibitem[Denevi et~al.(2019)Denevi, Ciliberto, Grazzi, and
  Pontil]{denevi2019learning}
Giulia Denevi, Carlo Ciliberto, Riccardo Grazzi, and Massimiliano Pontil.
\newblock Learning-to-learn stochastic gradient descent with biased
  regularization.
\newblock In \emph{International Conference on Machine Learning}, pages
  1566--1575, 2019.

\bibitem[Deng et~al.(2020)Deng, Kamani, and Mahdavi]{deng2020adaptive}
Yuyang Deng, Mohammad~Mahdi Kamani, and Mehrdad Mahdavi.
\newblock Adaptive personalized federated learning.
\newblock \emph{arXiv preprint arXiv:2003.13461}, 2020.

\bibitem[Dinh et~al.(2020)Dinh, Tran, and Nguyen]{dinh2020personalized}
Canh~T Dinh, Nguyen~H Tran, and Tuan~Dung Nguyen.
\newblock Personalized federated learning with moreau envelopes.
\newblock \emph{arXiv preprint arXiv:2006.08848}, 2020.

\bibitem[Du et~al.(2020)Du, Hu, Kakade, Lee, and Lei]{du2020few}
Simon~S Du, Wei Hu, Sham~M Kakade, Jason~D Lee, and Qi~Lei.
\newblock Few-shot learning via learning the representation, provably.
\newblock \emph{arXiv preprint arXiv:2002.09434}, 2020.

\bibitem[Duchi(2019)]{duchi2016lecture}
John Duchi.
\newblock Lecture notes for statistics 311/electrical engineering 377.
\newblock \url{http://web.stanford.edu/class/stats311/lecture-notes.pdf}, 2019.
\newblock Accessed: 2020-10-03.

\bibitem[Evgeniou and Pontil(2004)]{evgeniou2004regularized}
Theodoros Evgeniou and Massimiliano Pontil.
\newblock Regularized multi--task learning.
\newblock In \emph{Proceedings of the tenth ACM SIGKDD international conference
  on Knowledge discovery and data mining}, pages 109--117, 2004.

\bibitem[Fallah et~al.(2020)Fallah, Mokhtari, and
  Ozdaglar]{fallah2020personalized}
Alireza Fallah, Aryan Mokhtari, and Asuman Ozdaglar.
\newblock Personalized federated learning: A meta-learning approach.
\newblock \emph{arXiv preprint arXiv:2002.07948}, 2020.

\bibitem[Finn et~al.(2017)Finn, Abbeel, and Levine]{finn2017model}
Chelsea Finn, Pieter Abbeel, and Sergey Levine.
\newblock Model-agnostic meta-learning for fast adaptation of deep networks.
\newblock In \emph{Proceedings of the 34th International Conference on Machine
  Learning-Volume 70}, pages 1126--1135. JMLR. org, 2017.

\bibitem[Foster et~al.(2019)Foster, Sekhari, Shamir, Srebro, Sridharan, and
  Woodworth]{foster2019complexity}
Dylan~J. Foster, Ayush Sekhari, Ohad Shamir, Nathan Srebro, Karthik Sridharan,
  and Blake Woodworth.
\newblock The complexity of making the gradient small in stochastic convex
  optimization, 2019.

\bibitem[Haddadpour and Mahdavi(2019)]{haddadpour2019convergence}
Farzin Haddadpour and Mehrdad Mahdavi.
\newblock On the convergence of local descent methods in federated learning.
\newblock \emph{arXiv preprint arXiv:1910.14425}, 2019.

\bibitem[Hanneke and Kpotufe(2019)]{hanneke2019value}
Steve Hanneke and Samory Kpotufe.
\newblock On the value of target data in transfer learning.
\newblock In \emph{Advances in Neural Information Processing Systems}, pages
  9871--9881, 2019.

\bibitem[Hanneke and Kpotufe(2020)]{hanneke2020no}
Steve Hanneke and Samory Kpotufe.
\newblock A no-free-lunch theorem for multitask learning.
\newblock \emph{arXiv preprint arXiv:2006.15785}, 2020.

\bibitem[Hanzely and Richt{\'a}rik(2020)]{hanzely2020federated}
Filip Hanzely and Peter Richt{\'a}rik.
\newblock Federated learning of a mixture of global and local models.
\newblock \emph{arXiv preprint arXiv:2002.05516}, 2020.

\bibitem[Hanzely et~al.(2020)Hanzely, Hanzely, Horv{\'a}th, and
  Richtarik]{hanzely2020lower}
Filip Hanzely, Slavom{\'\i}r Hanzely, Samuel Horv{\'a}th, and Peter Richtarik.
\newblock Lower bounds and optimal algorithms for personalized federated
  learning.
\newblock \emph{Advances in Neural Information Processing Systems}, 33, 2020.

\bibitem[Jiang et~al.(2019)Jiang, Kone{\v{c}}n{\`y}, Rush, and
  Kannan]{jiang2019improving}
Yihan Jiang, Jakub Kone{\v{c}}n{\`y}, Keith Rush, and Sreeram Kannan.
\newblock Improving federated learning personalization via model agnostic meta
  learning.
\newblock \emph{arXiv preprint arXiv:1909.12488}, 2019.

\bibitem[Jose and Simeone(2021)]{jose2021information}
Sharu~Theresa Jose and Osvaldo Simeone.
\newblock An information-theoretic analysis of the impact of task similarity on
  meta-learning.
\newblock \emph{arXiv preprint arXiv:2101.08390}, 2021.

\bibitem[Kairouz et~al.(2019)Kairouz, McMahan, Avent, Bellet, Bennis, Bhagoji,
  Bonawitz, Charles, Cormode, and Cummings]{kairouz2019advances}
Peter Kairouz, H~Brendan McMahan, Brendan Avent, Aur{\'e}lien Bellet, Mehdi
  Bennis, Arjun~Nitin Bhagoji, Keith Bonawitz, Zachary Charles, Graham Cormode,
  and Rachel Cummings.
\newblock Advances and open problems in federated learning.
\newblock \emph{arXiv preprint arXiv:1912.04977}, 2019.

\bibitem[Kalan et~al.(2020)Kalan, Fabian, Avestimehr, and
  Soltanolkotabi]{kalan2020minimax}
Seyed Mohammadreza~Mousavi Kalan, Zalan Fabian, A~Salman Avestimehr, and Mahdi
  Soltanolkotabi.
\newblock Minimax lower bounds for transfer learning with linear and one-hidden
  layer neural networks.
\newblock \emph{arXiv preprint arXiv:2006.10581}, 2020.

\bibitem[Khaled et~al.(2019)Khaled, Mishchenko, and
  Richt{\'a}rik]{khaled2019first}
Ahmed Khaled, Konstantin Mishchenko, and Peter Richt{\'a}rik.
\newblock First analysis of local gd on heterogeneous data.
\newblock \emph{arXiv preprint arXiv:1909.04715}, 2019.

\bibitem[Khodak et~al.(2019)Khodak, Balcan, and Talwalkar]{khodak2019adaptive}
Mikhail Khodak, Maria-Florina~F Balcan, and Ameet~S Talwalkar.
\newblock Adaptive gradient-based meta-learning methods.
\newblock In \emph{Advances in Neural Information Processing Systems}, pages
  5917--5928, 2019.

\bibitem[Konobeev et~al.(2020)Konobeev, Kuzborskij, and
  Szepesvári]{konobeev2020optimality}
Mikhail Konobeev, Ilja Kuzborskij, and Csaba Szepesvári.
\newblock On optimality of meta-learning in fixed-design regression with
  weighted biased regularization.
\newblock \emph{arXiv preprint arXiv:2011.00344}, 2020.

\bibitem[Kulkarni et~al.(2020)Kulkarni, Kulkarni, and Pant]{kulkarni2020survey}
Viraj Kulkarni, Milind Kulkarni, and Aniruddha Pant.
\newblock Survey of personalization techniques for federated learning.
\newblock \emph{arXiv preprint arXiv:2003.08673}, 2020.

\bibitem[Lemar{\'e}chal and Sagastiz{\'a}bal(1997)]{lemarechal1997practical}
Claude Lemar{\'e}chal and Claudia Sagastiz{\'a}bal.
\newblock Practical aspects of the moreau--yosida regularization: Theoretical
  preliminaries.
\newblock \emph{SIAM Journal on Optimization}, 7\penalty0 (2):\penalty0
  367--385, 1997.

\bibitem[Li and Wang(2019)]{li2019fedmd}
Daliang Li and Junpu Wang.
\newblock Fedmd: Heterogenous federated learning via model distillation.
\newblock \emph{arXiv preprint arXiv:1910.03581}, 2019.

\bibitem[Li et~al.(2020{\natexlab{a}})Li, Cai, and Li]{li2020transfer}
Sai Li, T~Tony Cai, and Hongzhe Li.
\newblock Transfer learning for high-dimensional linear regression: Prediction,
  estimation, and minimax optimality.
\newblock \emph{arXiv preprint arXiv:2006.10593}, 2020{\natexlab{a}}.

\bibitem[Li et~al.(2018)Li, Sahu, Zaheer, Sanjabi, Talwalkar, and
  Smith]{li2018federated}
Tian Li, Anit~Kumar Sahu, Manzil Zaheer, Maziar Sanjabi, Ameet Talwalkar, and
  Virginia Smith.
\newblock Federated optimization in heterogeneous networks.
\newblock \emph{arXiv preprint arXiv:1812.06127}, 2018.

\bibitem[Li et~al.(2020{\natexlab{b}})Li, Huang, Yang, Wang, and
  Zhang]{li2020convergence}
Xiang Li, Kaixuan Huang, Wenhao Yang, Shusen Wang, and Zhihua Zhang.
\newblock On the convergence of fedavg on non-iid data, 2020{\natexlab{b}}.

\bibitem[Li and Richt{\'a}rik(2020)]{li2020unified}
Zhize Li and Peter Richt{\'a}rik.
\newblock A unified analysis of stochastic gradient methods for nonconvex
  federated optimization.
\newblock \emph{arXiv preprint arXiv:2006.07013}, 2020.

\bibitem[Lucas et~al.(2020)Lucas, Ren, Kameni, Pitassi, and
  Zemel]{lucas2020theoretical}
James Lucas, Mengye Ren, Irene Kameni, Toniann Pitassi, and Richard Zemel.
\newblock Theoretical bounds on estimation error for meta-learning.
\newblock \emph{arXiv preprint arXiv:2010.07140}, 2020.

\bibitem[Malinovsky et~al.(2020)Malinovsky, Kovalev, Gasanov, Condat, and
  Richtarik]{malinovsky2020local}
Grigory Malinovsky, Dmitry Kovalev, Elnur Gasanov, Laurent Condat, and Peter
  Richtarik.
\newblock From local sgd to local fixed point methods for federated learning.
\newblock \emph{arXiv preprint arXiv:2004.01442}, 2020.

\bibitem[Mangasarian and Solodov(1993)]{mangasarian1993backpropagation}
OL~Mangasarian and MV~Solodov.
\newblock Backpropagation convergence via deterministic nonmonotone perturbed
  minimization.
\newblock In \emph{Proceedings of the 6th International Conference on Neural
  Information Processing Systems}, pages 383--390, 1993.

\bibitem[Mansour et~al.(2020)Mansour, Mohri, Ro, and Suresh]{mansour2020three}
Yishay Mansour, Mehryar Mohri, Jae Ro, and Ananda~Theertha Suresh.
\newblock Three approaches for personalization with applications to federated
  learning.
\newblock \emph{arXiv preprint arXiv:2002.10619}, 2020.

\bibitem[Maurer(2005)]{maurer2005algorithmic}
Andreas Maurer.
\newblock Algorithmic stability and meta-learning.
\newblock \emph{Journal of Machine Learning Research}, 6\penalty0
  (Jun):\penalty0 967--994, 2005.

\bibitem[Maurer et~al.(2016)Maurer, Pontil, and
  Romera-Paredes]{maurer2016benefit}
Andreas Maurer, Massimiliano Pontil, and Bernardino Romera-Paredes.
\newblock The benefit of multitask representation learning.
\newblock \emph{The Journal of Machine Learning Research}, 17\penalty0
  (1):\penalty0 2853--2884, 2016.

\bibitem[McMahan et~al.(2017)McMahan, Moore, Ramage, Hampson, and
  y~Arcas]{mcmahan2017communication}
Brendan McMahan, Eider Moore, Daniel Ramage, Seth Hampson, and Blaise~Aguera
  y~Arcas.
\newblock Communication-efficient learning of deep networks from decentralized
  data.
\newblock In \emph{Artificial Intelligence and Statistics}, pages 1273--1282.
  PMLR, 2017.

\bibitem[Nesterov(2018)]{nesterov2018lectures}
Yurii Nesterov.
\newblock \emph{Lectures on convex optimization}, volume 137.
\newblock Springer, 2018.

\bibitem[Pan and Yang(2009)]{pan2009survey}
Sinno~Jialin Pan and Qiang Yang.
\newblock A survey on transfer learning.
\newblock \emph{IEEE Transactions on knowledge and data engineering},
  22\penalty0 (10):\penalty0 1345--1359, 2009.

\bibitem[Poushter(2016)]{poushter2016smartphone}
Jacob Poushter.
\newblock Smartphone ownership and internet usage continues to climb in
  emerging economies.
\newblock \emph{Pew research center}, 22\penalty0 (1):\penalty0 1--44, 2016.

\bibitem[Rakhlin et~al.(2011)Rakhlin, Shamir, and Sridharan]{rakhlin2011making}
Alexander Rakhlin, Ohad Shamir, and Karthik Sridharan.
\newblock Making gradient descent optimal for strongly convex stochastic
  optimization.
\newblock \emph{arXiv preprint arXiv:1109.5647}, 2011.

\bibitem[Shalev-Shwartz and Ben-David(2014)]{shalev2014understanding}
Shai Shalev-Shwartz and Shai Ben-David.
\newblock \emph{Understanding machine learning: From theory to algorithms}.
\newblock Cambridge university press, 2014.

\bibitem[Shalev-Shwartz et~al.(2009)Shalev-Shwartz, Shamir, Srebro, and
  Sridharan]{shalev2009stochastic}
Shai Shalev-Shwartz, Ohad Shamir, Nathan Srebro, and Karthik Sridharan.
\newblock Stochastic convex optimization.
\newblock In \emph{COLT}, 2009.

\bibitem[Shamir and Zhang(2013)]{shamir2013stochastic}
Ohad Shamir and Tong Zhang.
\newblock Stochastic gradient descent for non-smooth optimization: Convergence
  results and optimal averaging schemes.
\newblock In \emph{International conference on machine learning}, pages 71--79,
  2013.

\bibitem[Shui et~al.(2020)Shui, Chen, Wen, Zhou, Gagn{\'e}, and
  Wang]{shui2020beyond}
Changjian Shui, Qi~Chen, Jun Wen, Fan Zhou, Christian Gagn{\'e}, and Boyu Wang.
\newblock Beyond $\mathcal{H}$-divergence: Domain adaptation theory with
  jensen-shannon divergence.
\newblock \emph{arXiv preprint arXiv:2007.15567}, 2020.

\bibitem[Stich(2019)]{stich2019local}
Sebastian~U. Stich.
\newblock Local sgd converges fast and communicates little, 2019.

\bibitem[Tripuraneni et~al.(2020{\natexlab{a}})Tripuraneni, Jin, and
  Jordan]{tripuraneni2020provable}
Nilesh Tripuraneni, Chi Jin, and Michael~I Jordan.
\newblock Provable meta-learning of linear representations.
\newblock \emph{arXiv preprint arXiv:2002.11684}, 2020{\natexlab{a}}.

\bibitem[Tripuraneni et~al.(2020{\natexlab{b}})Tripuraneni, Jordan, and
  Jin]{tripuraneni2020theory}
Nilesh Tripuraneni, Michael~I Jordan, and Chi Jin.
\newblock On the theory of transfer learning: The importance of task diversity.
\newblock \emph{arXiv preprint arXiv:2006.11650}, 2020{\natexlab{b}}.

\bibitem[Vershynin(2010)]{vershynin2010introduction}
Roman Vershynin.
\newblock Introduction to the non-asymptotic analysis of random matrices.
\newblock \emph{arXiv preprint arXiv:1011.3027}, 2010.

\bibitem[Wang et~al.(2018)Wang, Wang, Kolar, and Srebro]{wang2018distributed}
Weiran Wang, Jialei Wang, Mladen Kolar, and Nathan Srebro.
\newblock Distributed stochastic multi-task learning with graph regularization.
\newblock \emph{arXiv preprint arXiv:1802.03830}, 2018.

\bibitem[Woodworth et~al.(2020)Woodworth, Patel, and
  Srebro]{woodworth2020minibatch}
Blake Woodworth, Kumar~Kshitij Patel, and Nathan Srebro.
\newblock Minibatch vs local sgd for heterogeneous distributed learning.
\newblock \emph{arXiv preprint arXiv:2006.04735}, 2020.

\bibitem[Yu(1997)]{yu1997assouad}
Bin Yu.
\newblock Assouad, fano, and le cam.
\newblock In \emph{Festschrift for Lucien Le Cam}, pages 423--435. Springer,
  1997.

\bibitem[Yu et~al.(2020)Yu, Bagdasaryan, and Shmatikov]{yu2020salvaging}
Tao Yu, Eugene Bagdasaryan, and Vitaly Shmatikov.
\newblock Salvaging federated learning by local adaptation, 2020.

\bibitem[Yuan and Ma(2020)]{yuan2020federated}
Honglin Yuan and Tengyu Ma.
\newblock Federated accelerated stochastic gradient descent.
\newblock \emph{arXiv preprint arXiv:2006.08950}, 2020.

\bibitem[Zhang et~al.(2020)Zhang, Yang, Wu, Su, and R{\'e}]{zhang2020sharp}
Hongyang~R Zhang, Fan Yang, Sen Wu, Weijie~J Su, and Christopher R{\'e}.
\newblock Sharp bias-variance tradeoffs of hard parameter sharing in
  high-dimensional linear regression.
\newblock \emph{arXiv preprint arXiv:2010.11750}, 2020.

\bibitem[Zheng et~al.(2021)Zheng, Chen, Long, and Su]{zheng2021federated}
Qinqing Zheng, Shuxiao Chen, Qi~Long, and Weijie~J Su.
\newblock Federated $ f $-differential privacy.
\newblock \emph{arXiv preprint arXiv:2102.11158}, 2021.

\end{thebibliography}

\end{document}